\documentclass[letterpaper,11pt]{article}

\usepackage{diagbox}
\usepackage{makecell}
\usepackage{booktabs}
\usepackage{tikz}
\usepackage{amsmath}
\usepackage{amsfonts}
\usepackage{amssymb}
\usepackage{amsthm}
\usepackage{url,ifthen}
\usepackage{enumitem}
\usepackage{srcltx}
\usepackage{multirow}
\usepackage{boxedminipage}
\usepackage[margin=1in]{geometry}
\usepackage{nicefrac}
\usepackage{xspace}
\usepackage{graphicx}
\usepackage{color}
\usepackage{subfigure}
\usepackage{colortbl}
\usepackage{setspace}
\usepackage{natbib}
\usepackage{dsfont}
\setcitestyle{square, authoryear}
\definecolor{DarkGreen}{rgb}{0.1,0.5,0.1}
\definecolor{DarkRed}{rgb}{0.5,0.1,0.1}
\definecolor{DarkBlue}{rgb}{0.1,0.1,0.5}
\definecolor{Gray}{rgb}{0.2,0.2,0.2}

\usepackage{listings}
\lstdefinestyle{mystyle}{
    commentstyle=\color{DarkBlue},
    keywordstyle=\color{DarkRed},
    numberstyle=\tiny\color{Gray},
    stringstyle=\color{DarkGreen},
    basicstyle=\footnotesize,
    breakatwhitespace=false,         
    breaklines=true,                 
    captionpos=b,                    
    keepspaces=true,                 
    numbers=left,                    
    numbersep=5pt,                  
    showspaces=false,                
    showstringspaces=false,
    showtabs=false,                  
    tabsize=2
}
\usepackage[small]{caption}
\usepackage[pdftex]{hyperref}
\hypersetup{
    unicode=false,          
    pdftoolbar=true,        
    pdfmenubar=true,        
    pdffitwindow=false,      
    pdfnewwindow=true,      
    colorlinks=true,       
    linkcolor=DarkBlue,          
    citecolor=DarkGreen,        
    filecolor=DarkRed,      
    urlcolor=DarkBlue,          
    pdftitle={},
    pdfauthor={},
}

\newcommand{\chapterref}[1]{\hyperref[ch:#1]{Chapter~\ref{ch:#1}}}

\newcommand{\claimref}[1]{\hyperref[claim:#1]{Claim~\ref{claim:#1}}}

\newcommand{\corollaryref}[1]{\hyperref[cor:#1]{Corollary~\ref{cor:#1}}}

\newcommand{\definitionref}[1]{\hyperref[def:#1]{Definition~\ref{def:#1}}}

\newcommand{\equationref}[1]{\hyperref[eq:#1]{Equation~\ref{eq:#1}}}

\newcommand{\factref}[1]{\hyperref[fact:#1]{Fact~\ref{fact:#1}}}

\newcommand{\figureref}[1]{\hyperref[fig:#1]{Figure~\ref{fig:#1}}}

\newcommand{\tableref}[1]{\hyperref[tab:#1]{Table~\ref{tab:#1}}}

\newcommand{\itemref}[1]{\hyperref[item:#1]{Item~(\ref{item:#1})}}

\newcommand{\lemmaref}[1]{\hyperref[lem:#1]{Lemma~\ref{lem:#1}}}

\newcommand{\propref}[1]{\hyperref[prop:#1]{Proposition~\ref{prop:#1}}}
\newcommand{\propositionlabel}[1]{\label{prop:#1}}
\newcommand{\propositionref}[1]{\hyperref[prop:#1]{Proposition~\ref{prop:#1}}}

\newcommand{\remarkref}[1]{\hyperref[rem:#1]{Remark~\ref{rem:#1}}}

\newcommand{\sectionref}[1]{\hyperref[sec:#1]{Section~\ref{sec:#1}}}

\newcommand{\appendixref}[1]{\hyperref[sec:#1]{Appendix~\ref{sec:#1}}}

\newcommand{\theoremref}[1]{\hyperref[thm:#1]{Theorem~\ref{thm:#1}}}

\newcommand{\exampleref}[1]{\hyperref[claim:#1]{Example~\ref{claim:#1}}}
\newcommand{\setuplabel}[1]{\label{claim:#1}}
\newcommand{\setupref}[1]{\hyperref[claim:#1]{Setup~\ref{claim:#1}}}

\usepackage[utf8]{inputenc}
\usepackage[T1]{fontenc}
\usepackage{kpfonts}
\usepackage{microtype}

\usepackage{macros}

\usepackage{thmtools,thm-restate}
\newtheorem{theorem}{Theorem}
\newtheorem{corollary}[theorem]{Corollary}
\newtheorem{lemma}[theorem]{Lemma}
\newtheorem{proposition}[theorem]{Proposition}
\theoremstyle{definition}

\newtheorem{fact}{Fact}

\newtheorem{property}{Property}

\newtheorem{assumption}{Assumption}
\newtheorem*{remark}{Remark}
\newtheorem{example}{Example}
\newtheorem{setup}{Setup}

\newtheorem*{property*}{Property}

\usepackage{apptools}
\AtAppendix{\counterwithin{proposition}{section}}
\AtAppendix{\counterwithin{corollary}{section}}
\AtAppendix{\counterwithin{lemma}{section}}

\title{Alternative Microfoundations for Strategic Classification} 

\author{\\Meena Jagadeesan~~~~Celestine Mendler-D\"unner~~~~Moritz Hardt\\
{\small \{mjagadeesan, mendler, hardt\}@berkeley.edu}
\\ \\University of California, Berkeley}
\date{}

\begin{document}
\maketitle

\begin{abstract}
When reasoning about strategic behavior in a machine learning context it is tempting to combine \emph{standard microfoundations} of rational agents with the statistical decision theory underlying classification. 
In this work, we argue that a direct combination of these standard ingredients leads to brittle solution concepts of limited descriptive and prescriptive value.
First, we show that rational agents with perfect information produce discontinuities in the aggregate response to a decision rule that we often do not observe empirically.
Second, when any positive fraction of agents is not perfectly strategic, desirable stable points---where the classifier is optimal for the data it entails---cease to exist.
Third, optimal decision rules under standard microfoundations maximize a measure of negative externality known as \emph{social burden} within a broad class of possible assumptions about agent behavior.

Recognizing these limitations we explore alternatives to standard microfoundations for binary classification. We start by describing a set of desiderata that help navigate the space of possible assumptions about how agents respond to a decision rule. In particular, we analyze a natural constraint on feature manipulations, and discuss properties that are sufficient to guarantee the robust existence of stable points. Building on these insights, we then propose the \emph{noisy response} model. Inspired by smoothed analysis and empirical observations, noisy response incorporates imperfection in the agent responses, which we show mitigates the limitations of standard microfoundations. Our model retains analytical tractability, leads to more robust insights about stable points, and imposes a lower social burden at optimality.\footnote{A shorter version of this manuscript was accepted for publication at ICML 2021.}
\end{abstract}

\section{Introduction}
\label{sec:introduction}
Consequential decisions compel individuals to react in response to the specifics of the decision rule. This individual-level response in aggregate can disrupt both statistical patterns and social facts that motivated the decision rule, leading to unforeseen consequences. A similar conundrum in the context of macroeconomic policy making fueled the \emph{microfoundations} program following the influential critique of macroeconomics by Lucas in the 1970s~\citep{lucas1976econometric}. Microfoundations refers to a vast theoretical project that aims to ground theories of aggregate outcomes and population forecasts in microeconomic assumptions about individual behavior. Oversimplifying a broad endeavor, the hope was that if economic policy were \emph{microfounded}, it would anticipate more accurately the response that the policy induces.

Predominant in neoclassical economic theory is the assumption of an agent that exhaustively maximizes a utility function on the basis of perfectly accurate information.
This modeling assumption about agent behavior underwrites many celebrated results on markets, mechanisms, and games. 
Although called into question by behavioral economics and related fields (e.g. see \citep{behavioralecon}), the assumption remains central to economic theory and has become standard in computer science, as well.

When reasoning about incentives and strategic behavior in the context of classification tasks, it is tempting to combine the predominant modeling assumptions from microeconomic theory with the statistical decision theory underlying classification. In the resulting model, agents have perfect information about the decision rule and compute best-response feature changes according to their utility function with the goal of achieving a more favorable classification outcome. We refer to this agent model as \textit{standard microfoundations}. Building on the assumption that agents follow standard microfoundations, the decision-maker then chooses the decision rule that maximizes their own objective in anticipation of the resulting agent response. This is the conceptual route taken in the area of \emph{strategic classification}, but similar observations may apply more broadly to the intersection of economics and learning.

\subsection{Our work}

We argue that standard microfoundations are a poor basis for studying strategic behavior in binary classification problems. We make this point through three observations that illustrate the limited descriptive power of the standard model and the problematic solution concepts it implies. In response, we explore the space of alternative agent models for strategic classification, and we identify desirable properties that when satisfied by microfoundations lead to more realistic and robust insights. Guided by these desiderata, we propose \emph{noisy response} as a promising alternative to the standard model. 

\subsubsection{Limitations of standard microfoundations}
In strategic classification, agents respond strategically to the deployment of a binary decision rule $f_\theta$ specified by classifier weights~$\theta$. The decision-maker assumes that agents follow standard microfoundations: that is, agents have full information about $f_\theta$ and change their features so as to maximize their utility function. The utility function captures the benefit of a positive classification outcome, as well as the cost of feature change. Consequently, an agent only invests in changing their features if the cost of feature change does not exceed the benefit of positive classification.

Our first observation concerns the \textit{aggregate response}---the distribution $\cD(\theta)$ over feature, label pairs induced by a classifier $f_\theta$. We show that in the standard model, the aggregate response necessarily exhibits discontinuities that we often do not observe in empirical settings. The problem persists even if we assume an approximate best response or allow for heterogeneous cost functions.

Our second observation reveals that, apart from lacking descriptive power, the standard model also leads to brittle conclusions about the solution concept of \textit{performative stability}.
Performative stability \citep{PZMH20} refers to decision rules that are optimal on the particular distribution they entail. Stable points thus represent fixed points of \textit{retraining methods}, which repeatedly update the classifier weights to be optimal on the data distribution induced by the previous classifier. We show that the existence of performatively stable classifiers breaks down whenever a positive fraction of randomly chosen agents in the population are non-strategic. This brittleness of the existence of fixed points suggests that the standard model does not constitute a reliable basis for investigating dynamics of retraining algorithms.

Our last observation concerns the solution concept of \textit{performative optimality}. Performative optimality \citep{PZMH20} refers to a decision rule that exhibits the highest accuracy on the distribution it induces. The global nature of this solution concept means that finding performatively optimal points requires the decision-maker to anticipate strategic feedback effects. We prove that relying on standard microfoundations to model strategic behavior leads to extreme decision rules that maximize a measure of negative externality called \textit{social burden} within a broad class of alternative models. Social burden, proposed in recent work, quantifies the expected cost that positive instances of a classification problem have to incur in order to be accepted. Thus, 
standard microfoundations produce optimal solutions that are unfavorable for the population being classified.

\subsubsection{Alternative microfoundations}
Recognizing the limitations of standard microfoundations, we systematically explore alternatives to the standard model.
We investigate alternative assumptions on agent responses, encompassing general agent behavior that need not be fully informed, strategic, or utility maximizing. We formalize microfoundations as a randomized mapping $M:X\times Y\rightarrow \AllTypes$ that assigns each agent to a response type $t\in\AllTypes$. The response type $t$ is associated with a response function $\Response_t\colon X\times\Theta\to X$ specifying how agents of type~$t$ change their features $x$ in response to each decision rule $f_{\theta}$. 

Letting $\DBase$ be the base distribution over features and labels prior to any strategic adaptation, the \emph{aggregate response} to a classifier~$f_\theta$ is given by the distribution $\DMap(\theta;M)$ over induced feature, label pairs  $(\Response_t(x,\theta), y)$ for a random draw $(x, y)\sim\DBase$ and $t= M(x,y)$. In this sense, the mapping $M$ \emph{microfounds} the distributions induced by decision rules, endowing the distributions with structure that allows the decision-maker to deduce the aggregate response from a model of individual behavior.

To guide our search for more appropriate microfoundations for binary classification, we introduce a collection of properties that are desirable for a model of agent responses to satisfy. The first condition, that we call \textit{aggregate smoothness}, rules out the discontinuities arising from standard microfoundations. Conceptually, it requires that varying the classifier weights slightly must change the aggregate response smoothly. We find that this property alone is sufficient to guarantee the robust existence of stable points under mixtures with non-strategic agents. 

The second condition, that we call the \textit{expenditure constraint}, helps ensure that the model encodes realistic agent-level responses $\Response_t$. At a high level, it requires that each agent does not spend more on changing their features than the utility of a positive outcome. This natural constraint gives rise to a large set of potential models. For any such model that satisfies an additional weak assumption, the social burden of the optimal classifier is no larger than the social burden of the optimal classifier deduced from the standard model. Moreover, the optimal points are determined by local behavior. This frees the decision-maker from fully understanding the aggregate response $\cD(\theta)$ and makes the task of finding an approximately optimal classifier more tractable.

\subsubsection{Noisy response---an alternative model} 
Using the properties described above as a compass to navigate the large space of alternative models, we identify \emph{noisy response} as a compelling model of microfoundations. In this model, each agent best responds with respect to $\theta+\xi$, where $\xi$ is an independent sample from a zero mean noise distribution. This model is inspired by \textit{smoothed analysis} \citep{smoothedanalysis} and encodes imperfection in the population's response to a decision rule by perturbing the manipulation targets of individual agents. 

We show that noisy response satisfies a number of desirable properties that make it a promising model of microfoundations for classification in strategic settings, both for theoretical analyses and from a practical standpoint. First, noisy response satisfies aggregate smoothness, and thus leads to the robust existence of stable points. Moreover, the model satisfies the expenditure constraint, and thus encodes natural agent-level responses which can be used to reason about metrics such as social burden. When used to anticipate strategic feedback effects and compute optimal points, noisy response leads to strictly less pessimistic acceptance thresholds than those computed under standard microfoundations, given the same constraints on manipulation expenditure. In fact, we show via simulations that a larger variance $\sigma$ of the noise in the manipulation target leads to a more conservative optimal threshold, and for $\sigma\rightarrow 0$, we approximate the extreme case of standard microfoundations. Finally, from a practical perspective, we demonstrate that the distribution map $\cD(\theta)$ for noisy response can be estimated by gathering information about individuals to learn the parameter $\sigma$ and the cost function $c$. This has the desirable consequence that the decision-maker can compute performatively optimal points without ever deploying a classifier. 
\subsection{Related work}
\label{ref:relatedwork}

Existing work on strategic classification in machine learning has mostly followed standard microfoundations for modeling agent behavior in response to a decision rule, e.g.,~\citep{DDMSV04, BS11, HMPW16, Khaje19opt,tsirtsis2020decisions} to name a few. This includes works that focus on minimizing Stackelberg regret~\citep{dong18revealedpref, CLP20}, quantify the price of transparency~\citep{akyol16transp}, and investigate the benefits of randomization in the decision rule~\citep{BG20}. Investigations on externalities such as social cost~\citep{MMDH19, HIW19} and whether classifiers incentivize improvement as opposed to gaming~\citep{Kleinberg19strat, MMH20, SEA20, HILW20} have also mostly built on the standard assumption of best-responding agents with perfect information. Recent work by~\cite{levanon2021strategic} studied practical implications of the optimization problem resulting from standard microfoundations and how to make it more amenable to practical optimization methods. 

A handful of works have suggested potential limitations of the standard strategic classification framework. 
\citet{bruckner12pred} recognized that the standard model leads to very conservative Stackelberg solutions, and proposed resorting to Nash equilibria as an alternative solution concept. We instead take a different route and advocate for altogether rethinking standard microfoundations that lead to these conservative acceptance thresholds. Concurrent and independent work by \citet{ghalme2021strategic} and \citet{bechavod2021information} 
relaxed the perfect information assumption in the standard model and studied strategic classification when the classifier is not fully revealed to the agents. In this work, we argue that agents often do not perfectly respond to the classifier even when the decision rule is fully transparent. Therefore, we propose incorporating imperfection into the model of microfoundations in order to anticipate natural deviations from the standard assumptions.

Related work in economics also investigates strategic responses to decision rules. This line of work, initiated by \cite{Spence73}, has shown that information about individuals can become \emph{muddled} as a result of heterogeneous gaming behavior \citep{FK19}, investigated the role of commitment power of the decision-maker \citep{FK20}, considered the impact of an intermediary who aggregates the agents' multi-dimensional features \citep{Ball2020}, and considered the performance of different training approaches in strategic environments \citep{HG20}.  
A notable work by~\cite{BBK20} investigates strategic behavior through a field experiment in the micro-lending domain, with a focus on evaluating approaches for designing strategy-robust classifiers. An important distinction is that these works tend to study regression, while we focus on classification. These settings appear to be qualitatively different in the context of strategic feedback effects (e.g. see note in~\citep{HG20}).

Our work is conceptually related to recent work in economics that has recognized mismatches between the predictions of standard models and empirical realities, for example in macroeconomic policy \citep{S18, KV18, CGK18} and in mechanism design \citep{L17}. These works, and many others, have explored incorporating richer behavioral and informational assumptions into typical models used in economic settings. Although our work also explores alternatives to standard microfoundations, we focus on algorithmic decision-making, where the limitations of the standard model had not been previously identified. We believe that our approach of navigating the entire space of potential models using a collection of properties could be of broader interest when developing alternative microfoundations.

\subsection{Setup and basic notation}\label{sec:background}
Let $X \subseteq \mathbb{R}^m$ denote the feature space, and let $Y = \left\{0,1 \right\}$ be the space of binary outcomes. Each agent is associated to a feature vector $x \in X$ and a binary outcome $y \in Y$ which represents their true label. A feature, label pair $(x,y)$ need not uniquely describe an agent, and many agents may be associated to the same pair $(x,y)$. The base distribution $\DBase$ is a joint distribution over $X \times Y$ describing the population prior to any strategic adaption. Throughout the paper we assume that $\DBase$ is continuous and has zero mass on the boundary of $X$.  We focus on binary classification where each classifier $f_{\theta}: X \rightarrow \left\{0,1\right\}$ is parameterized by $\theta\in \mathbb{R}^d$, and the decision-maker selects classifier weights $\theta$ from $\Theta\subseteq \mathbb{R}^d$ which is a compact, convex set. We assume that for every $\theta \in \Theta$, the set $\left\{x \in X \mid f_{\theta}(x) = 1 \right\}$ is closed, and the decision boundary is measure $0$. We adopt the notion of a distribution map $\cD(\theta)$ from~\citep{PZMH20} to describe the distribution over $X\times Y$ induced by strategic adaptation of agents drawn from the base distribution in response to the classifier $f_{\theta}$.

\section{Limitations of standard microfoundations}

\label{sec:limitations}
In the strategic classification literature, the typical agent model is a \textit{rational agent with perfect information}. At the core of this model lies the assumption that agents have perfect knowledge of the classifier and maximize their utility given the classifier weights. The utility consists of two terms: a reward for obtaining a positive classification and a cost of manipulating features. The reward is denoted $\gamma>0$ and the manipulation cost is represented by a function $c: X \times X \rightarrow \mathbb{R}$ where $c(x,x')$ reflects how much agents need to expend to change their features from $x$ to $x'$. A valid cost function satisfies a natural monotonicity requirement as stated in Assumption~\ref{assumption:validcost}. Given a feature vector $x$ and a classifier $f_{\theta}$, agents solve the following utility maximization problem:
\begin{equation}
\label{eq:SM}
\argmax_{x'\in X} \;\left[\gamma f_{\theta}(x') - \costfn(x, x')\right].
\end{equation}
\begin{assumption} 
\vspace{-0.5cm}
\label{assumption:validcost}
A cost function $c: X \times X \rightarrow \mathbb{R}$ is \emph{valid}, if it is continuous in both arguments, it holds that $\costfn(x, x')=0$ for $x=x'$, and $c$ increases with distance in the sense that $\costfn(x,  \bar x) < \costfn(x, x')$ and $\costfn(\bar x, x) < \costfn(x', x)$ for every  $\bar x\in X$ that lies on the line segment connecting the two points $x,x'\in X$.\footnote{We model a non-zero cost for all modifications to features, regardless of whether they result in positive classification or not. Generalizing beyond standard microfoundations, this accounts for how agents may erroneously expend effort on changing their features in an incorrect direction, as empirically demonstrated by \cite{BBK20}.}
\end{assumption}
\noindent We will refer to this response model as \emph{standard microfoundations}.

\subsection{Discontinuities in the aggregate response}
\label{sec:degeneracy}

A striking property of distributions induced by standard microfoundations in response to a binary classifier is that they are necessarily either trivial or discontinuous. The underlying cause is that agents behaving according to standard microfoundations either change their features exactly up to the decision boundary, or they do not change their features at all.

\begin{restatable}{proposition}{propdegenerate}
\label{prop:degenerate}
\propositionlabel{prop:degenerate}
Given a base distribution $\DBase$, let $\DMap(\theta)$ be the distribution induced by a classifier $f_\theta$. Then, if $\DMap(\theta)$ is continuous and $\DMap(\theta) \neq \DBase$, there does not exist a valid cost function $c$ such that $\DMap(\theta)$ is an aggregate of agents following standard microfoundations.
\end{restatable}

In addition to the discontinuities implied by Proposition \ref{prop:degenerate}, the aggregate response induced by standard microfoundations faces additional degeneracies. Namely, a similar argument shows that any non-trivial distribution arising from the standard model must have a region of zero density below the decision boundary. 

These properties are unnatural in a number of practical applications, as we discuss in the following examples. 
\begin{example}[Lending decisions and credit scores]
\label{ex:FICO}
Consider banking lending decisions and the corresponding distribution over FICO credit scores. If lending decisions are based on FICO scores, then under standard microfoundations, the distribution over credit scores should exhibit a discontinuity at the threshold. However, this is not what we observe empirically. In particular, previous work \citep{HPS16} studied a FICO dataset from 2003, where credit scores range from 300 to 850, and a cutoff of 620 was commonly used for prime-rate loans. The observed distribution over credit scores appears continuous and is supported across the full range of scores.\footnote{In practice, lending decisions might be based on additional features beyond credit scores, and different lenders might use different cutoffs. In any case, this example demonstrates that the observed aggregate response cannot be captured by agents behaving according to standard microfoundations in response to a decision rule that is a threshold function of the credit score.
}
\end{example}
\begin{example}[Yelp online ratings and rounding thresholds]
\label{ex:yelp}
Restaurant ratings on Yelp are rounded to the nearest half star, and the star rating of a restaurant can significantly influence restaurant customer flows. In this setting, strategic adaptation arises from restaurants leaving fake reviews. Under standard microfoundations, the distribution of restaurant ratings would exhibit discontinuities at the rounding thresholds. However, previous work \citep{AM12} examined the distribution of restaurant ratings, and showed that there is no significant discontinuity in the density of the restaurants at the rounding thresholds (see Figure 4 in their work). 
\end{example}
\begin{example}[High school exit exams and score cutoffs]
\label{ex:testscores}
New York high school exit exams have important stakes for students, teachers, and schools, based on whether students meet designated score cutoffs. Interestingly, test score distributions did exhibit discontinuities prior to reforms on teacher grading procedures in 2012 \citep{dee19}. These discontinuities resulted from teachers deliberately adjusting student scores to be just above the cutoff during grading. This example demonstrates that sharp discontinuities can arise when there is perfect information about the decision rule and perfect control over manipulations. However, following grading reforms that largely eliminated the possibility for strategic manipulation by \emph{teachers}, these discontinuities disappeared and test score distributions appeared  continuous \citep{dee19}. Similar to observations in Examples~\ref{ex:FICO}-\ref{ex:yelp}, strategic adaptation by \textit{students} cannot be described by standard microfoundations. 
\end{example}

It is important to note that the degeneracies in the aggregate response induced by standard microfoundations arise from the fact that classification decisions are discrete and based on a \emph{hard} decision. Agents who are not classified positively receive no reward: it does not matter how close to the decision boundary the agent is. This discontinuity in the utility is specific to classification and does not arise in regression problems that are predominantly studied in the economics literature. 
However, in machine learning and statistical decision theory, binary classification is ubiquitous, and degeneracies that we have identified pertain to general settings where the decisions are binary. 

The reader might imagine that common variations and generalizations of standard microfoundations can mitigate these issues. Unfortunately, the two variations of standard microfoundations that are typically considered---\textit{heterogeneous cost functions} \citep{HIW19}, and \textit{approximate best response} \citep{MMH20}---result in similar degeneracies. Heterogeneity in the cost (or utility) function can only change whether or not an agent decides to change their features, but it does not change their target of manipulation. If agents approximately best-respond, and thus move to features $x'$ that approximately maximize their utility, the model no longer leads to point masses at the decision boundary, but agents will never \textit{undershoot} the decision boundary. This means that any nontrivial aggregate distribution must have a region of zero density below the decision boundary to comply with standard microfoundations and any of these variants. 

In fact, agent behavior that is not consistent with standard microfoundations or variants has been observed in field experiments. In particular, agents both overshoot and undershoot the decision boundary as well as generally exhibit noisy responses, even if the classifier is fully transparent.
\begin{example}[Field Experiment \citep{BBK20}]
\label{ex:FE}
The authors developed an app that mimicked aspects of ``digital credit'' applications, and deployed it in Kenya in order to empirically investigate strategic behavior. Participants were rewarded if the app guessed that they were a high-income earner. When the participants were given access to the coefficients of the decision rule, they tended to change their features in the right direction, but a high variance in their responses was observed---see Table 5 in their work. The noise in the response was even more pronounced when participants were only given opaque access to the decision rule. In this case, agents often did not even change their features in the right direction.
\end{example}

\subsection{Brittleness under natural model misspecifications} 
\label{sec:misspecifications}

We describe two scenarios where the behavior of standard microfoundations is undesirable under natural model misspecifications. In particular, the existence of stable solutions crucially relies on \emph{all} agents being perfectly strategic, and the optimal solutions associated with the standard model cause extreme negative externalities within a broad class of alternative approaches to model agent behavior.

\subsubsection{Stability as a fragile solution concept} 
\label{subsubsec:fragilestability}

Performatively stable solutions are guaranteed to exist under standard microfoundations (see \cite{MMDH19}). Our first result demonstrates that this no longer holds if any positive fraction of randomly chosen individuals are non-strategic. 

Recall that performative stability \citep{PZMH20} requires that a classifier is optimal on the data distribution that it induces: that is, that $\thetaPS$ is a global optimum of the following optimization problem:
\begin{equation}
    \min_{\theta \in \Theta} \;\mathbb{E}_{(x,y)\sim \cD(\thetaPS)} \;\Indicator\left\{y\neq f_\theta(x) \right\}.
    \label{eq:localPS}
\end{equation}
Since we focus on the 0-1 loss in this work, the objective in \eqref{eq:localPS} need not be convex, and it is natural to consider a local relaxation of performative stability. We say $\thetaPS$ is \emph{locally stable} if $\thetaPS$ is a local minimum or a stationary point of \eqref{eq:localPS}. Note that any performatively stable solution is locally stable. 

Performative stability defines the fixed points of \textit{repeated risk minimization}---the retraining method where the decision-maker repeatedly updates the classifier to a global optimum on the data distribution induced by the previous classifier. Thus, there is no incentive for the decision-maker to deviate from a stable model based on the observed data. Similarly, when the objective in \eqref{eq:localPS} is differentiable, locally stable points correspond to fixed points of  \textit{repeated gradient descent}~\citep{PZMH20}. Another interesting property of performatively stable points is that they closely relate to the concept of a pure strategy (local)\textit{ Nash equilibrium} in a simultaneous game between the strategic agents who respond to the classifier $f_\theta$ and the decision-maker who responds to the observed distribution $\cD(\theta)$.\footnote{More precisely, when agents follow standard microfoundations, any pure strategy local Nash equilibrium is locally stable.}

To showcase that the existence of locally stable classifiers under standard microfoundations crucially relies on all agents following the modeling assumptions, we focus on the following simple 1-dimensional setting.
\begin{setup}[1-dimensional]
\label{ex:stablepointsnotexist}
\setuplabel{ex:stablepointsnotexist}
Let $X \subseteq \mathbb{R}$ and consider a threshold functions $f_{\theta}(\noargument) =\mathds{1}\{\noargument \ge \theta\}$ with  $\theta\in \Theta\subseteq\mathbb R$. Let $\mu(x)$ be the conditional probability over $\DBase$ of the true label being $1$ given features $x$. Suppose that $\mu(x)$ is strictly increasing in $x$ and there is an $\theta\in \Int(\Theta)$ such that  $\mu(\theta)=0.5$.
\end{setup}

\begin{restatable}{proposition}{propstablepoints}
\vspace{-0.1cm}
\label{propo:stablepointsnotexist}
Consider Setup~\ref{ex:stablepointsnotexist}. Suppose that a $p$ fraction of agents drawn from $\DBase$ do not ever change their features, and a $1-p$ fraction of agents drawn independently from $\DBase$ follow standard microfoundations with a valid cost function $c$. Then, we have the following properties:
\begin{itemize}[itemsep=0ex, topsep=-0.2ex]
    \item[a)] For $p \in\{ 0,1\}$, locally stable points exist.
    \item[b)] For $p\in (0,1)$, locally stable points do not exist.
    \item[c)] Let $\thetaPS^{\mathrm{SM}}$ denote the smallest locally stable point when all agents follow standard microfoundations, and let $\theta_{\mathrm{SL}}$ denote the optimal classifier for the base distribution $\DBase$. 
    
    For $p\in (0,1)$, repeated risk minimization will oscillate between $\theta_{\mathrm{SL}}$ and a threshold $\tau(p) \in (\theta_{\mathrm{SL}}, \thetaPS^{\mathrm{SM}})$.The threshold $\tau(p)$ is decreasing in $p$, approaching $\theta_{\mathrm{SL}}$ as $p \rightarrow 1$ and $\thetaPS^{\mathrm{SM}}$ as $p \rightarrow 0$. 
\end{itemize}
\end{restatable}
\begin{proof}[Proof Sketch]
For $p = 1$, it is easy to see that $\theta_{\mathrm{SL}}$ is the unique locally stable point, and for $p = 0$, the claim follows from an argument similar to Lemma 3.2 in~\citep{MMDH19}. For $p\in(0,1)$, the core observation is that for any $\theta$ the distribution $\cD(\theta)$ contains no \emph{strategic} agents in the interval $\mathsf{Gap}(\theta):=[\theta-\Delta,\theta]$ for some $\Delta>0$. Furthermore, for any $\theta>\theta_{\mathrm{SL}} $  the misclassification rate on \emph{non-strategic} agents could be improved by reducing the threshold to $\theta_{\mathrm{SL}}$. Thus, it is not hard to see that or any $\theta>\theta_{\mathrm{SL}}$, the threshold $\max(\theta - \Delta, \theta_{\mathrm{SL}})$ achieves smaller loss than $\theta$, and thus $\theta$ cannot be stable. We formalize this argument and the case for $\theta\leq\theta_{\mathrm{SL}}$ in Appendix~\ref{app:nostablepoints}. 
\end{proof}

\begin{figure*}[t!]
\subfigure[\footnotesize{SM mixed with non-strategic agents}]{
\includegraphics[width = 0.48\textwidth]{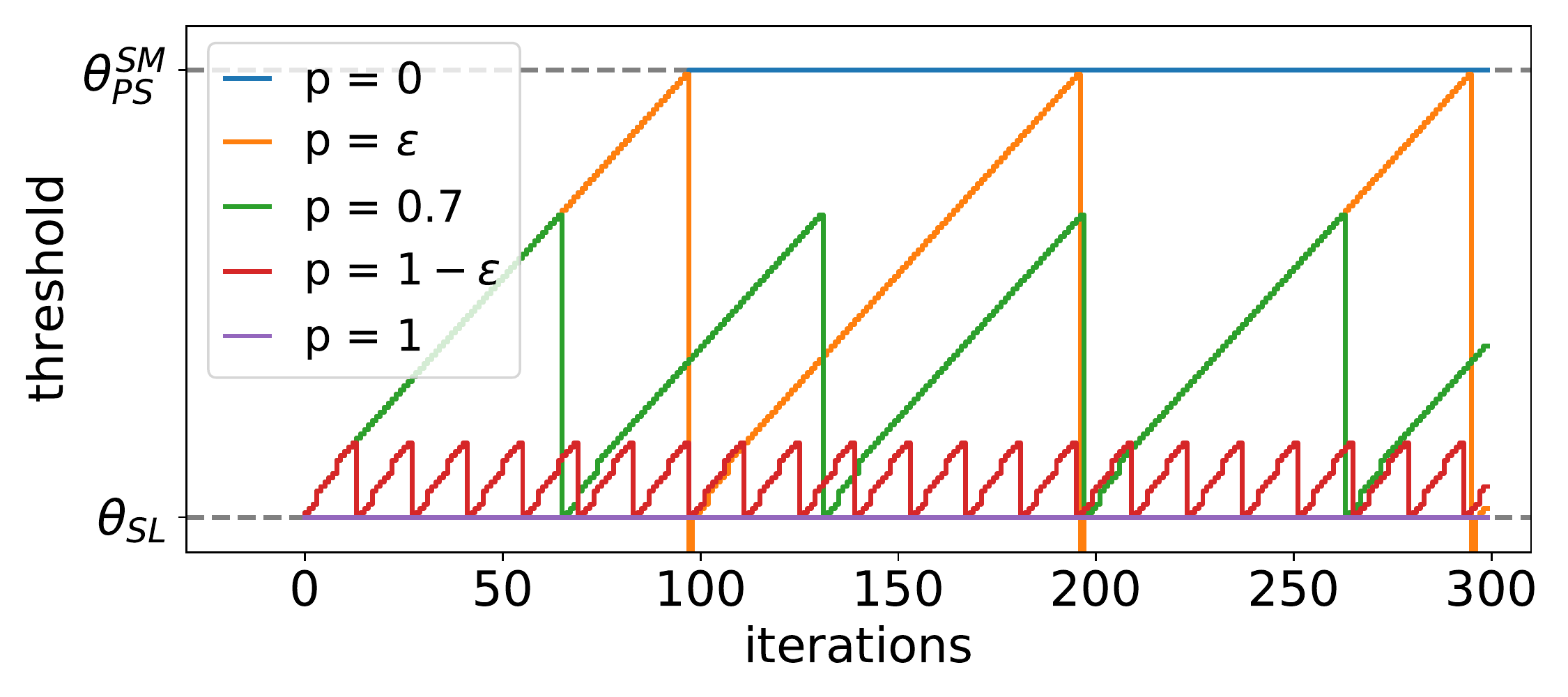}
\label{fig:oscillationBR}}
\subfigure[\footnotesize{NR mixed with non-strategic agents}]{
\includegraphics[width = 0.48\textwidth]{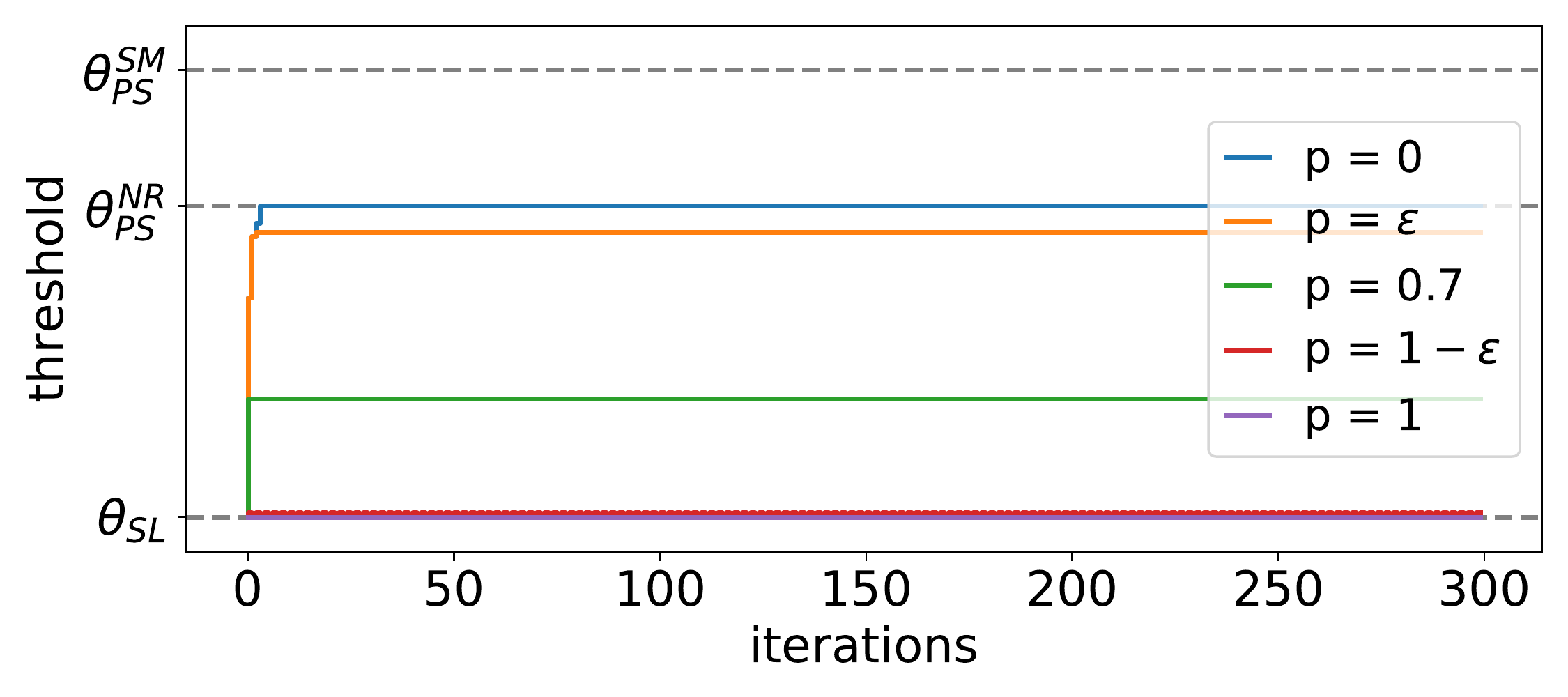}
\label{fig:oscillation-fuzzy}}
\caption{Convergence of retraining algorithm in a 1d-setting for different values of $p$ with $\epsilon=10^{-2}$. The population consists of $10^{5}$ individuals. Half of the individuals are sampled from $x\sim\mathcal N(1,0.33)$ with true label $1$ and the other half is sampled from $x\sim\mathcal N(0,0.33)$ with true label $0$. $\thetaPS^{\mathrm{SM}}$ and $\theta_{\mathrm{SL}}$ are defined as in Proposition~\ref{propo:stablepointsnotexist}, and $\thetaPS^{\mathrm{NR}}$ is defined to be the smallest locally stable point when all agents follow noisy response (NR).  The parameter of the noisy response in (b) is taken to be $\sigma^2 = 0.1$. }

\label{fig:oszillation}
\end{figure*}

Proposition~\ref{propo:stablepointsnotexist} implies that not only does the existence of locally stable points break down if a positive fraction $p\in(0,1)$ of randomly chosen agents are non-strategic, but also repeated risk minimization oscillates between two extreme points. To illustrate this, we have implemented a simple instantiation of Setup~\ref{ex:stablepointsnotexist}, and we visualize the trajectories of repeated risk minimization for different values of $p$ in Figure~\ref{fig:oscillationBR}. The main insight is that repeated risk minimization starts oscillating substantially even when $p$ is very close to $0$ (only an $\epsilon$ fraction of agents are not following standard microfoundations), even though this method converges when $p = 0$. This sensitivity of the trajectory to natural deviations from the modeling assumptions suggests that standard microfoundations do not constitute a reliable model to study algorithm dynamics. 

\begin{remark}
Unlike repeated risk minimization, repeated gradient descent is not well-defined for standard microfoundations. Because of the induced discontinuities, the optimization objective in \eqref{eq:localPS} for Setup \ref{ex:stablepointsnotexist} with $p < 1$ is not differentiable in the classifier weights. Thus, standard microfoundations do not serve as a useful basis to study repeated gradient descent.
\end{remark}

\subsubsection{Maximal negative externalities at optimality}\label{subsec:extreme}
Our next result describes a natural scenario where performatively optimal classifiers computed under standard microfoundations lead to the highest negative externalities within a broad class of alternative models for agent responses.

Recall that a performatively optimal solution corresponds to the best classifier for the decision-maker from a global perspective, but it is not necessarily stable under retraining. Formally, a classifier $\thetaPO$ is \textit{performatively optimal}~\citep{PZMH20} if it minimizes the performative risk:

\begin{equation}  
\thetaPO:=\argmin_{\theta\in\Theta} \mathbb{E}_{(x,y)\sim\cD(\theta)} \,\mathds{1}\{y\neq f_\theta(x)\}.
    \label{eq:PO}
\end{equation}

\noindent
Performative optimality is closely related to the concept of a \textit{Stackelberg equilibrium} in a leader-follower game, where the decision-maker plays first and the agents respond.

The key challenge of computing performative optima is that optimizing \eqref{eq:PO} requires the decision-maker to anticipate the population's response $\cD(\theta)$ to any classifier $f_\theta$. A natural approach to model this response is to build on microfoundations and deduce properties of the distribution map from individual agent behavior. Different models for agent behavior can lead to solutions with qualitatively different properties.

While the decision-maker is unlikely to have a fully specified model for agent behavior at hand, we outline a few natural criteria that agent responses could reasonably satisfy. 
To formalize these criteria, we again focus on the 1-dimensional setting. 
\begin{property}[Expenditure monotonicity]
\label{property:SB}
For every feature vector $x \in X$, any agent $a$ with true features $x$ must have manipulated features $\Response_a(x; \theta)$ in response to each classifier $f_{\theta}$ that satisfy: 
\begin{enumerate}[itemsep=-0.5ex, topsep=0.5ex]
    \item[a)]  $c(x,\mathcal R_a(x;\theta)) < \gamma $ for every $\theta \in \Theta$. 
    \item[b)] $f_{\theta}(\mathcal R_a(x;\theta)) = 1$ $\implies$ $f_{\theta'}(\mathcal R_a(x;\theta')) = 1$ $\forall \theta' \le \theta$. 
\end{enumerate} 
\end{property}
\noindent Property~\ref{property:SB} describes agents that a)  do not expend more on gaming than their utility from a positive outcome, and b) do not have their outcome worsened if the threshold is lowered. However, agents complying with Property \ref{property:SB} do not necessarily behave according to standard microfoundations. For example, Property \ref{property:SB} is satisfied by non-strategic agents who do not ever change their features and by the imperfect agents that we describe in Section \ref{sec:framework}. 

We now show that within the broad class of microfoundations that exhibit Property~\ref{property:SB}, the standard model leads to an extreme acceptance threshold. For the formal statement, see Appendix~\ref{app:socialburden}.

\begin{proposition}[Informal]
\label{prop:socialburdeninformal}
Consider Setup \ref{ex:stablepointsnotexist}. Let $\mathscr{D}$ be the class of distribution maps $\cD:\Theta\rightarrow \Delta(X\times Y)$ that can be represented by a population of agents who all satisfy Property~\ref{property:SB}. Then under mild assumptions, for every distribution map $\cD\in\mathscr{D}$, it holds that \vspace{-0.2cm}
\begin{align*}
    &\quad\quad\quad\thetaPO(\DMap_{\text{SM}}) \geq \thetaPO(\DMap)
\end{align*}
where $\DMap_{\text{SM}}$ is the distribution map induced by standard microfoundations.
\end{proposition}

A problematic implication of Proposition~\ref{prop:socialburdeninformal} is that standard microfoundations also maximize the negative externality called \emph{social burden}~\citep{MMDH19}:
\[\SocialBurden{\theta}:=\mathbb{E}_{(x,y) \in \DBase} \left[\min_{x' \in X} \left\{c(x, x') \mid f_{\theta}(x') = 1\right\} \mid y = 1\right].\]
Social burden quantifies the average cost that a positively labeled agent has to expend in order to be positively classified by $f_\theta$. While previous work introduced and studied social burden within standard microfoundations, and showed that Nash equilibria lead to smaller social burden than Stackelberg equilibria, we instead use social burden as a metric to study implications of different modeling assumptions on agent behavior. In particular, the following corollary demonstrates that standard microfoundations lead to \textit{worst possible social burden} across all microfoundations that satisfy Property~\ref{property:SB}. 

\begin{corollary}
\label{coro:socialburdeninformal}
Under the same assumptions as Proposition~\ref{prop:socialburdeninformal}, for every distribution map $\cD\in\mathscr{D}$, it holds that \vspace{-0.1cm}
\begin{align*}
    &\SocialBurden{\thetaPO(\DMap_{\text{SM}})} \geq \SocialBurden{\thetaPO(\DMap)}.
\end{align*}
where  $\DMap_{\text{SM}}$ is the distribution map induced by standard microfoundations.
\end{corollary}

This result has implications for the natural situation where standard microfoundations do not exactly describe agent behavior. In particular, relative to the performative optimal point of the true agent responses, the solutions computed using standard microfoundations would not only experience suboptimal performative risk but also would cause unnecessarily high social burden. This makes it hard for the decision-maker to justify the use of standard microfoundations as a representative model for agent behavior. Implicit in our argument is the following moral stance: \textit{given a set of criteria for what defines a plausible model for microfoundations, the decision-maker should not select the one that maximizes negative externalities}.

\section{Alternative microfoundations}\label{sec:agnostic}

We now depart from the classical approach and systematically search for models that are more appropriate for binary classification. We define the space of all possible alternative microfoundations and collect a set of useful properties that we show are desirable for microfoundations to satisfy.

\subsection{Defining the space of alternatives}
The principle behind microfoundations for strategic classification is to equip the distribution map with structure by viewing the distribution induced by a decision rule as an aggregate of the responses of individual agents. We consider agent responses in full generality by introducing a family of response types $\AllTypes$ that represents the space of all possible ways that agents can perceive and react to the classifier $f_\theta$. The response type fully determines agent behavior through the \textit{agent response function} $\Response_t: X \times \Theta \rightarrow X$. In particular, an agent with true features $x$ and response type $t$ changes their features to $x'=\Response_t(x, \theta)$ when the classifier $f_\theta$ is deployed.
\begin{remark}
Using the language of agent response functions, non-strategic agents correspond to a response type $t_{\text{NS}}$ such that $\Response_{t_{\text{NS}}}(x, \theta) = x$ for all $\theta \in \Theta$, and standard microfoundations correspond to a response type $t_{\text{SM}}$ where $\Response_{t_{\text{SM}}}(x, \theta)$ is given by \eqref{eq:SM} for all $\theta \in \Theta$. Note that a population of agents could be heterogeneous and exhibit a mixture of different types, or even be described by a \textit{continuum} of response types.
\end{remark}

We formalize microfoundations through a \textit{mapping} $M: X \times Y \rightarrow \AllTypes$ from agents to response types. We denote the set of possible mappings $M$ by the collection $\cM$ that consists of all\footnote{These mappings are subject to mild measurability constraints that we describe in detail in Appendix \ref{subsec:requirements}.} possible randomized functions $X \times Y \rightarrow \AllTypes$. For example, standard microfoundations correspond to the mapping $M_{\textsc{SM}} \in \cM$ such that $M_{\textsc{SM}}(x, y) = t_{\textsc{SM}}$ for all $x, y \in X \times Y$, and a population of non-strategic agents corresponds to the mapping $M_{\textsc{NS}} \in \cM$ such that $M_{\textsc{NS}}(x, y) = t_{\textsc{NS}}$ for all $x, y \in X \times Y$. While these two homogeneous populations can be captured by deterministic mappings, randomization is necessary to capture heterogeneity in agent responses across agents with the same original features. For example, randomness allows us to capture a \textit{mixed} population of agents where some agents behave according to standard microfoundations and other agents are non-strategic. 

Conceptually, the mapping $M\in\cM$ sets up the rules of agent behavior. One aspect that distinguishes our framework from typical approaches to microfoundations in the economics literature is that it directly specifies agent responses, rather than specifying an underlying behavioral mechanism. An advantage of this approach is that responses can be observed, whereas the behavioral mechanism is harder to infer.

Importantly, the mapping $M$ coupled with the base distribution $\DBase$ provides all the necessary information to specify the population's response to a classifier $f_\theta$. In particular, for each $\theta \in \Theta$, the \textit{aggregate response} $\cD(\theta; M)$ is the distribution over $(\Response_t(x, \theta), y)$ where $(x, y) \sim \DBase$ and $t=M(x,y)$. We use the notation $\cD(\cdot; M): \Theta \rightarrow \Delta(X \times Y)$ to denote the aggregate response map induced by $M$. By defining the distribution map, the mapping $M$ thus provides sufficient information to reason about the performative risk; in addition, it also provides sufficiently fine-grained information about individuals to reason about metrics such as social burden. 

Naturally, with such a flexible model, \textit{any} distribution map can be microfounded, albeit with complex response types, as long as feature manipulations do not change the fraction of positively labeled agents in the population. 
We refer to Appendix \ref{appendix:microfound} for an explicit construction of $M$.
\begin{restatable}{proposition}{microfound}
\label{prop:microfound}
Let $\DBase$ be a non-atomic distribution. Let $\DMap(\theta)$ be any distribution map that preserves the marginal distribution over $Y$ of $\DBase$. Then, there exists a mapping $M \in \mathcal{M}$ such that $\cD(\noargument; M)$ is equal to $\cD(\noargument)$.
\end{restatable}

This result primarily serves as an existence result to show that our general framework for microfoundations can capture continuous distributions that are observed empirically (e.g. Examples \ref{ex:FICO}-\ref{ex:testscores}). 
In the following subsections, we focus on narrowing down the space of candidate models and describe two properties that we believe microfoundations should satisfy. 

\subsection{Aggregate smoothness}

The first property pertains to the induced distribution and its interactions with the function class. This aggregate-level property rules out unnatural discontinuities in the distribution map.
We call this property \emph{aggregate smoothness}, and formalize it in terms of the \textit{decoupled performative risk}~\citep{PZMH20}.

\begin{property}[Aggregate smoothness]
\label{prop:AS}
Define the decoupled performative risk induced by $M$ to be
$\PRDecoupledM{\theta}{ \theta'} := \mathbb{E}_{(x,y)\sim \cD(\theta;M)} [\;\Indicator\{y\neq f_{\theta'}(x)\}]$. For a given base distribution $\DBase$, a mapping $M$ satisfies \textit{aggregate smoothness} if the derivative of the decoupled performative risk  with respect to $\theta'$ exists and is continuous in $\theta$ and $\theta'$ across all of $\Theta$. 
\end{property}

Intuitively, the existence of the partial derivative of $\PRDecoupledM{\theta}{ \theta'}$ with respect to $\theta'$ guarantees that \textit{each distribution $\DMap(\theta; M)$ is sufficiently continuous (and cannot have a point mass at the decision boundary)}, and assuming continuity of the derivative we guarantee that \textit{$\DMap(\theta;M)$ changes continuously in $\theta$}. This connection between aggregate smoothness and continuity of the distribution map can be made explicit in the case of 1-dimensional features:
\begin{restatable}{proposition}{onedsmooth}
\label{prop:1dsmooth}
Suppose that $X \subseteq \mathbb{R}$, and let $\Theta \subseteq \mathbb{R}$ be a function class of threshold functions. Then, if the distribution map $\DMap(\noargument; M)$ has the following properties, the mapping $M$ satisfies aggregate smoothness w.r.t. $\Theta$:
\begin{enumerate}[itemsep=-0.5ex, topsep=-0.2ex]
    \item For each $\theta$, the probability density $p_{\theta}(x,y)$ of $\DMap(\theta; M)$ exists everywhere and is continuous in $x$. 
    \item For each $x, y$, the probability density $p_{\theta}(x,y)$ is continuous in $\theta$. 
\end{enumerate}
\end{restatable}
 
We believe that these two continuity properties are natural and likely to capture practical settings, given the empirical evidence in Examples \ref{ex:FICO}-\ref{ex:FE}. 
A consequence of aggregate smoothness is that it is sufficient to guarantee the existence of locally stable points. 
\begin{restatable}{theorem}{existencethm}
 \label{lemma:existence}
 Given a base distribution $\DBase$ and function class $\Theta$, for any mapping $M$ that satisfies aggregate smoothness, there exists a locally stable point. 
 \end{restatable}

In fact, this result implies that stable points exist under deviations from the model, as long as aggregate smoothness is preserved. Our next result shows that under weak assumptions on the base distribution this is the case for any mixture with non-strategic agents. For ease of notation, we formalize such a mixture through the operator $\Phi_p(M)$, where for $p \in [0,1]$, we let $\Phi_p(M(x,y))$ be equal to $t_{\text{NS}}$ with probability $p$ and equal to $M(x, y)$ otherwise. 
 
\begin{restatable}{proposition}{mixtures}
\label{cor:mixtures}
Suppose that the non-performative risk $\mathrm R(\theta) := \mathbb{E}_{(x,y) \in \DBase} \mathds{1}\{f_{\theta}(x) = y\}$ is continuously differentiable for all $\theta\in\Theta$. Then, for any $p \in [0,1]$, aggregate smoothness of a mapping $M$ is preserved under the operator $\Phi_p(M)$ .
\end{restatable}

Proposition \ref{cor:mixtures}, together with Theorem \ref{lemma:existence}, implies the robust existence of locally stable points under mixtures with non-strategic agents, for any microfoundations model that satisfies aggregate smoothness.

Conceptually, our investigations in this section have been inspired by the line of work on performative prediction~\citep{PZMH20,MPZH20} that demonstrated that regularity assumptions on the aggregate response alone can be sufficient to guarantee the existence of stable points for smooth, strongly convex loss functions. However, our results differ from these previous analyses of performative stability in that we instead focus on the 0-1 loss. In Appendix \ref{appendix:lipschitz}, we provide an example to show why the Lipschitzness assumptions  on the distribution map used in prior work are not sufficient to guarantee the existence of stable points in a binary classification setting.

\subsection{Constraint on manipulation expenditure} 
\label{sec:tractability} 
\label{subsec:microassumptions}

While aggregate smoothness focused on the population-level properties of the induced distribution, a model for microfoundations must also be descriptive of realistic agent-level responses in order to yield useful qualitative insights about metrics such as social burden or accuracy on subgroups. A minimal assumption on agent responses is that an agent never expends more on manipulation than the utility of a positive outcome. 
 \begin{property}[Expenditure constraint]
 \label{def:weakmicro}
Given a function class $\Theta$ and a cost function $c$, a mapping $M \in \mathcal{M}$ is \textit{expenditure-constrained} if $c(x, \Response_t(x,\theta))\leq\gamma$ for every $ \theta \in \Theta$ and every $t \in  \text{Image}(M)$. 
 \end{property}

This constraint is implicitly encoded in standard microfoundations and many of its variants. We have previously encountered the expenditure constraint in Section~\ref{sec:misspecifications}, where we showed that if $c$ is a valid cost function, then this property, together with a basic monotonicity requirement on agent's feature manipulations, defines a  set of microfoundations models among which the standard model achieves extreme social burden at optimality. In Section \ref{sec:framework} we will describe on one particular model for microfoundations within this set which results in a \emph{strictly} lower social burden than the standard model.

\paragraph{Reducing the complexity of estimating the distribution map.} Apart from defining a natural class of feasible  microfoundations models, an additional advantage of Property~\ref{def:weakmicro} is that it naturally constrains each agent's range of manipulations. This can significantly reduce the complexity of estimating the distribution map for a decision-maker who wants to compute a strategy robust classifier offline.

Assume the decision-maker follows a two-stage estimation procedure to estimate a performatively optimal point, similar to \citep{MPZ21}. First, they compute an estimate $\tilde M$ of the true mapping $M$ and infer $ \cD(\noargument; \tilde M)$ from the base distribution $\DBase$. Second, they assume the model reflects the true decision dynamics and approximate optimal points as follows:
\begin{align}
\label{eq:PO_M}
\thetaPO(\tilde M) &:=\argmin_{\theta\in\Theta} \mathbb E_{(x,y)\sim\cD(\theta;\tilde M)} \left[\mathds{1}\{y\neq f_\theta(x)\}\right].
\end{align}
Using a naive bound (see Lemma \ref{lemma:TVbound}) it is not difficult to see that it suffices to compute an estimate $\tilde M$ of $M$, such that $\sup_\theta \TV(\cD(\theta;\tilde M),\cD(\theta;M))\le \xi$  to guarantee that $\PR(\thetaPO(M))-\PR(\thetaPO(\tilde M) )\leq 2\xi$. 
However, achieving this level of accuracy fundamentally requires a full specification of the response types for every agent in the population. 

The expenditure constraint helps to make this task more tractable, in that the decision-maker only needs to estimate responses for a small fraction of the agents to achieve the same bound on the suboptimality of the obtained performative risk. To formalize this, let's assume the decision-maker can define a set $\Theta_0\subseteq \Theta$ that contains the performatively optimal classifier $\thetaPO(M)$. Then, given the implied restriction in the search space in \eqref{eq:PO_M}, the expenditure constraint enables us to restrict the set of covariates that are relevant for the optimization problem to 
 \begin{align}
\label{eq:S}
S(\Theta_0,c)&:= \cup_{\theta \in \Theta_0} \{x\in X: \exists x' \in X : f_{\theta}(x') \neq f_{\theta}(x) \wedge c(x, x') \le \gamma\}.
\end{align} The salient part $S(\Theta_0,c) \subseteq X$ captures all agents who are sufficiently close to the decision boundary for some $\theta\in\Theta_0$ so they are able to cross it without expending more than $\gamma$  units of cost. 

The subset $S(\Theta_0,c)$ can be entirely specified by the cost function $c$ and can be much smaller than $X$. To see this, consider the following example setting. 
\begin{example}[Informal]
\label{ex:salient}
Consider Setup \ref{ex:stablepointsnotexist}. Define the cost function to be linear: $c(x,x') = \alpha |x - x'|$ for  some $\alpha>0$. Then, if $M$ satisfies the expenditure constraint and Assumption \ref{assumption:gamingbehavior}, then:  
\[\thetaPO \in \Theta_0 := [\theta_{\text{SL}} - 3/\alpha, \theta_{\text{SL}} + 3/\alpha] \quad S(\Theta_0, c) = \left\{ x \mid \mu(x) \in [\theta_{\text{SL}}- 4/\alpha, \theta_{\text{SL}} + 4/\alpha]\right\},\]
where $\theta_{\text{SL}}$ is such that $\mu(\theta) = 0.5$ (where $\mu$ is defined in Setup \ref{ex:stablepointsnotexist}). 
\end{example}
\noindent Example \ref{ex:salient} demonstrates the salient part consists of agents who are sufficiently close to the supervised learning threshold, where closeness is measured by the cost function. We formalize this example in Appendix~\ref{appendix:construction}.

We now describe the implications of constraining to the salient part for a 1-dimensional setting where $X \subseteq \mathbb{R}$ and $f_\theta$ is a threshold function.\footnote{Proposition \ref{prop:samplecomplexity} directly extends to \textit{posterior threshold functions} \citep{MMDH19}.} Let us define an \textit{agent response oracle} that given $x$ and $\theta$, outputs a draw $x'$ from the response distribution $(\Response_t(x, \theta), y)$ where $(x, y) \sim \DBase$. We show with few calls to the oracle,  the decision-maker can build an sufficiently precise estimate of $M$.

\begin{restatable}{proposition}{oracle}
\label{prop:samplecomplexity} Let $X \subseteq \mathbb{R}$, let $\Theta \subseteq \mathbb{R}$ be the function class of threshold functions. Suppose that $M$ satisfies the expenditure constraint, the distribution map $\DMap(\cdot; M)$ is 1-Lipschitz with respect to TV distance\footnote{That is, for any $\theta, \theta' \in \Theta$, we have that $\TV\left(\DMap(\theta'; M), \DMap(\theta; M)\right) \le \norm{\theta - \theta'}_2$.}, and $\Theta_0\subseteq \Theta:\thetaPO(M)\in \Theta_0$. We further assume that an agent's type does not depend on their label, i.e., $M(x,0)=M(x,1)$ for all $x\in X$. 
Then, with $O\left(\zeta^2 \frac{\ln(1/\epsilon)}{2\epsilon^3}\right)$ calls to the agent response oracle, where $\zeta:= \mathbb{P}_{\DBase}[x \in S(\Theta_0, c)]$, the decision-maker can create an estimate $\tilde{M}$ so that: 
\[\PR(\thetaPO(\tilde M)) \le \PR(\thetaPO(M)) + \epsilon.\] 
with probability $0.9$ for any $\epsilon>0$.\footnote{We note that the decision-maker will only search over $\Theta_0$ in  \eqref{eq:PO_M} when computing $\thetaPO(\tilde{M})$. In particular, they compute $\argmin_{\theta\in\Theta_0} \mathbb E_{(x,y)\sim\cD(\theta;\tilde M)} \left[\mathds{1}\{y\neq f_\theta(x)\}\right]$ rather than $\argmin_{\theta\in\Theta} \mathbb E_{(x,y)\sim\cD(\theta;\tilde M)} \left[\mathds{1}\{y\neq f_\theta(x)\}\right]$.}
\end{restatable}
\noindent
The number of necessary calls to the response function oracle for estimating $M$ decays with $\zeta:=\mathbb{P}_{\DBase}[x \in S(\Theta_0, c)]$. Without any assumption on agent responses we have  $S(\Theta_0,c) = X$ and $\zeta=1$. However, when the decision-maker is able to constrain $S(\Theta_0, c)$ to a small part of the input space by relying on the expenditure constraint, domain knowledge, or stronger assumptions on agent behavior, $\zeta$ and thus the number of oracle calls can be reduced significantly.

The concept of a salient part bears resemblance to the approaches by \citet{ZC21}; \citet{ZCC21}, which directly specify the set of feature changes that an agent may make, rather than implicitly specifying agent actions through a cost function. While these models assume that agents best-respond, our key finding is that constraining agent behavior alone can lessen the empirical burden on the decision-maker.

\section{Microfoundations based on imperfect agents}
\label{sec:framework}

Using the properties established in the previous section as a guide, we propose an alternate model for microfoundations that naturally allows agents to undershoot or overshoot the decision boundary, while complying with aggregate smoothness and expenditure monotonicity. Furthermore, we show that this model, called \textit{noisy response}, leads to strictly smaller social burden than the standard model while retaining analytical tractablility. 

\subsection{Noisy response}
Noisy response captures the idea of an \textit{imperfect agent} who does not perfectly best-respond to the classifier weights. This imperfection can arise from many different sources---including interpretability issues, imperfect control over actions, or opaque access to the classifier. Inspired by \emph{smoothed analysis}~\citep{smoothedanalysis}, we do not directly specify the source of imperfection but instead capture imperfection in an agnostic manner, by adding small random perturbations to the classifier weights targeted by the agents. Since smoothed analysis has been successful in explaining convergence properties of algorithms in practical (instead of worst case) situations, we similarly hope to better capture empirically observed strategic phenomena. 

We define the relevant set of types $T_{\text{noisy}}\subset\AllTypes$ so that each type $t \in T_{\text{noisy}}$ is associated with a noise vector $\eta_t \in \mathbb{R}^m$. An agent of type $t\in T_{\text{noisy}}$ perceives $\theta$ as $\theta + \eta_t$ and responds to the classifier $f_\theta$ as follows: 
\begin{equation}
\Response_t(x, \theta) := \label{eq:RFP} \argmax_{x'\in X'} \; \left[\gamma \cdot f_{\theta + \eta_t}(x') - \costfn(x, x')\right],
\end{equation}
where $c$ denotes a valid cost function, $\gamma>0$ denotes the utility of a positive outcome, and $X' \subseteq \mathbb{R}^d$ is a compact, convex set containing $X$.\footnote{We assume that $c$ is defined on all of $X'\times X'$, and $c(x, x') > \gamma$ for all $x \in X$ and all $x'$ that are on the boundary of $X'$.} For each $(x,y) \in X \times Y$, we model the distribution over noise across all agents with feature, label pair $(x,y)$ as a multivariate Gaussian. To formalize this, we define a \textit{randomized} mapping $M_{\sigma}: X \times Y \rightarrow \AllTypes$ as follows. For each $(x,y)$, the random variable $M_\sigma(x,y)$ is defined so that if $t \sim M_\sigma(x,y)$, then $\eta_t$ is distributed as $\mathcal N (0, \sigma^2 I)$. This model results in the perceived values of $\theta$ across all agents with a given feature, label pair following a Gaussian distribution centered at $\theta$. The noise level $\sigma$ reflects the degree of imperfection in the population.

Conceptually, our model of noisy response bears similarities to models of \textit{incomplete information} \citep{harsanyi68incomplete} that are standard in game theory (but that have not been traditionally considered in the strategic classification literature). However, a crucial difference is that we advocate for modeling agents actions as imperfect even if the classifier is fully transparent, because we believe that imperfection can also arise from other sources. This is supported by the empirical study of \cite{BBK20} discussed in Example \ref{ex:FE} where agents act imperfectly even when the classifier weights are revealed. 

We want to emphasize that we instantiate imperfection by adding noise to the manipulation targets, instead of directly adding noise to the \textit{responses}. While both approaches would mitigate the discontinuities in the aggregate distribution, the approach of adding noise directly to the responses results in a less natural model for agent behavior that violates the expenditure constraint.

\subsection{Aggregate-level properties of noisy response}
\label{sec:limitA}
Intuitively, the noise in the manipulation target of noisy response smooths out the discontinuities of standard microfoundations, eliminating the point mass at the decision boundary and region of zero density below the decision boundary. We show this explicitly in a 1-dimensional setting. 
\begin{restatable}{proposition}{continuous}
\label{thm:continuous}
Let $\Theta \subseteq \mathbb{R}$ be a function class of threshold functions, and suppose also that $X \subseteq \mathbb{R}$. For any $\sigma \in (0, \infty)$, the distribution map $\DMap(\noargument; \Fuzzy{\sigma})$ satisfies the continuity properties in Proposition \ref{prop:1dsmooth}, and thus the mapping $\Fuzzy{\sigma}$ satisfies aggregate smoothness. 
\end{restatable}

\begin{remark}
This result implies that noisy response inherits the robust existence of stable points under mixtures with non-strategic agents from Theorem~\ref{lemma:existence}. Furthermore, we illustrate in Figure~\ref{fig:oscillation-fuzzy} how noisy response mitigates the large oscillations of repeated retraining that we observed for standard microfoundations. We observe that in the case of $p = 0$, noisy response results in a lower stable point than standard microfoundations.
\end{remark}

\begin{figure*}[t!]
\subfigure[\footnotesize{$\cD(\theta)$ of NR vs. SM}]{
\includegraphics[width = 0.31\textwidth]{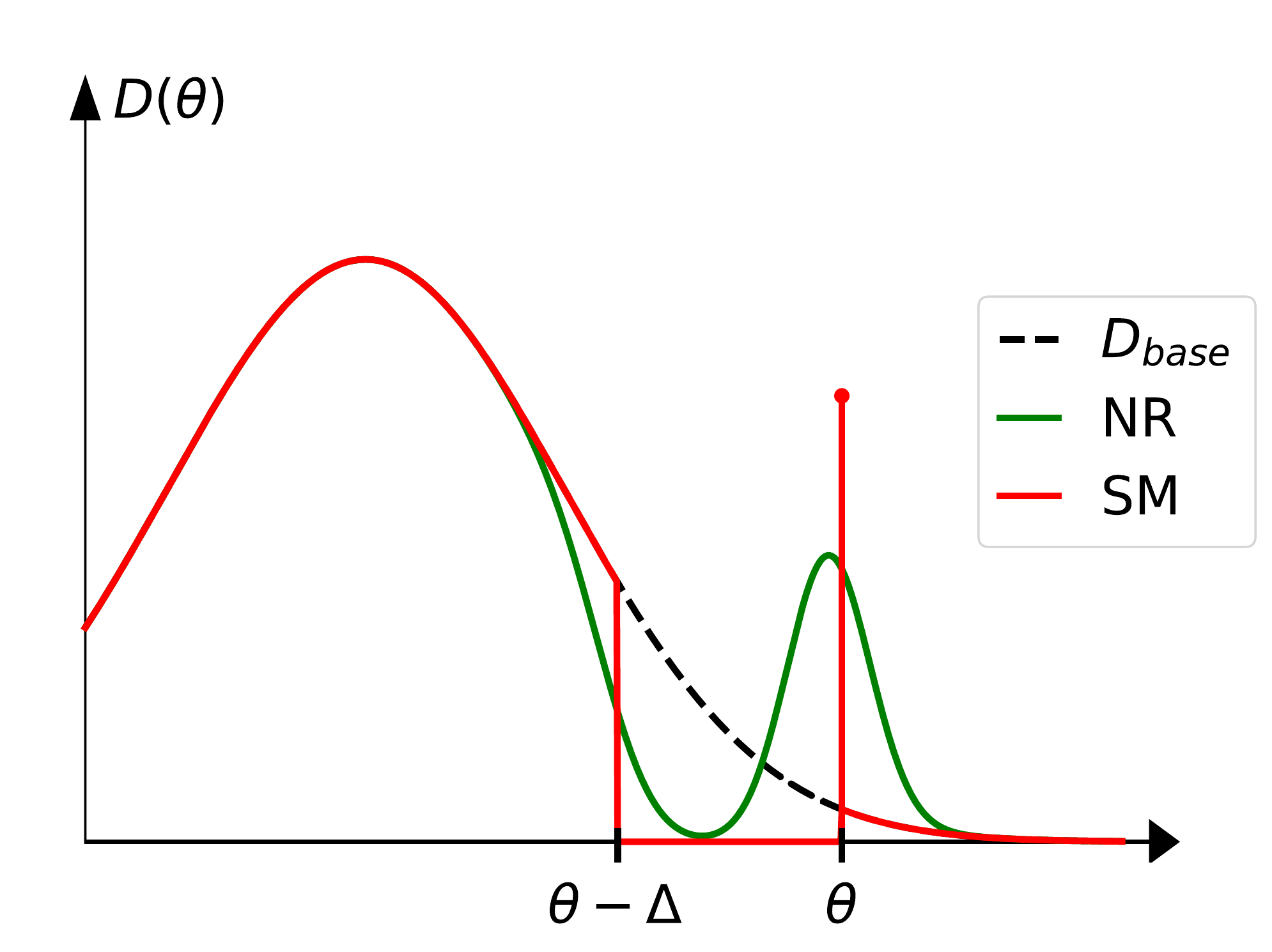}
\label{fig:response}}
\subfigure[\footnotesize{$\cD(\theta)$ of NR for different $\sigma$}]{
\includegraphics[width = 0.31\textwidth]{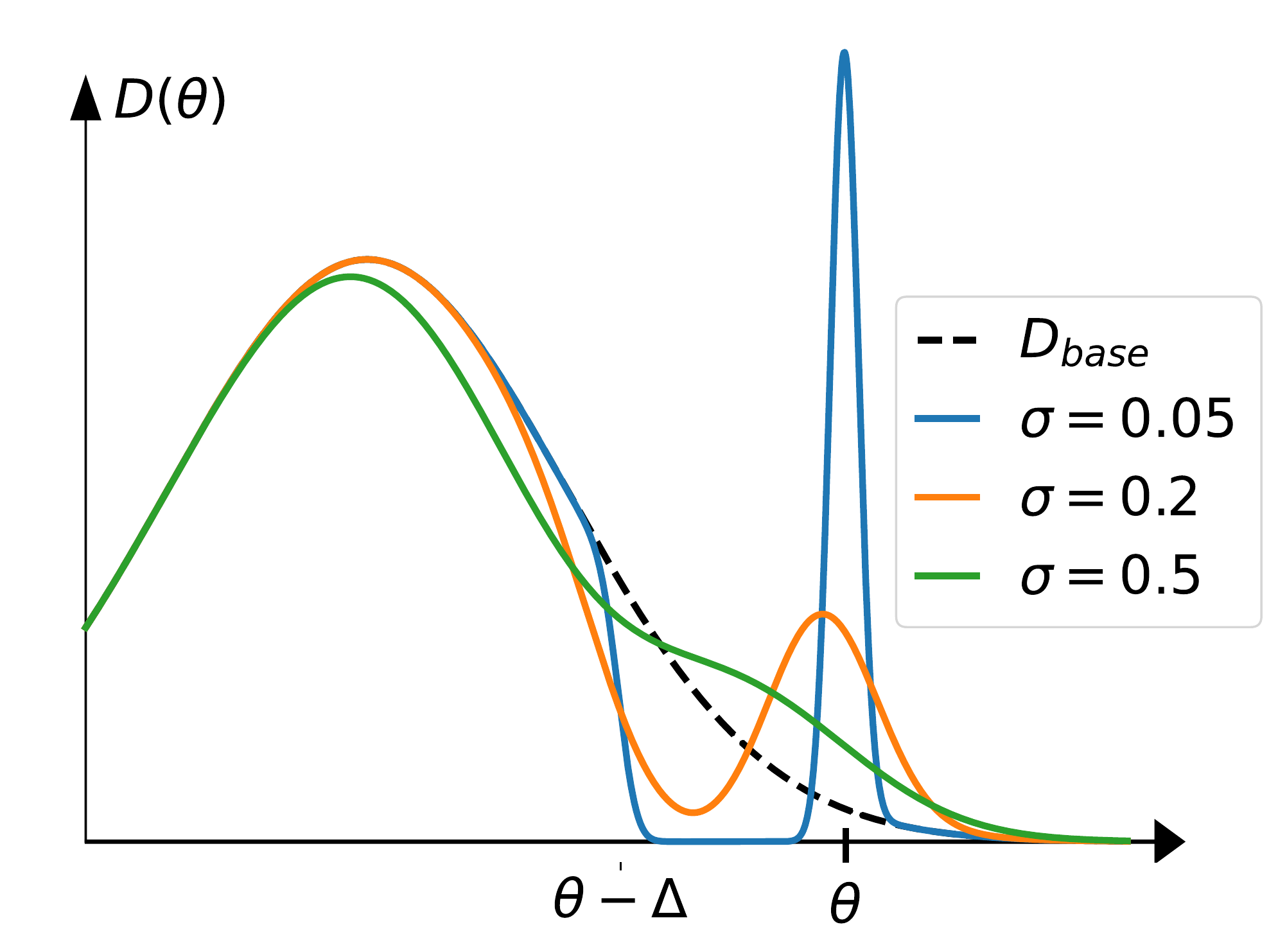}
\label{fig:sig}}
\subfigure[\footnotesize{$\cD(\theta)$ of NR for different $\theta$}]{
\includegraphics[width = 0.31\textwidth]{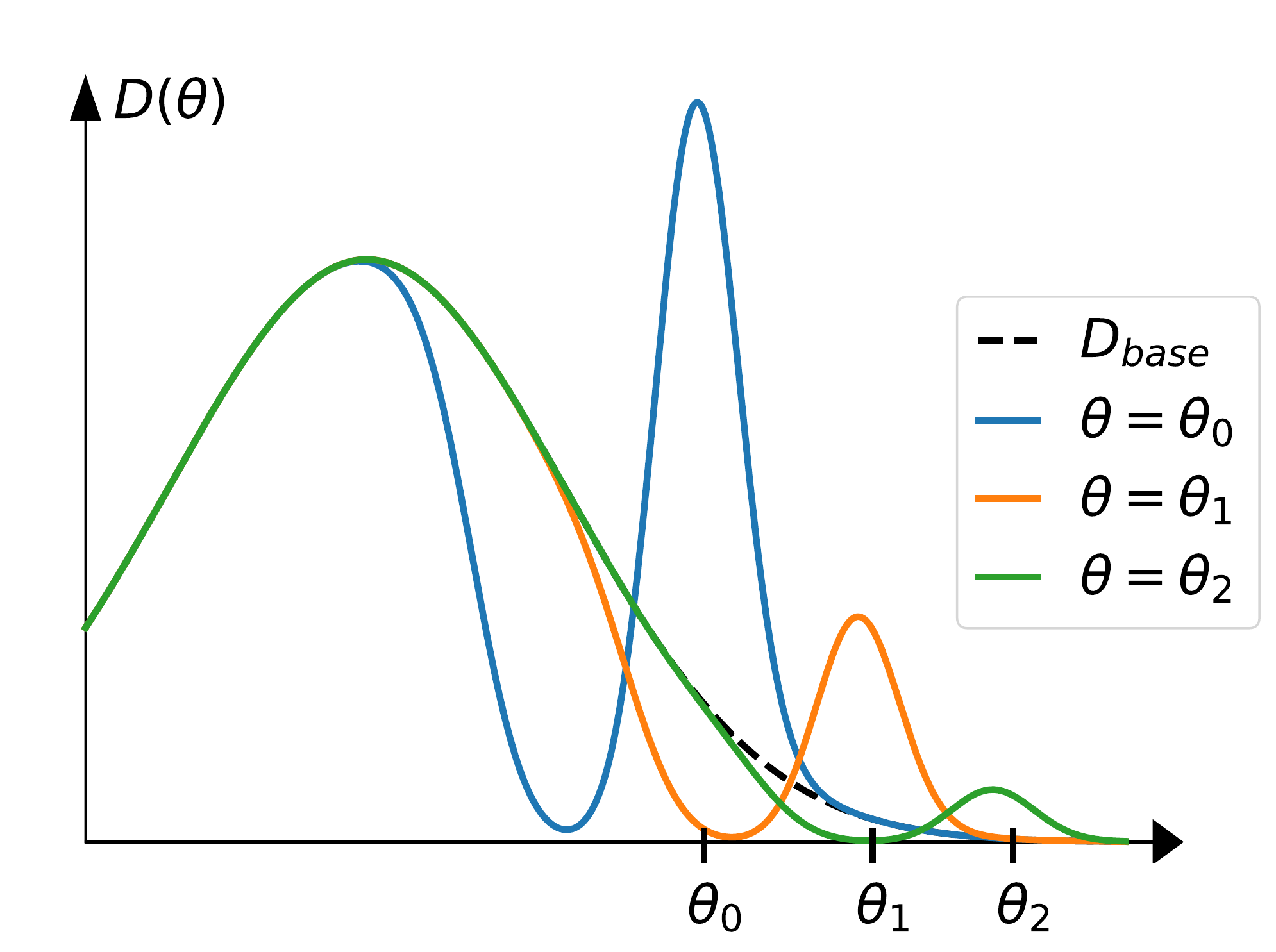}
\label{fig:theta}}
\vspace{-0.2cm}
\caption{Probability density of the aggregate response $\DMap(\theta)$ in a 1d-setting, where the base distribution  $\cD_{X}$ is a Gaussian with $x\sim\mathcal N(0,0.5)$. We illustrate (a) $\DMap(\theta)$ for a population of agents that follow noisy response (NR) compared to standard microfoundations (SM), (b) how $\DMap(\theta)$ of NR changes for different $\theta$, (c) variations in  $\DMap(\theta)$ of NR for different values of $\sigma$.}
\label{fig:pdf}
\end{figure*}

To visualize the aggregate-level properties of noisy response and contrast them with standard microfoundations we depict the respective density functions for a 1-dimensional setting with a Gaussian base distribution in Figure~\ref{fig:response}.  The distribution $\cD(\theta)$ can be bimodal, since agents closer to the threshold $\theta$ are more likely to change their features. The shape of the response distribution changes with $\sigma$ as illustrated in Figure~\ref{fig:sig}. As $\sigma \rightarrow 0$, the aggregate response of a population of noisy response agents approaches that of standard microfoundations. This means that noisy response maintains continuity while also being able to approximate the aggregate response of standard microfoundations to arbitrary accuracy. 
Finally, we note that the distribution map of noisy response changes \textit{continuously} with $\theta$, as visualized in Figure~\ref{fig:theta}. In fact, the distribution map induced by noisy response is Lipschitz in total-variation distance, where the Lipschitz constant grows with $1 / \sigma$.

\begin{restatable}{lemma}{tvlipschitz}
\label{lemma:tvlipschitz}
Given $\sigma\in(0,\infty)$, the distribution map $\DMap(\theta; M_{\sigma})$ is continuous in TV distance and supported on all of $X$. Moreover, the distribution map is Lipschitz in TV distance. That is, for any $\theta, \theta' \in \Theta$, we have that $\TV\left(\DMap(\theta'; \Fuzzy{\sigma}), \DMap(\theta; \Fuzzy{\sigma})\right) \le \frac{1}{2 \sigma} \norm{\theta - \theta'}_2$. 
\end{restatable}
\noindent This result highlights a favorable property of noisy response compared to standard microfoundations, in that the performative risk changes smoothly with changes in the classifier weights. 

\begin{remark}[Implications beyond 0-1 loss]
Lemma~\ref{lemma:tvlipschitz} implies that noisy response induces a distribution map that is Lipschitz in Wasserstein distance \textit{for any cost function}, where the Lipschitz constant depends on the diameter of the set $X'$. For smooth and strongly convex loss functions, this readily implies convergence of repeated retraining~\citep{PZMH20, MPZH20}.
\end{remark}

\subsection{Trade-off between imperfection and social burden}\label{sec:limitC}
Apart from satisfying desirable aggregate-level properties, noisy response also satisfies the expenditure monotonicity requirement in Property~\ref{property:SB} (see Proposition \ref{prop:expenditurerationality} for a proof). By Corollary~\ref{coro:socialburdeninformal}, this implies that in Setup \ref{ex:stablepointsnotexist} the social burden of the optimal classifier computed under noisy response is no larger than that of standard microfoundations. That is,  $\SocialBurden{\thetaPO(\Fuzzy{\sigma})} \le \SocialBurden{\thetaPO(M_{\mathrm{SM}})}$. 
In certain cases, we can obtain a stronger result and show that the social burden of noisy response is \emph{strictly} lower than that of standard microfoundations.
\begin{restatable}{corollary}{socialburdenpos}
\label{cor:socialburden}
Consider Setup \ref{ex:stablepointsnotexist}. Let $M_{\mathrm{SM}}$ be the mapping associated with standard microfoundations, let $\sigma \in (0, \infty)$, and let the cost function be of the form $c(x_1, x_2) = |x_1 - x_2|$. Suppose that $[\theta_\mathrm{SL}, \theta_{\mathrm{SL}} + 1] \in \Theta \cap X$, where $\theta_{\mathrm{SL}}$ is defined so that $\mu(\theta_{\mathrm{SL}}) = 0.5$. Then, it holds that:
\[\SocialBurden{\thetaPO(\Fuzzy{\sigma})} < \SocialBurden{\thetaPO(M_{\mathrm{SM}})}.\]
\end{restatable}

\begin{figure*}[t!]
\centering
\subfigure[\footnotesize{Optimal points}]{
\includegraphics[height = 0.245\textwidth]{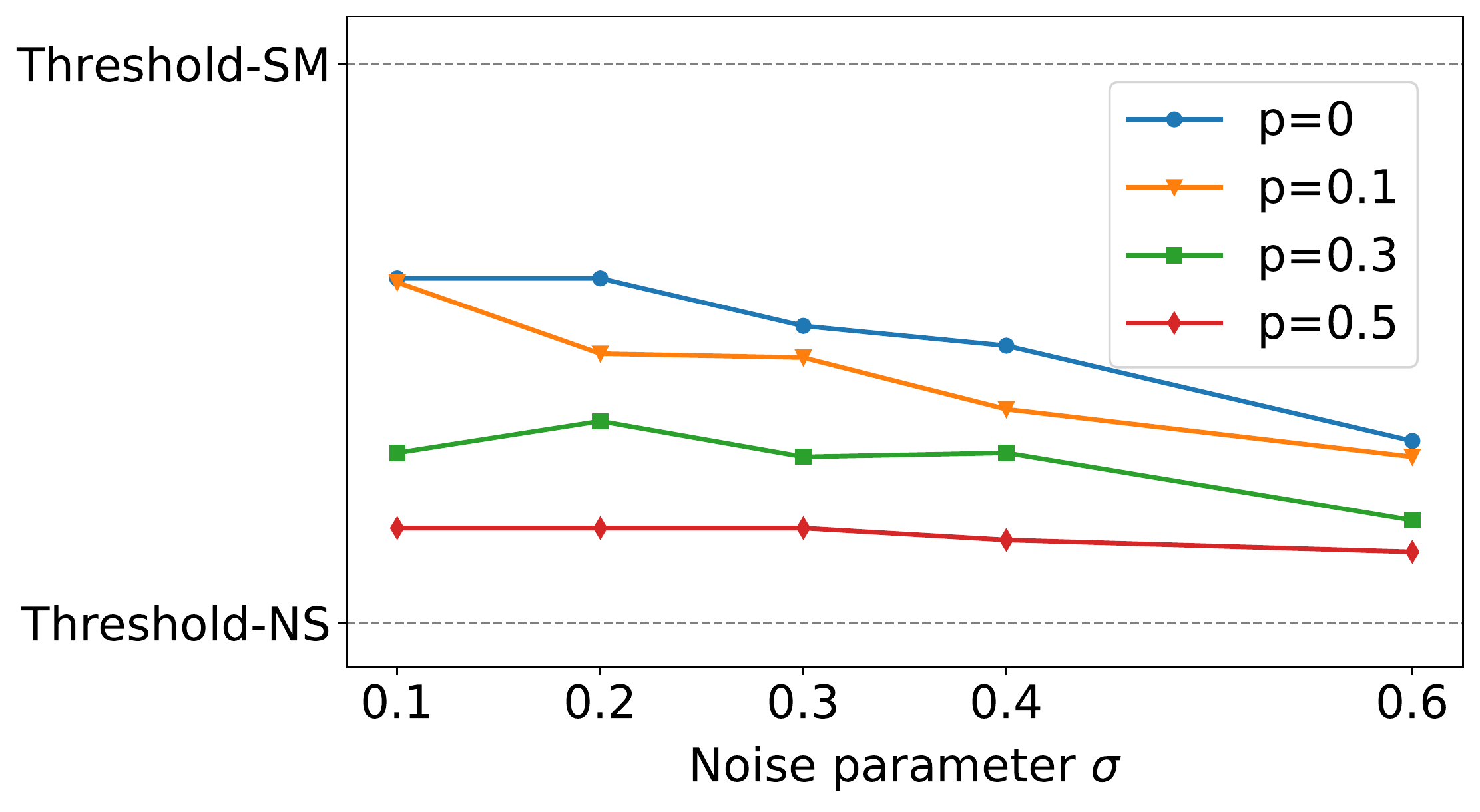}
\label{fig:thresholds}}\hspace{0.5cm}
\subfigure[\footnotesize{Social burden of optimal points}]{
\includegraphics[height = 0.245\textwidth]{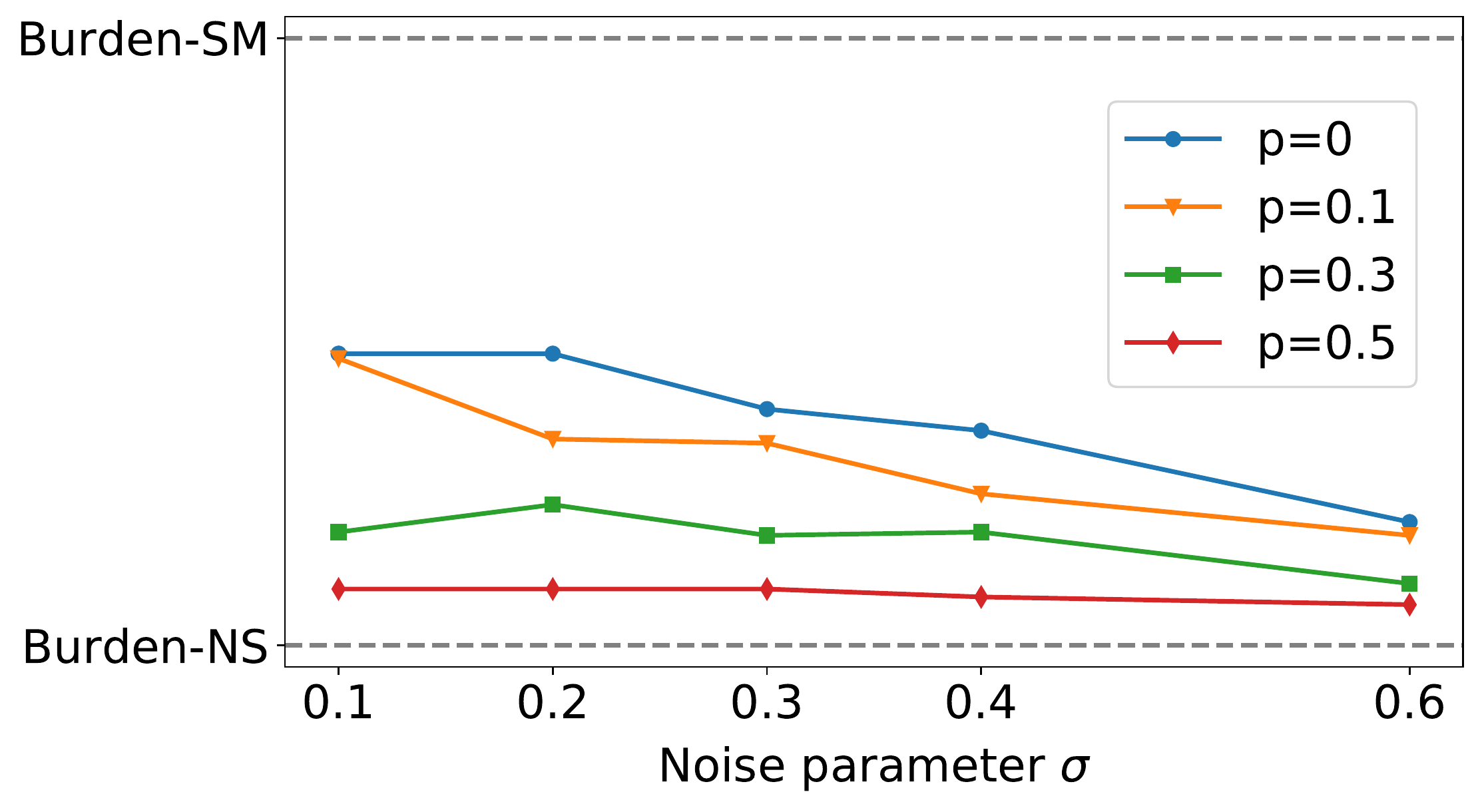}
\label{SB}}
\vspace{-0.2cm}
\caption{Optimal points in a 1d-setting evaluated for different values of $\sigma$ (noise in the response model) and $p$ (fraction of non-strategic agents). The population consists of $10^{5}$ individuals. Half of the individuals are sampled from $x\sim\mathcal N\big(1,\tfrac 1 3\big)$ with true label $1$ and the other half is sampled from $x\sim\mathcal N\big(0,\tfrac 1 3\big)$ with true label $0$. }
\label{fig:socialburdenthresholds}
\end{figure*}

In fact, the social burden for fuzzy perception can be well below the social burden of standard microfoundations. To demonstrate this we visualize the threshold and social burden across a variety of different values of the noise parameter $\sigma$ and the fraction of non-strategic agents $p$ in Figure \ref{fig:socialburdenthresholds}.  The dashed lines indicate the respective reference values for standard microfoundations (SM) and for a population of non-strategic agents (NS). We observe that the threshold as well as the social burden decrease with the fraction $p$ of non-strategic agents in the population. Furthermore, if every agent follows noisy response ($p = 0$), the threshold and the social burden are decreasing with $\sigma$ (and hence the degree of imperfection in agents response). Overall, the acceptance threshold and the social burden of optimal classifiers derived under our new microfoundations is significantly lower than for standard microfoundations.

\subsection{Approximating the distribution map for noisy response}\label{subsec:practical}

Microfoundations are useful in practice to model and anticipate how the population responds to a deployed classifier. This tool is particularly powerful if the decision-maker can estimate agent behavior without needing to expose the population to potentially inaccurate and harmful classifiers to explore and learn about agents responses.

An appealing aspect of assuming a parameterized model of the agent responses is that the complex task of learning agent behavior is reduced to a parameter estimation problem. For noisy response, the aggregate response $\DMap(\theta)$ is parameterized by the variance $\sigma$ of the target noise. In practice such parameters of individual responses can often be estimated via \textit{individual experiments}, i.e., by gathering information about individuals without ever deploying a classifier. For example, the decision-maker can randomly sample agents in the population, ask survey questions to learn about their perceptions of the deployed classifier, and infer the variance $\sigma$ from these samples. We refer to \citep{BBK20} for an actual field experiment that shows an example procedure for how to similarly obtain a reasonable estimate of the cost function $c$.
We now show that an error in parameter estimation can then be translated into an error in the aggregate response.

\begin{restatable}{lemma}{tvboundfuzzy}
\label{lemma:tvboundfuzzy}
Given a population of agents following noisy response with parameter $\sigma\in(0,\infty)$, and an estimate $\tilde{\sigma}$ of $\sigma$, it holds that $\TV\big(\DMap(\theta; \Fuzzy{\sigma}), \DMap(\theta; \Fuzzy{\tilde{\sigma}})\big) \le \frac 1 2 \sqrt{\frac {|\sigma^2 - \tilde{\sigma}^2| m}{\min(\sigma^2, \tilde{\sigma}^2)}}$ where $m$ is the dimension of the feature space. 
\end{restatable}
Combining Lemma \ref{lemma:tvboundfuzzy} with a bound on the performative risk, we obtain the following robustness guarantee of the performative risk to estimation errors in the noise parameter $\sigma$. 
\begin{restatable}{corollary}{fuzzyTV}
\label{cor:fuzzyTV}
Let the population be the aggregate of agents following noisy response with parameter $\sigma\in(0,\infty)$, and let $\tilde{\sigma} \in (0,\infty)$ be an estimate of the perception parameter ${\sigma}$. Then the suboptimality of the estimated performative risk of $\thetaPO( M_{\tilde \sigma})$ on the true population represented by $M_\sigma$ is bounded by: 
\[\PR(\thetaPO({M}_{\tilde \sigma}))- \PR(\thetaPO(M_\sigma))\leq \sqrt{\frac {|\sigma^2 - \tilde{\sigma}^2| m}{\min(\sigma^2, \tilde{\sigma}^2)}}. \] 
\end{restatable}
\noindent 
Hence, if the true distribution map is sufficiently close to some parameterization of noisy response, estimating the noise parameter $\sigma$ provides a robust procedure to infer an estimate of performative optima in practice.

Overall, noisy response offers a more descriptive and prescriptive model of agent behavior compared to standard microfoundations, and still maintains analytical tractability. While we have focused on Gaussian noise in the perception function throughout this work, the outlined benefits of noisy response also apply to other \textit{parameterized} noise distributions, as long as they are sufficiently smooth and continuous on all of $\mathbb{R}^m$. Hence, depending on the application, the decision-maker might prefer to pick a different noise model that can better capture the expected particularities of agents imperfections. The inference procedure via individual experimentation can then be adapted to obtain performative risk estimates that depend on the parameters of the noise distribution.

\section{Discussion}

To anticipate the strategic behavior of a population in response to a classifier, the decision-maker needs to understand and model the distribution shifts induced by decision rules. Traditional approaches for modeling these distribution shifts are either purely individual-level or purely population-level:  \emph{strategic classification} typically builds on standard microfoundations as a specification of individual-level behavior to deduce aggregate-level responses, whereas \textit{performative prediction} directly works with a population-level characterization of distribution shifts. 

In this work, we provided a fresh perspective on microfoundations for strategic classification that is inspired by performative prediction. 
While microfoundations can provide valuable structure to distribution shifts, we illustrated that the standard model for agent responses can lead to degenerate behavior in aggregate. These unanticipated degeneracies motivated us to consider population-level properties, rather than only individual-level properties, when searching through the space of alternative models for agent responses. This approach helped us identify noisy response as a promising candidate model for reasoning about strategic behavior in response to a binary decision rule.

While we have focused on strategic classification in this work, we believe that striving for a synergy between individual-level and aggregate-level considerations of agent behavior can more broadly lead to interesting insights about algorithmic decision-making in a social context.  

\section*{Acknowledgments}
We would like to thank Jacob Steinhardt and Tijana Zrnic for feedback on this manuscript. MJ acknowledges support from the Paul and Daisy Soros Fellowship. CMD acknowledges support from the Swiss National Science Foundation Postdoc.Mobility fellowship program. MH acknowledges support from NSF Award 1750555.

\bibliography{bib}

\begin{thebibliography}{44}
\providecommand{\natexlab}[1]{#1}
\providecommand{\url}[1]{\texttt{#1}}
\expandafter\ifx\csname urlstyle\endcsname\relax
  \providecommand{\doi}[1]{doi: #1}\else
  \providecommand{\doi}{doi: \begingroup \urlstyle{rm}\Url}\fi

\bibitem[Akyol et~al.(2016)Akyol, Langbort, and Basar]{akyol16transp}
Emrah Akyol, Cedric Langbort, and Tamer Basar.
\newblock Price of transparency in strategic machine learning.
\newblock \emph{Arxiv:1610.08210}, 2016.

\bibitem[Anderson and Magruder(2012)]{AM12}
Michael Anderson and Jeremy Magruder.
\newblock Learning from the crowd: Regression discontinuity estimates of the
  effects of an online review database.
\newblock \emph{The Economic Journal}, 122\penalty0 (563):\penalty0 957--989,
  2012.

\bibitem[Ball(2020)]{Ball2020}
Ian Ball.
\newblock Scoring strategic agents.
\newblock \emph{ArXiv:1909.01888}, 2020.

\bibitem[Bechavod et~al.(2021)Bechavod, Podimata, Wu, and
  Ziani]{bechavod2021information}
Yahav Bechavod, Chara Podimata, Zhiwei~Steven Wu, and Juba Ziani.
\newblock Information discrepancy in strategic learning.
\newblock \emph{Arxiv:2103.01028}, 2021.

\bibitem[Bj{\"{o}}rkegren et~al.(2020)Bj{\"{o}}rkegren, Blumenstock, and
  Knight]{BBK20}
Daniel Bj{\"{o}}rkegren, Joshua~E. Blumenstock, and Samsun Knight.
\newblock Manipulation-proof machine learning.
\newblock \emph{Arxiv:2004.03865}, 2020.

\bibitem[Braverman and Garg(2020)]{BG20}
Mark Braverman and Sumegha Garg.
\newblock The role of randomness and noise in strategic classification.
\newblock In \emph{Proc.~$1$st FORC 2020}, volume 156 of \emph{Leibniz
  International Proceedings in Informatics (LIPIcs)}, pages 9:1--9:20, 2020.

\bibitem[Brown et~al.(2020)Brown, Hod, and Kalemaj]{BHK20}
Gavin Brown, Shlomi Hod, and Iden Kalemaj.
\newblock Performative prediction in a stateful world.
\newblock \emph{Arxiv:2011.03885}, 2020.

\bibitem[Br{\"{u}}ckner and Scheffer(2011)]{BS11}
Michael Br{\"{u}}ckner and Tobias Scheffer.
\newblock Stackelberg games for adversarial prediction problems.
\newblock In \emph{Proc.~$17$th ACM {KDD}}, page 547–555, 2011.

\bibitem[Br\"{u}ckner et~al.(2012)Br\"{u}ckner, Kanzow, and
  Scheffer]{bruckner12pred}
Michael Br\"{u}ckner, Christian Kanzow, and Tobias Scheffer.
\newblock Static prediction games for adversarial learning problems.
\newblock \emph{JMLR}, 13\penalty0 (1):\penalty0 2617–2654, September 2012.

\bibitem[Camerer et~al.(2004)Camerer, Loewenstein, and Rabin]{behavioralecon}
Colin~F. Camerer, George Loewenstein, and Matthew Rabin.
\newblock \emph{Advances in Behavioral Economics}.
\newblock {Princeton University Press}, 2004.

\bibitem[Chen et~al.(2020)Chen, Liu, and Podimata]{CLP20}
Yiling Chen, Yang Liu, and Chara Podimata.
\newblock Learning strategy-aware linear classifiers.
\newblock In \emph{Proc.~$33$rd {NeurIPS}}, volume~33, pages 15265--15276,
  2020.

\bibitem[Coibion et~al.(2018)Coibion, Gorodnichenko, and Kamdar]{CGK18}
Olivier Coibion, Yuriy Gorodnichenko, and Rupal Kamdar.
\newblock The formation of expectations, inflation, and the phillips curve.
\newblock \emph{Journal of Economic Literature}, 56\penalty0 (4):\penalty0
  1447--1491, 2018.

\bibitem[Dalvi et~al.(2004)Dalvi, Domingos, Mausam, Sanghai, and
  Verma]{DDMSV04}
Nilesh~N. Dalvi, Pedro~M. Domingos, Mausam, Sumit~K. Sanghai, and Deepak Verma.
\newblock Adversarial classification.
\newblock In \emph{Proc. $10$th {KDD}}, page 99–108, 2004.

\bibitem[Dee et~al.(2019)Dee, Dobbie, Jacob, and Rockoff]{dee19}
Thomas~S. Dee, Will Dobbie, Brian~A. Jacob, and Jonah Rockoff.
\newblock {The Causes and Consequences of Test Score Manipulation: Evidence
  from the New York Regents Examinations}.
\newblock \emph{American Economic Journal: Applied Economics}, 11\penalty0
  (3):\penalty0 382--423, July 2019.

\bibitem[{Devroye} et~al.(2018){Devroye}, {Mehrabian}, and {Reddad}]{TVbound}
Luc {Devroye}, Abbas {Mehrabian}, and Tommy {Reddad}.
\newblock {The total variation distance between high-dimensional Gaussians}.
\newblock October 2018.

\bibitem[Dong et~al.(2018)Dong, Roth, Schutzman, Waggoner, and
  Wu]{dong18revealedpref}
Jinshuo Dong, Aaron Roth, Zachary Schutzman, Bo~Waggoner, and Zhiwei~Steven Wu.
\newblock Strategic classification from revealed preferences.
\newblock In \emph{Proc.~{EC}}, page 55–70, 2018.

\bibitem[Frankel and Kartik(2019)]{FK19}
Alex Frankel and Navin Kartik.
\newblock {Muddled Information}.
\newblock \emph{Journal of Political Economy}, 127\penalty0 (4):\penalty0
  1739--1776, 2019.

\bibitem[Frankel and Kartik(2020)]{FK20}
Alex Frankel and Navin Kartik.
\newblock {Improving Information via Manipulable Data}.
\newblock \emph{Working Paper}, 2020.

\bibitem[Gatzouras(2002)]{G02}
Dimitris Gatzouras.
\newblock On images of borel measures under borel mappings.
\newblock \emph{Proceedings of the American Mathematical Society}, 130\penalty0
  (9):\penalty0 2687--2699, 2002.

\bibitem[Ghalme et~al.(2021)Ghalme, Nair, Eilat, Talgam-Cohen, and
  Rosenfeld]{ghalme2021strategic}
Ganesh Ghalme, Vineet Nair, Itay Eilat, Inbal Talgam-Cohen, and Nir Rosenfeld.
\newblock Strategic classification in the dark.
\newblock \emph{Arxiv:2102.11592}, 2021.

\bibitem[Haghtalab et~al.(2020)Haghtalab, Immorlica, Lucier, and Wang]{HILW20}
Nika Haghtalab, Nicole Immorlica, Brendan Lucier, and Jack~Z. Wang.
\newblock Maximizing welfare with incentive-aware evaluation mechanisms.
\newblock In \emph{Proc.~$29$th {IJCAI}}, pages 160--166, 2020.

\bibitem[Hardt et~al.(2016{\natexlab{a}})Hardt, Megiddo, Papadimitriou, and
  Wootters]{HMPW16}
Moritz Hardt, Nimrod Megiddo, Christos~H. Papadimitriou, and Mary Wootters.
\newblock Strategic classification.
\newblock In \emph{Proc.~$7$th {ITCS}}, page 111–122. {ACM},
  2016{\natexlab{a}}.

\bibitem[Hardt et~al.(2016{\natexlab{b}})Hardt, Price, and Srebro]{HPS16}
Moritz Hardt, Eric Price, and Nati Srebro.
\newblock Equality of opportunity in supervised learning.
\newblock In \emph{Proc.~$29$th {NeurIPS}}, pages 3315--3323,
  2016{\natexlab{b}}.

\bibitem[Harsanyi(1968)]{harsanyi68incomplete}
John~C. Harsanyi.
\newblock Games with incomplete information played by "bayesian" players,
  i-iii. part ii. bayesian equilibrium points.
\newblock \emph{Management Science}, 14\penalty0 (5):\penalty0 320--334, 1968.

\bibitem[Hennessy and Goodhart(2020)]{HG20}
Christopher Hennessy and Charles Goodhart.
\newblock Goodhart's law and machine learning.
\newblock \emph{SSRN}, 2020.

\bibitem[Hu et~al.(2019)Hu, Immorlica, and Vaughan]{HIW19}
Lily Hu, Nicole Immorlica, and Jennifer~Wortman Vaughan.
\newblock The disparate effects of strategic manipulation.
\newblock In \emph{Proc.~{FAccT}}, page 259–268, 2019.

\bibitem[Kaplan and Violante(2018)]{KV18}
Greg Kaplan and Giovanni~L. Violante.
\newblock {Microeconomic Heterogeneity and Macroeconomic Shocks}.
\newblock \emph{Journal of Economic Perspectives}, 32\penalty0 (3):\penalty0
  167--194, 2018.

\bibitem[Khajehnejad et~al.(2019)Khajehnejad, Tabibian, Sch{\"{o}}lkopf,
  Singla, and Gomez{-}Rodriguez]{Khaje19opt}
Moein Khajehnejad, Behzad Tabibian, Bernhard Sch{\"{o}}lkopf, Adish Singla, and
  Manuel Gomez{-}Rodriguez.
\newblock Optimal decision making under strategic behavior.
\newblock \emph{Arxiv:1905.09239}, 2019.

\bibitem[Kleinberg and Raghavan(2019)]{Kleinberg19strat}
Jon Kleinberg and Manish Raghavan.
\newblock How do classifiers induce agents to invest effort strategically?
\newblock In \emph{Proc.~{EC}}, page 825–844, 2019.

\bibitem[Levanon and Rosenfeld(2021)]{levanon2021strategic}
Sagi Levanon and Nir Rosenfeld.
\newblock Strategic classification made practical.
\newblock \emph{Arxiv:2103.01826}, 2021.

\bibitem[Li(2017)]{L17}
Shengwu Li.
\newblock Obviously strategy-proof mechanisms.
\newblock \emph{American Economic Review}, 107\penalty0 (11):\penalty0
  3257--3287, 2017.

\bibitem[Lucas~Jr(1976)]{lucas1976econometric}
Robert~E Lucas~Jr.
\newblock Econometric policy evaluation: A critique.
\newblock In \emph{Carnegie-Rochester conference series on public policy},
  volume~1, pages 19--46. North-Holland, 1976.

\bibitem[Mendler-D{\"u}nner et~al.(2020)Mendler-D{\"u}nner, Perdomo, Zrnic, and
  Hardt]{MPZH20}
Celestine Mendler-D{\"u}nner, Juan Perdomo, Tijana Zrnic, and Moritz Hardt.
\newblock Stochastic optimization for performative prediction.
\newblock \emph{Proc.~$33$rd {NeurIPS}}, 33:\penalty0 4929--4939, 2020.

\bibitem[Miller et~al.(2020)Miller, Milli, and Hardt]{MMH20}
John Miller, Smitha Milli, and Moritz Hardt.
\newblock Strategic classification is causal modeling in disguise.
\newblock In \emph{Proc.~$37$th {ICML}}, volume 119, pages 6917--6926, 2020.

\bibitem[Miller et~al.(2021)Miller, Perdomo, and Zrnic]{MPZ21}
John Miller, Juan~C. Perdomo, and Tijana Zrnic.
\newblock Outside the echo chamber: Optimizing the performative risk.
\newblock \emph{Arxiv:2102.08570}, 2021.

\bibitem[Milli et~al.(2019)Milli, Miller, Dragan, and Hardt]{MMDH19}
Smitha Milli, John Miller, Anca~D. Dragan, and Moritz Hardt.
\newblock The social cost of strategic classification.
\newblock In \emph{Proc.~{FAccT}}, page 230–239, 2019.

\bibitem[Perdomo et~al.(2020)Perdomo, Zrnic, Mendler-D{\"u}nner, and
  Hardt]{PZMH20}
Juan~C. Perdomo, Tijana Zrnic, Celestine Mendler-D{\"u}nner, and Moritz Hardt.
\newblock Performative prediction.
\newblock In \emph{Proc.~$37$th {ICML}}, volume 119, pages 7599--7609, 2020.

\bibitem[Shavit et~al.(2020)Shavit, Edelman, and Axelrod]{SEA20}
Yonadav Shavit, Benjamin~L. Edelman, and Brian Axelrod.
\newblock Causal strategic linear regression.
\newblock In \emph{Proc.~$37$th {ICML}}, volume 119, pages 8676--8686, 2020.

\bibitem[Spence(1973)]{Spence73}
Michael Spence.
\newblock Job market signaling.
\newblock \emph{The Quarterly Journal of Economics}, 87\penalty0 (3):\penalty0
  355--374, 1973.

\bibitem[Spielman and Teng(2009)]{smoothedanalysis}
Daniel~A. Spielman and Shang-Hua Teng.
\newblock Smoothed analysis: An attempt to explain the behavior of algorithms
  in practice.
\newblock \emph{Commun. ACM}, 52\penalty0 (10):\penalty0 76–84, October 2009.

\bibitem[Stiglitz(2018)]{S18}
Joseph~E Stiglitz.
\newblock {Where modern macroeconomics went wrong}.
\newblock \emph{Oxford Review of Economic Policy}, 34\penalty0 (1-2):\penalty0
  70--106, 01 2018.

\bibitem[Tsirtsis and Gomez~Rodriguez(2020)]{tsirtsis2020decisions}
Stratis Tsirtsis and Manuel Gomez~Rodriguez.
\newblock Decisions, counterfactual explanations and strategic behavior.
\newblock \emph{Proc.~$33$rd {NeurIPS}}, 33:\penalty0 16749--16760, 2020.

\bibitem[Zhang and Conitzer(2021)]{ZC21}
Hanrui Zhang and Vincent Conitzer.
\newblock Incentive-aware {PAC} learning.
\newblock \emph{Proc.~$35$th {AAAI}}, 2021.

\bibitem[Zhang et~al.(2021)Zhang, Chen, and Conitzer]{ZCC21}
Hanrui Zhang, Yu~Chen, and Vincent Conitzer.
\newblock Automated mechanism design for classification with partial
  verification.
\newblock \emph{Proc.~$35$th {AAAI}}, 2021.

\end{thebibliography}
\bibliographystyle{plainnat}

\appendix
\onecolumn

\section*{Appendix}

\noindent For all of the proofs in the Appendix, we assume WLOG that the utility of a positive outcomes $\gamma$ is equal to $1$.

\section{Additional discussion of assumptions}

\subsection{Cost function} 
Let us discuss some context and implications of Assumption \ref{assumption:validcost} defining a valid cost function. Unlike prior work \citep{MMDH19, MMH20, BG20}, we model a nonzero cost for \textit{all} modifications to features, regardless of whether these modifications are in the right direction. In the spirit of generalizing beyond standard microfoundations, this accounts for how agents may erroneously expend effort on changing their features in an incorrect direction, as empirically demonstrated in Example \ref{ex:FE}. We further note that the definition of a valid cost function does not require symmetry in the arguments, which differentiates it from a metric.

\subsection{Measurability requirements for alternative microfoundations}\label{subsec:requirements}
We now describe the measurability requirements that we need in order to define and work with maps $M \in \cM$. If we ignore measurability requirements for a moment, then notice that each map $M \in \cM$ can be associated with a distribution $\DTrue \in \Delta(\mathcal{T} \times X \times Y)$ given by $(M(x,y), x, y)$. Defining measurability requirements on $\DTrue$ gives an implicit specification of requirements on $M$. First, we define the probability space: Consider the sample space $\mathcal{T} \times X \times Y$. We can define a sigma algebra $\mathcal{F}$ over $\Omega$ by viewing $\mathcal{T}$ as the set of functions $X \times \Theta \rightarrow X$, and using that $X \subseteq \mathbb{R}^d, \Theta \subseteq \mathbb{R}^d$. The probability measure can then be given by $\DTrue$. 

Since $\text{image}(M) =\text{supp}(\DTrue)$ contains a very small fraction of the sample space $\mathcal{T} \times X \times Y$, we can work with a much smaller probability space in this context. This probability space is defined as follows: the sample space is $\text{supp}(\DTrue) \in \mathcal{F}$ (i.e. a subset of $\mathcal{T} \times X \times Y$ in the sigma-algebra), and the sigma-algebra is intersections of every set in $\mathcal{F}$ with $\text{supp}(\DTrue)$. The probability measure given by $\DTrue$ can be defined over this smaller probability space. 

The distribution map $\DMap$ can thus be viewed as random variables over this probability space. In particular, $\DMap(\theta)$ is the distribution of the random variable $(\Response_t(x, \theta), y)$. In order for this random variable to be well-defined, we place the following measurability assumption. 
\begin{assumption}[Measurability requirement on $\Response$]
\label{assumption:measurability}
 We require that for each $\theta \in \Theta$ the function $F_\theta: \text{supp}(\DTrue) \rightarrow X \times Y$ given by $F_{\theta}(t, x, y) = (\Response_t(x, \theta), y)$ is measurable. 
\end{assumption}

\subsection{Assumption on gaming behavior}

In Section~\ref{subsec:microassumptions} we make the assumption that agent cannot have differing types solely on the basis of their true label. In other words, the map $M$ cannot take into account the true label.
\begin{assumption}
\label{assumption:gamingbehavior}
For a map $M \in \cM$, we require that $M(x, 0) = M(x,1)$ for all $x \in X$.
\end{assumption}
\noindent  Assumption \ref{assumption:gamingbehavior} means that agents with features $x$ who have true label $0$ versus true label $1$ have identical distributions over response types in aggregate. We need this assumption to reason about performatively optimal points, because a decision-maker has no access to the true labels beyond agents' reported features when anticipating strategic behavior. 
\subsection{Compactness of $X$}
This assumption guarantees that the behavior of agents who follow standard microfoundations is well-defined. 
\begin{fact}
\label{fact:existencewelldefined}
Suppose that $c$ is a valid cost function, and $X$ is compact. Then $\sup_{x' \in X} (f_{\theta}(x') - c(x,x'))$ is attained on some $x^* \in X$ and the behavior of rational agents with perfect information is well-defined.
\end{fact}
\begin{proof}
Let $X_{\text{pos}} := \left\{x \in X \mid f_{\theta}(x) = 1\right\} \subseteq X$ and $\Phi:=\sup_{x' \in X} (f_{\theta}(x') - c(x,x'))$. Suppose that  $\Phi=0$. Then, by Assumption \ref{assumption:validcost} the supremum is attained at $x' = x$. Now, suppose that $\Phi > 0$. Then $\Phi = 1 - \inf_{x' \in X_{\text{pos}}}  \big\{\costfn(x, x')\big\}$. By the fact that $X_{\text{pos}}$ is a closed subset of a compact set (and thus compact), and $\costfn$ is continuous, we know that $\inf_{x' \in X_{\text{pos}}}  \big\{\costfn(x, x')\big\}$ is attained on some $x' \in X$. 
\end{proof}

\section{Proofs for Section \ref{sec:limitations}}

\subsection{Proof of Proposition~\ref{prop:degenerate}}\label{appendix:proofsmicro}

In order to prove Proposition \ref{prop:degenerate}, we show that the gaming behavior of rational agents with perfect information can be characterized in the following way: Any rational agent with perfect information either will not change their features at all or will change their features exactly up to the decision boundary. We use the notation:
\begin{equation}
\Response_{t_{\text{SM}}}(x, \theta) := \argmax_{x' \in X} (f_{\theta}(x') - c(x,x'))
\label{eq:RSM}
\end{equation}
to denote how an agent with features $x$ who follows standard microfoundations will change their features in response to $f_{\theta}$. 
\begin{lemma}
\label{lemma:br}
Suppose that $c$ is a valid cost function. Then for any $x$ the response \eqref{eq:RSM} is either $\Response_{t_{\mathrm{SM}}}(x, \theta) = x$ or $\Response_{t_{\mathrm{SM}}}(x, \theta)$ is on the decision boundary of $f_{\theta}$.
\end{lemma}
\begin{proof}[Proof of Proposition \ref{lemma:br}]
By Fact \ref{fact:existencewelldefined}, we know that the quantity $\argmax_{x' \in X} (f_{\theta}(x') - c(x,x'))$ is well defined. It suffices to show that if $\Response_{t_{\text{SM}}}(x, \theta) \neq x$, then $\Response_{t_{\text{SM}}}(x, \theta)$ is on the decision boundary of $f_{\theta}$. If $\Response_{t_{\text{SM}}}(x, \theta) \neq x$, then we know that $c(x, \Response_{t_{\text{SM}}}(x, \theta)) > 0$. This means that $f_{\theta}(x) = 0$ and $\Response_{t_{\text{SM}}}(x, \theta) \in \argmax_{x' \in X} (f_{\theta}(x') - c(x,x')) = \argmin_{x' \in X_{\text{pos}}} c(x,x')$, where $X_{\text{pos}} := \left\{x \in X \mid f_{\theta}(x) = 1\right\}$. Assume for sake of contradiction that $\Response_{t_{\text{SM}}}(x, \theta)$ is not on the decision boundary. Then since $x \not\in X_{\text{pos}}$ and $\Response_{t_{\text{SM}}}(x, \theta) \in X_{\text{pos}}$, there must exist $x'$ on the line segment between $x$ and $\Response_{t_{\text{SM}}}(x, \theta)$ such that $x'$ is on boundary of $X_{\text{pos}}$, and thus the decision boundary of $f_{\theta}$. Moreover, by Assumption \ref{assumption:validcost}, we know that $c(x, x') < c(x, \Response_{t_{\text{SM}}}(x, \theta))$. Since $X_{\text{pos}}$ is closed, we see that $f_{\theta}(x') = 1$. Thus, $[f_{\theta}(x') - c(x,x')] < [f_{\theta}(\Response_{t_{\text{SM}}}(x, \theta)) - c(x,x')]$ which is a contradiction.
\end{proof}

Now, we use Lemma \ref{lemma:br} to prove Proposition \ref{prop:degenerate}. 
\begin{proof}[Proof of Proposition \ref{prop:degenerate}]
It suffices to show that $\DMap(\theta)$ is either equal to $\DBase$ or is a discontinuous distribution. Let $Q(\theta) \subseteq X$ be the set of agents who change their features at $f_{\theta}$, i.e. 
\[Q(\theta) := \left\{ x \in X \mid \Response_{t_{\text{SM}}}(x, \theta) \neq x \right\}.\] If $\mathbb{P}_{(x,y) \in \DBase} [x \in Q(\theta)] = 0$, then $\DMap(\theta) = \DBase$. Otherwise, suppose that $\mathbb{P}_{(x,y) \in \DBase} [x \in Q(\theta)] > 0$. By Lemma \ref{lemma:br}, all of the agents in $Q(\theta)$ will game to somewhere on the decision boundary: that is, $ \Response_{t_{\text{SM}}}(x, \theta)$ will be on the decision boundary for all $x \in Q(\theta)$. Thus, in $\DMap(\theta)$, there will be at least a $\mathbb{P}_{(x,y) \in \DBase} [x \in Q(\theta)]$ probability mass of agents at the decision boundary, which is measure $0$. This means that $\DMap(\theta)$ is not a continuous distribution. 
\end{proof}

\subsection{Proof of Proposition~\ref{propo:stablepointsnotexist}}
\label{app:nostablepoints}

For convenience, we break down Proposition \ref{propo:stablepointsnotexist} into a series of propositions, roughly corresponding to part (a), part (b), and part (c), which we prove one-by-one.

First, let's consider the case where $p = 0$. By the assumptions in Setup \ref{ex:stablepointsnotexist}, we know that there exists a unique $\theta \in \Theta$ such that $\mu(\theta) = 0.5$. We denote this value by $\theta_{\mathrm{SL}}$ (and it is in the interior of $\Theta$). We claim that this is the unique locally stable point when $p = 0$.
\begin{lemma}
\label{lemma:p1}
Consider Setup \ref{ex:stablepointsnotexist} when $p = 1$ fraction of agents are non-strategic. Then, $\theta_{\mathrm{SL}}$ (defined above) is the unique locally stable point.
\end{lemma}
\begin{proof}
Since $p = 1$, the distribution map is given by $\DMap(\theta) = \DBase$. A locally stable point $\theta$ must be a local minimum or a stationary point of the following optimization problem:
\[\min_{\theta \in \Theta} \mathbb{E}_{(x,y) \in \DBase}[\Indicator\{f_{\theta}(x) = y\}] = \min_{\theta \in \Theta} \left( \mathbb{E}_{(x,y) \in \DBase}[\Indicator\{x \ge \theta\} (1 - \mu(x))] + \mathbb{E}_{(x,y) \in \DBase}[\Indicator\{x < \theta\} \mu(x) ] \right).\]
Notice that the unique such $\theta$ is $\theta_{\mathrm{SL}}$. 
\end{proof}

We introduce some basic properties and notation for agents who behave according to standard microfoundations. By Lemma \ref{lemma:br}, we know if an agent games when the classifier $f_{\theta}$ is deployed, then they will game up to boundary, which in this case, is $\theta$. We adopt similar notation to the proof of Proposition \ref{prop:degenerate}, and we denote the set of who game by\footnote{Technically, the agents $x \in Q(\theta)$ for whom $c(x, \theta) = 1$ are indifferent between not gaming and gaming to $\theta$, but this is a measure $0$ set by the assumption that $\DBase$ is continuous, and the assumption that $\costfn$ is valid (Assumption \ref{assumption:validcost})}:
\[Q(\theta) := \left\{ x \in X \mid \Response_{t_{\text{SM}}}(x, \theta) \neq x\right\} = \left\{x \in X \mid c(x, \theta) \le 1, x < \theta \right\}.\] 
For $\theta \neq \min(\Theta)$, we see that for $\DMap(\theta)$, there will be a point mass at $\theta$ (from agents in $Q(\theta)$), the region $Q(\theta)$ will have zero probability density, and the rest of the distribution will remain identical to $\DBase$.

We first characterize the set of stable points at $p =0$. This follows a very similar argument to Lemma 3.2 in \citep{MMDH19}, but since our assumptions as well as our requirements for stability are slightly weaker, the characterization result looks slightly different. (In particular, points above the Stackelberg equilibrium can be locally stable points.) 
\begin{lemma}
\label{lemma:p0}
Consider Setup \ref{ex:stablepointsnotexist} when a $p = 0$ fraction of agents are non-strategic. Then, there exists a locally stable point, and moreover, the set of locally stable points forms an interval $[\thetaPS^{\mathrm{SM}}, \max(\Theta)]$, where $\thetaPS^{\mathrm{SM}}$ is the unique value such that:
\[\frac{\mathbb{E}_{(x,y) \in \DBase}[\mu(x) I_{x \in Q(\thetaPS^{\mathrm{SM}})}]}{\mathbb{E}_{(x,y) \in \DBase}[I_{x \in Q(\thetaPS^{\mathrm{SM}})}]}= 0.5.\] (Moreover, it holds that $\thetaPS^{\mathrm{SM}} > \theta_{\mathrm{SL}}$, and $c(\theta_{\mathrm{SL}}, \thetaPS^{\mathrm{SM}})  \le 1$.)
\end{lemma}
\begin{proof}
First, we show that $\theta^* = \min(\Theta)$ cannot be a locally stable point. Assume for sake of contradiction that $\theta^* = \min(\Theta)$ is a locally stable point. Notice that $\DMap(\theta^*) = \DBase$. Thus, $\theta^*$ is a local minimum or stationary point of $\min_{\theta \in \Theta} \mathbb{E}_{(x,y)\sim \DBase} \;\Indicator\{y\neq f_\theta(x) \}$. However, this is not possible because $\mu(\min(\Theta)) < 0.5$ by the assumptions in Setup~\ref{ex:stablepointsnotexist}. 

Now, we consider $\theta \neq \min(\Theta)$. In this case, as discussed above, $\DMap(\theta)$ has a point mass at $\theta$. Roughly speaking, the only property that needs to be satisfied for $\theta$ in the interior of $\Theta$ to be a local minimum of $\mathbb{E}_{(x,y) \in \DMap(\theta)}[\Indicator\{f_{\theta'}(x) = y\}]$ is that it needs to be suboptimal for the decision-maker to move just above the point mass (the decision-maker never benefits from moving to $\theta - \epsilon$ because there is a region of zero probability density underneath $\theta$). The loss induced from the point mass at $\theta$ by the classifier $f_{\theta}$ is $\mathbb{E}_{(x,y) \in \DBase}[(1 - \mu(x)) \Indicator\{x \in Q(\theta)\}]$. On the other hand, the loss induced from the point mass at $\theta$ by the classifier $f_{\theta'}$ in the limit as $\theta' \rightarrow \theta^+$ is $\mathbb{E}_{(x,y) \in \DBase}[\mu(x) \Indicator\{x \in Q(\theta)\}]$. A necessary and sufficient condition for $f_{\theta}$ to be locally stable is thus $\mathbb{E}_{(x,y) \in \DBase}[\mu(x) \Indicator\{{x \in Q(\theta)}\}] \ge \mathbb{E}_{(x,y) \in \DBase}[(1 - \mu(x)) \Indicator\{x \in Q(\theta)\}]$, which can be written as
\begin{equation}
\label{eq:condition}
\Gamma(\theta):=\frac{\mathbb{E}_{(x,y) \in \DBase}[\mu(x) \Indicator\{x \in Q(\theta)\}]}{\mathbb{E}_{(x,y) \in \DBase}[\Indicator\{x \in Q(\theta)\}]} \ge 0.5.
\end{equation}
\noindent It suffices to show that the set of points where \eqref{eq:condition} is satisfied is an interval of the form $[\thetaPS^{\mathrm{SM}}, \max(\Theta)]$.

First, we show that the set of stable points forms an interval. It suffices to show that $\Gamma(\theta)$ is continuous and strictly increasing in $\theta$. By the assumption on $c$ (Assumption \ref{assumption:validcost}), we see that the endpoints of the interval $Q(\theta)$ are strictly increasing in $\theta$. This, coupled with the fact that $\mu$ is strictly increasing in $x$ (assumed in Setup \ref{ex:stablepointsnotexist}), implies that $\Gamma(\theta)$ is continuous and strictly increasing as desired. 

Furthermore, when $\mu(\theta) \le 0.5$, the condition in \eqref{eq:condition} is never satisfied, and thus all stable points $\theta$ satisfy $\theta > \theta_{\mathrm{SL}}$, and hence $\thetaPS^{\mathrm{SM}} > \theta_{\mathrm{SL}}$.

Lastly, we show that this interval is not nonempty, and that $c(\theta_{\mathrm{SL}}, \thetaPS^{\mathrm{SM}}) \le 1$. Consider $\theta$ such that $c(\theta_{\mathrm{SL}}, \theta) = 1$ (which we know exists by Setup \ref{ex:stablepointsnotexist}), we see that $Q_{\theta} = [\theta_{\mathrm{SL}}, \theta]$. Using the conditions on $\ell$, this means that condition \eqref{eq:condition} is satisfied and there is actually a strict equality. Using that $c$ is valid, this means that $c(\theta_{\mathrm{SL}}, \thetaPS^{\mathrm{SM}}) < 1$. 
\end{proof}

We now prove that no locally stable points exist for $0 < p < 1$.
\begin{lemma}
\label{lemma:p01}
Consider Setup \ref{ex:stablepointsnotexist} when a $0 < p < 1$ fraction of agents are non-strategic. Then, there are no locally stable points.
\end{lemma}
\begin{proof}
When $0 < p < 1$, we show that there are no locally stable points. Assume for sake of contradiction that $\theta^*$ is a locally stable point. Recall that for $\theta^*$ to be locally stable, $\theta^*$ must either be a stationary point or a local minimum of $\min_{\theta \in \Theta} \mathbb{E}_{(x,y)\sim \cD(\thetaPS)} \;\Indicator\{y\neq f_\theta(x) \}$. We divide into three cases: (1) $\theta^* = \min(\Theta)$, (2) $\theta^* > \min(\Theta)\wedge \mu(\theta^*) \le 0.5$, (3) $\theta^* > \min(\Theta)\wedge \mu(\theta^*) > 0.5$, and show that each results in a contradiction.

For the case (1), where $\theta^* := \min(\Theta)$, we see that $\DMap(\theta^*) = \DBase$. Thus, $\theta^*$ is a local minimum or stationary point of $\min_{\theta \in \Theta} \mathbb{E}_{(x,y)\sim \DBase)} \;\Indicator\{y\neq f_\theta(x) \}$. However, this is not possible because $\mu(\min(\Theta)) < 0.5$ by the assumptions in Setup \ref{ex:stablepointsnotexist}. For the remaining two cases, we know that $\DMap(\theta^*)$ has a point mass at $\theta^*$. This means that $\mathbb{E}_{(x,y) \in \DMap(\theta^*)} \Indicator\{y \neq f_{\theta'}(x)\}$ is not differentiable at $\theta' = \theta^*$, and so $\theta^*$ must be a local minimum. 

For case (2), we see that $Q(\theta^*)$ consists a nonzero density of agents for whom $\mu(x) < 0.5$, and for all agents $x \in Q(\theta^*)$, it holds that $\mu(x) \le 0.5$. The decision-maker thus wishes to move just to the other side of the point mass at $\theta$. (This is possible because $\theta^* < \max(\Theta)$ based on the fact that $\mu(\theta^*) < 0.5$ and the assumptions in Setup \ref{ex:stablepointsnotexist}.) In particular, 
\[\lim_{\epsilon \rightarrow 0^+} \mathbb{E}_{(x,y) \in \DMap(\theta)} \Indicator\{y \neq f_{\theta^* + \epsilon}(x)\} < \mathbb{E}_{(x,y) \in \DMap(\theta)} \Indicator\{y \neq f_{\theta^*}(x)\},\]
which is a contradiction. 

For case (3), notice that there exists $\epsilon > 0$ such that $\mu(\theta) > 0.5$ and $\theta \in Q(\theta^*)$ for all $\theta \in (\theta^* - \epsilon, \theta^*)$. The presence of non-strategic agents means that the decision-maker wishes to move to $\theta^* - \epsilon$ to achieve better performance on non-strategic agents. Since there are no strategic agents within $(\theta^* - \epsilon, \theta^*)$, this can be done without affecting the classification of strategic agents. In particular, \[\mathbb{E}_{(x,y) \in \DMap(\theta)} \Indicator\{y \neq f_{\theta - \epsilon}(x)\} < \mathbb{E}_{(x,y) \in \DMap(\theta)} \Indicator\{y \neq f_{\theta}(x)\},\]
which is a contradiction.
\end{proof}

Now, we analyze the oscillation behavior of repeated risk minimization when $0 < p < 1$.
\begin{lemma}
\label{lemma:rrm}
Consider Setup \ref{ex:stablepointsnotexist} when a $0 < p < 1$ fraction of agents are non-strategic. Repeated risk minimization will oscillate between $\theta_{\mathrm{SL}}$ and a threshold $\tau(p) > \theta_{\mathrm{SL}}$. The threshold $\tau(p)$ is decreasing in $p$, approaching $\theta_{\mathrm{SL}}$ as $p \rightarrow 1$ and approaching $\thetaPS^{\mathrm{SM}}$ as $p \rightarrow 0$.
\end{lemma}
\begin{proof}
Using Lemma \ref{lemma:p1}, we see that there is a unique performatively stable point for $p = 1$ given by $\theta_{\mathrm{SL}} = \theta_{\mathrm{SL}}$. We see that the smallest locally stable point for $p = 0$ is given by $\thetaPS^{\mathrm{SM}}$ as characterized in Lemma \ref{lemma:p0}. 

In the case of $p \in (0,1)$, the distribution map $\DMap(\theta)$ takes the form of a mixture with $p$ weight on $\DBase$ and with $1 - p$ weight on the distribution map of agents who behave according to standard microfoundations (which has a point mass at $\theta$, zero density within $Q_{\theta}$, and the same as the original distribution elsewhere). The main step in our proof is an analysis of the global optima of 
\[B(\theta) = \argmin_{\theta' \in \Theta} \mathbb{E}_{(x,y)\sim \cD(\theta)} \;\Indicator\{y\neq f_{\theta'}(x) \}.\] for each $\theta \in \Theta$. For convenience, we let 
\[\text{DPR}(\theta, \theta') := \mathbb{E}_{(x,y)\sim \cD(\theta)} \;\Indicator\{y\neq f_{\theta'}(x)\}.\]
We split into three cases: (a) $\theta \ge \thetaPS^{\mathrm{SM}}$, (b) $\theta < \theta_{\mathrm{SL}}$, and (c) $\theta_{\mathrm{SL}} \le \theta \le \thetaPS^{\mathrm{SM}}$. 

\paragraph{Case (a): $\theta \ge \thetaPS^{\mathrm{SM}}$.} We claim that $B(\theta) = \left\{\theta_{\mathrm{SL}}\right\}$. 

First, we show that $\theta' \not\in B(\theta)$ for $\theta' > \theta$. To see this, we note that the proof of Lemma \ref{lemma:p0} tells us that moving just above the point mass at $\theta$ will incur no better risk than at $\theta$. Moreover, since $\mu(x) > 0.5$ for all $x \ge \theta \ge \thetaPS^{\mathrm{SM}}$, we see that $\text{DPR}(\theta, \theta') > \text{DPR}(\theta, \theta)$ for all $\theta' > \theta$, as desired. 

Next, we show that $\theta' \not\in B(\theta)$ for $\theta' < \theta_{\mathrm{SL}}$. This is because $\mu(x) < 0.5$ for $x < \theta_{\mathrm{SL}}$, and so $\text{DPR}(\theta, \theta') > \text{DPR}(\theta, \theta_{\mathrm{SL}})$ for all $\theta' < \theta_{\mathrm{SL}}$, as desired. 

Next, we show that $\theta' \not\in B(\theta)$ for $\theta_{\mathrm{SL}} < \theta' \le \theta$, and $\theta_{\mathrm{SL}} \in B(\theta)$. Because of the presence of non-strategic agents, a $p$ fraction of agents will be present in $Q_{\theta}$, and these agents do not change their features. Since $\mu(x) > 0.5$ for $\theta_{\mathrm{SL}} < x \le \theta$, we see that $\text{DPR}(\theta, \theta') < \text{DPR}(\theta, \theta)$ for all $\theta_{\mathrm{SL}} \le \theta' \le \theta$. Moreover, this argument actually shows that $\theta_{\mathrm{SL}} = \argmin_{\theta_{\mathrm{SL}} \le \theta' \le \theta} \text{DPR}(\theta, \theta')$, as desired. 

\paragraph{Case (b):  $\theta < \theta_{\mathrm{SL}}$.}
We claim that $B(\theta) = \left\{\theta_{\mathrm{SL}}\right\}$. 

First, we claim that $\theta' \not\in B(\theta)$ for $\theta' > \theta_{\mathrm{SL}}$. This is because $\mu(x) > 0.5$ for $x > \theta_{\mathrm{SL}}$, and so $\text{DPR}(\theta, \theta') > \text{DPR}(\theta, \theta_{\mathrm{SL}})$ for all $\theta' > \theta_{\mathrm{SL}}$ as desired. 

Next, we claim that $\theta' \not\in B(\theta)$ for $\theta' < \theta$. This is because $\mu(x) < 0.5$ for $x < \theta$ and so $\text{DPR}(\theta, \theta') > \text{DPR}(\theta, \theta_{\mathrm{SL}})$ for all $\theta' < \theta$ as desired. 

Finally, we claim that $\theta' \not\in B(\theta)$ for $\theta \le \theta' < \theta_{\mathrm{SL}}$ and $\theta_{\mathrm{SL}} \in B(\theta)$.  Since $\mu(x) < 0.5$ for $x < \theta_{\mathrm{SL}}$, we see that $\text{DPR}(\theta, \theta') < \text{DPR}(\theta, \theta)$ for all $\theta \le \theta' < \theta_{\mathrm{SL}}$. Moreover, this argument actually shows that $\theta_{\mathrm{SL}} = \argmin_{\theta \le \theta' < \theta_{\mathrm{SL}}} \text{DPR}(\theta, \theta')$, as desired. 

\paragraph{Case (c): $\theta_{\mathrm{SL}} \le \theta \le \thetaPS^{\mathrm{SM}}$.} In this case, $B(\theta)$ can sometimes out to technically be the empty set because of the discontinuities in the optimization problem. Let's be slightly imprecise and introduce a classifier $\theta^+$ given by the threshold right above the point mass at $\theta$ (that is the classifier given by the limit $\lim_{\epsilon \rightarrow 0^+} f_{\theta + \epsilon}$). (In practice, since repeated risk minimization necessarily has finite precision and discretizes the space $\Theta$, this corresponds to a classifier $f_{\theta + \epsilon}$ where $\epsilon$ is very small.) With the addition of this classifier, now $B(\theta)$ is well-defined. We show that in some cases, $B(\theta) = \left\{\theta_{\mathrm{SL}}\right\}$ and in some cases $B(\theta) = \left\{\theta^+\right\}$. 

First, we observe that $\theta' \not\in B(\theta)$ for $\theta' < \theta_{\mathrm{SL}}$. This follows from the same argument as Case (a), which shows that $\text{DPR}(\theta, \theta') > \text{DPR}(\theta, \theta_{\mathrm{SL}})$ for all $\theta' < \theta_{\mathrm{SL}}$.

Moreover, we claim that $\theta' \not\in B(\theta)$ for $\theta' > \theta$. To show this, we see that $\text{DPR}(\theta, \theta^+) < \text{DPR}(\theta, \theta')$ for any $\theta' > \theta$. This is because all agents $x > \theta$ have $\mu(x) > 0.5$. 

We now claim that $\theta' \not\in B(\theta)$ for $\theta_{\mathrm{SL}} < \theta' \le \theta$. It suffices to show that $\text{DPR}(\theta, \theta') > \text{DPR}(\theta, \theta_{\mathrm{SL}})$. This is because $c(\theta_{\mathrm{SL}}, \thetaPS^{\mathrm{SM}}) \le 1$ by Lemma \ref{lemma:p0}, so $c(\theta_{\mathrm{SL}}, \theta) \le 1$, so $\theta_{\mathrm{SL}} \in Q_{\theta}$. This means that the behavior of $f_{\theta_{\mathrm{SL}}}$ and $f_{\theta'}$ is the same on non-strategic agents not in $Q_{\theta} \cap [\theta_{\mathrm{SL}}, \max(\Theta)]$ as well as all strategic agents. Moreover, $f_{\theta_{\mathrm{SL}}}$ achieves lower loss than $f_{\theta'}$ on non-strategic agents in $Q_{\theta} \cap [\theta_{\mathrm{SL}}, \max(\Theta)]$, as desired. 

Thus, we have shown that $B(\theta)$ cannot contain any elements except for $\theta_{\mathrm{SL}}$ and $\theta^+$. Thus, all that remains is to compare these two classifiers.  Notice that these two classifiers behave the same on strategic agents with true features $x \not\in Q_{\theta}$ (this is because $\theta_{\mathrm{SL}} \in Q_{\theta}$, because by Lemma \ref{lemma:p0}, we know that $c(\theta_{\mathrm{SL}}, \theta) < c(\theta_{\mathrm{SL}}, \thetaPS^{\mathrm{SM}}) < 1$.). Moreover, they also behave the same on non-strategic agents not in $\theta_{\mathrm{SL}} \le x \le \theta$. Thus, we only need to focus on strategic agents with true features in $Q_{\theta}$ and non-strategic agents with $\theta_{\mathrm{SL}} \le x \le \theta$. Thus, we use the expression in the proof of Lemma \ref{lemma:p0} to see that: 
\begin{align*}
  &\text{DPR}(\theta, \theta^+) - \text{DPR}(\theta, \theta_{\mathrm{SL}}) \\
  &= p \cdot \mathbb{E}_{(x,y) \in \DBase}[\Indicator\{x \in [\theta_{\mathrm{SL}}, \theta]\} \mu(x)] + (1-p) \mathbb{E}_{(x,y) \in \DBase}[\Indicator\{x \in Q(\theta)\}  \mu(x)] \\
  &- p \mathbb{E}_{(x,y) \in \DBase}[\Indicator\{x \in [\theta_{\mathrm{SL}}, \theta]\} (1 - \mu(x))] - (1-p) \mathbb{E}_{(x,y) \in \DBase}[\Indicator\{x \in Q(\theta)\} (1 - \mu(x))]
    \\
  &= 2p  \mathbb{E}_{(x,y) \in \DBase}[\Indicator\{x \in [\theta_{\mathrm{SL}}, \theta]\} \mu(x)] + 2(1-p) \mathbb{E}_{(x,y) \in \DBase}[\Indicator\{x \in Q(\theta)\}  \mu(x)] \\
  &- p \mathbb{E}_{(x,y) \in \DBase}[\Indicator\{x \in [\theta_{\mathrm{SL}}, \theta]\}] - (1-p) \mathbb{E}_{(x,y) \in \DBase}[\Indicator\{x \in Q(\theta)\}].
\end{align*}
The relevant quantity is:
\[Z(p, \theta) := p \left( \mathbb{E}_{(x,y) \in \DBase}[I_{x \in [\theta_{\mathrm{SL}}, \theta]} (2 \mu(x) - 1)]\right) + (1-p) \left(\mathbb{E}_{(x,y) \in \DBase}[I_{x \in Q(\theta)}  (2 \mu(x) - 1)]\right) \]
We see that $B(\theta) = \left\{\theta_{SL}\right\}$ if and only if $Z(p, \theta) > 0$, $B(\theta) = \left\{ \theta^+ \right\}$ if and only if $Z(p, \theta) < 0$, and $B(\theta) = \left\{ \theta^+, \theta_{\mathrm{SL}} \right\}$ if and only if $Z(p, \theta) = 0$. 

\paragraph{Concluding the behavior of repeated risk minimization.} 
We wish to show that repeated risk minimization will oscillate between $\theta_{\mathrm{SL}}$ and $\tau(p)$, where $\tau(p)$ is the value such that $Z(p, \tau(p)) = 0$. To show this, it suffices to show that $Z(p, \theta)$ is increasing in $\theta$, and that $Z(p, \theta_{\mathrm{SL}}) < 0$ and $Z(p, \thetaPS^{\mathrm{SM}}) > 0$. Let $p_{\text{base}}$ be the pdf of $\DBase$. The derivative of the first term is: $p (2 \mu(\theta) - 1) p_{\text{base}}(\theta) > 0$, and the derivative of the second term is:
\[(1-p)(2 \mu(\theta) - 1) p_{\text{base}}(\theta) - (1-p)(2 \mu(\min(Q(\theta)) - 1) p_{\text{base}}(\min(Q(\theta))) > 0,\] as desired. The behavior at $\theta_{\mathrm{SL}}$ and $\thetaPS^{\mathrm{SM}}$ follows from the definition.

Now, we analyze how $\tau(p)$ changes with $p$. To see that $\tau(p)$ is decreasing in $p$, notice that $Z(p)$ is increasing in $p$ for all $\theta_{\mathrm{SL}} \le \theta \le \thetaPS^{\mathrm{SM}}$. As $p \rightarrow 0$, it is easy to see that $\tau(p) \rightarrow \thetaPS^{\mathrm{SM}}$. As $p \rightarrow 1$, it is easy to see that $\tau(p) \rightarrow \theta_{\mathrm{SL}}$. 
\end{proof}

Using the above results, we can conclude Proposition \ref{propo:stablepointsnotexist}. 
\begin{proof}[Proof of Proposition \ref{propo:stablepointsnotexist}]
When $p = 1$, we can apply Lemma \ref{lemma:p1}. When $p = 0$, we can apply Lemma \ref{lemma:p0} to see that a locally stable point exists. When $0 < p < 1$, we can apply Lemma \ref{lemma:p01} to see that no locally stable point exists. For the behavior of repeated risk minimization, we can apply Lemma \ref{lemma:rrm}. 
\end{proof}

\subsection{Formal Statements and Proofs of Proposition
\ref{prop:socialburdeninformal} and Corollary \ref{coro:socialburdeninformal}}
\label{app:socialburden}
We give formal statements of Proposition \ref{prop:socialburdeninformal} and Corollary \ref{coro:socialburdeninformal} using the technology of alternative microfoundations. Let $c$ be a valid cost function. First, we formalize expenditure monotonicity (Property \ref{property:SB}) in the language of alternative microfoundations. 
\begin{property}
\label{prop:expenditurerat}
Let $\Theta$ be a function class of threshold functions, and let $c$ be a cost function. A mapping $M \in \mathcal{M}$ satisfies \textit{expenditure monotonicity} if $c(\Response_t(x,\theta),x)\leq\gamma$ for every $\theta \in \Theta$ and every $t \in  \text{Image}(M)$, and if $f_{\theta}(\mathcal \Response_t(x;\theta)) = 1$, then $f_{\theta'}(\mathcal \Response_t(x;\theta')) = 1$ for all $\theta' \le \theta$. 
\end{property}

Let $\mathcal{M}^*$ be the set of maps $M$ such that every $t \in \cup_{(x,y) \in X} \text{supp}(M(x))$ satisfies expenditure monotonicity (Property \ref{prop:expenditurerat}) and such that Assumption \ref{assumption:gamingbehavior} is satisfied. Let $\mathscr{D}$ be the set of distribution maps $\DMap(\cdot; M)$ for $M \in \mathcal{M}^*$. 

\begin{proposition}[Formal Statement of Proposition \ref{prop:socialburdeninformal}]
\label{prop:socialburdenformal}
Consider Setup \ref{ex:stablepointsnotexist}. Let $\mathscr{D}$ be the class of distribution maps defined above. Then: 
\[\thetaPO(\DMap_{\mathrm{SM}}) \geq \thetaPO(\DMap),\]
where $\DMap_{\mathrm{SM}}$ denotes the distribution map given by standard microfoundations, and $\thetaPO(\DMap)$ denotes the minimal performatively optimal point associated with the distribution map $\DMap$.
\end{proposition}
\begin{proof}
For ease of notation, let $\theta_{\text{SL}}$ be the unique value such that $p(\theta_{\text{SL}}) = 0.5$. It is easy to see that $\theta' = \thetaPO(\DMap_{\text{SM}})$ is the unique point such that $c(\theta_{\text{SL}}, \theta') = 1$ and $\theta' > \theta_{\text{SL}}$. 

It suffices to show that for $\theta > \thetaPO(\DMap_{\text{SM}})$ and for any $\DMap \in \mathcal{D}$, it holds that $\PR(\theta) \le \PR(\thetaPO(\DMap_{\text{SM}}))$, where the performative risk is with respect to $\DMap$. 

First, let's consider the set of agents $S_1 := \left\{(t, x, y) \mid x < \theta_{\text{SL}} \right\}$. For $(t, x, y) \in S_1$, notice that $c(x, \theta) > c(x, \thetaPO(\DMap_{\text{SM}})) > c (\theta_{\text{SL}},\thetaPO(\DMap_{\text{SM}})) = c(\theta_{\text{SL}},\thetaPO(\DMap_{\text{\RDPI}})) \ge 1$. By the expenditure constraint, this means that these agents will not game on $f_{\theta}$ or $f_{\thetaPO(\DMap_{\text{SM}})}$: i.e, $\mathcal{R}_t(x, \theta) = x$ and $\mathcal{R}_t(x, \thetaPO(\DMap_{\text{SM}})) = x$ for $(t,x, y) \in S_1$. Thus, $f_{\thetaPO(\DMap_{\text{SM}})}(\mathcal{R}_t(x, \thetaPO(\DMap_{\text{SM}}))) = f_{\theta}(\mathcal{R}_t(x, \theta)) = 0$. The performative risk with respect to $\DMap$ is thus equivalent on $S_1$ for $f_{\theta}$ and $f_{\thetaPO(\DMap_{\text{SM}})}$.

Now, let's consider the remaining set of agents $S_2 := \left\{(t,x, y) \mid x \ge \theta_{\text{SL}} \right\}$. Let $S_2' \subseteq S_2$ be the set 
\[S_2' = \left\{(t,x, y) \in S_2 \mid f_{\theta}(\mathcal{R}_t(x, \theta)) = 1 \right\}.\] and 
\[S_2'' = \left\{(t,x,y) \in S_2 \mid f_{\thetaPO(\DMap_{\text{SM}})}(\mathcal{R}_t(x, \thetaPO(\DMap_{\text{SM}}))) = 1 \right\}.\] 
We claim that $S_2' \subseteq S_2''$. This is because of the second condition in expenditure monotonicity that if $x$ was labeled positively by $f_{\theta}$, then $x$ is also labeled positively by $f_{\thetaPO(\DMap_{\text{SM}})}$: in particular, we can thus conclude that $f_{\theta}(\mathcal{R}_t(x, \theta) = 1$ implies that $f_{\thetaPO(\DMap_{\text{SM}})}(\mathcal{R}_t(x, \thetaPO(\DMap_{\text{SM}}))) = 1$. 

Now, we claim that the performative risk with respect to $\DMap$ on $S_2$ is no better for $\theta$ than for $\thetaPO(\DMap_{\text{SM}})$. This follows from the fact that $S_2' \subseteq S_2''$, and the fact that $\mu(x) \ge 0.5$ for $(t, x, y) \in S_2$ coupled with Assumption \ref{assumption:gamingbehavior}. This completes the proof.
\end{proof}

\begin{corollary}[Formal Statement of Corollary \ref{coro:socialburdeninformal}]
\label{coro:socialburdenformal}
Consider Setup \ref{ex:stablepointsnotexist}. Let $\mathscr{D}$ be the class of distribution maps defined above. Then: 
\[\SocialBurden{\thetaPO(\DMap_{\mathrm{SM}})} \geq \SocialBurden{\thetaPO(\DMap)},\]
where $\DMap_{\mathrm{SM}}$ denotes the distribution map given by standard microfoundations, and $\thetaPO(\DMap)$ denotes the minimal performatively optimal point associated with the distribution map $\DMap$.
\end{corollary}
\begin{proof}
We apply Proposition \ref{prop:socialburdenformal} to see that $\thetaPO(\DMap_{\text{SM}}) \geq \thetaPO(\DMap)$. Now, we observe that it follows immediately from the definition of social burden that it is monotonic in the threshold $\theta$: in particular, the social burden increases with $\theta$. Thus, we can conclude that $\SocialBurden{\thetaPO(\DMap_{\mathrm{SM}})} \geq \SocialBurden{\thetaPO(\DMap)}$ as desired. 
\end{proof}

\section{Proofs for Section \ref{sec:agnostic}}

\subsection{Proof of Proposition \ref{prop:microfound}}\label{appendix:microfound}

We prove Proposition \ref{prop:microfound}. The intuition is that there is a response type for every possible agent response, and it remains to show that the appropriate choice of agent response types can ``shift the mass'' from $\DBase$ to $\DMap(\theta)$. In fact, $M$ only needs to map the population to two different response types. Now, we formally prove this result. 

\begin{proof}[Proof of Proposition \ref{prop:microfound}]
We prove Proposition \ref{prop:microfound} by construction and show that there is an $M$ that can microfound any distribution map.
We construct $M$ as follows. We construct response types $t_0$ and $t_1$, and define $M(x, 0) = t_0$ for all $x \in X$ and $M(x,1) = t_1$ for all $x \in X$. In other words, we associate agents with true label $0$ with the type $t_0$ and agents with true label $1$ with the type $t_1$.

In order to construct $t_0$ and $t_1$, we define the following probability measures over the measure space $X \subseteq \mathbb{R}^D$ equipped with the Borel sigma-algebra. We consider $\mu^0(\theta)$ to be the probability measure given by the distribution over $x$ when $(x,y) \in \DMap(\theta)$ and $y = 0$. We define $\mu^1(\theta)$ similarly. We let $\mu_{\text{XY}}^0$ be the probability measure given by the distribution $x$ where $(x,y) \in \DBase$ and $y = 0$, and we define $\mu_{\text{XY}}^1$ analogously. 

First, we claim that it suffices to prove that for each $\theta \in \Theta$ there is a measurable map $f_{0,\theta}: X \rightarrow X$ that maps the probability measure $\mu_{\text{XY}}^0$ to $\mu^0(\theta)$, and a measurable map $f_{1,\theta}: X \rightarrow X$ that maps $\mu_{\text{XY}}^1$ to $\mu^1(\theta)$. In this case, we can define $t_0$ to be given by $\Response_{t_0}(x, \theta) = f_{0, \theta}(x)$ and $t_1$ to be given by $\Response_{t_1}(x, \theta) = f_{1, \theta}(x)$. Let's now consider the distribution  given by $(\Response_{t}(x, \theta), y)$ where $(t,x,y) \sim \DTrue$. The condition distribution over $y = 0$ is given by $\mu^0(\theta)$ and the conditional distribution over $y = 1$ is given by $\mu^1(\theta)$, which means that the distribution over all is given by $\DMap(\theta)$, as desired. Moreover, the measurability requirements on $f_{0,\theta}$ and $f_{1,\theta}$ guarantee that Assumption \ref{assumption:measurability} is satisfied. 

Thus, it suffices to construct $f_{0,\theta}$ and $f_{1,\theta}$ for $\theta \in \Theta$ that satisfy the above conditions. To do this, we make use of Proposition 3 in \citep{G02}, which says that there exists a Borel mapping from any tight non-atomic measure to any other probability measure. Since the probability measure  associated to $\DBase$ is non-atomic, we see that $\mu_{\text{XY}}^0$ and $\mu_{\text{XY}}^1$ are non-atomic as desired, and so a Borel mapping from $\mu_{\text{XY}}^0$ to $\mu^0(\theta)$ exists and a Borel mapping from $\mu_{\text{XY}}^1$ to $\mu^1(\theta)$ exists. 
\end{proof}

\subsection{Proof of Proposition \ref{prop:1dsmooth}}\label{appendix:1dsmooth}
\begin{proof}[Proof of Proposition \ref{prop:1dsmooth}]
To prove Proposition \ref{prop:1dsmooth} we  show that $\DPR{\theta}{ \theta'}$ is continuous in $\theta$ and $\theta'$. We see that
\[\PRDecoupled{\theta}{ \theta'} = \int_{x' \ge \theta'} p_{\theta}((x', 0)) dx' + \int_{x' < \theta'} p_{\theta}((x', 1)) dx'. \]
Let's take a derivative with respect to $\theta'$ to obtain:
\[\DPR{\theta}{\theta'} = -p_{\theta}((\theta', 0)) + p_{\theta}((\theta',1)). \]
The first continuity requirement tells us that this is continuous in $\theta'$, and the second continuity requirement tells us that this continuous in $\theta$. 
\end{proof}

\subsection{Proof of Theorem \ref{lemma:existence}}
\label{app:proofsmacro}

We first recall the definition of the decoupled performative risk \citep{PZMH20}:
\[ \PRDecoupled{\theta}{ \theta'} := \mathbb{E}_{(x,y) \in \DMap(\theta)} \left[\mathds{1} \left(f_{\theta'}(x) \neq y\right) \right].\]
\noindent The gradient of the decoupled performative risk plays an important role in our analysis of locally stable points. In order to take derivatives at the boundary, we consider an open set $\Theta' \supset \Theta$ that is also bounded and convex, and assume there are classifiers associated with each $\theta \in \Theta'$, although the decision-maker only considers classifier weights in $\Theta$. We use the notation:
\[\DPR{\theta}{ \theta'} := \nabla_{\theta'} \PRDecoupled{\theta}{\theta'} = \nabla_{\theta'} \mathbb{E}_{(x,y)\sim \cD(\theta)} [\;\Indicator\{y\neq f_{\theta'}(x)\}] \]
to denote the gradient of the decoupled performative risk with respect to the second argument. 
To prove Theorem \ref{lemma:existence} we show that the continuity of the derivatives of the decoupled performative risk guarantees the existence of stable points under mixtures with non-strategic agents. 

\begin{proof}[Proof of Theorem \ref{lemma:existence}]
Our main technical ingredient is this proof is applying Brouwer's fixed point theorem on the map $G_{\text{RGD}}: \Theta \rightarrow \Theta$ given by $G_{\text{RGD}}(\theta) = \text{Proj}_{\Theta}(\theta + \eta \DPR{\theta}{\theta} )$. By assumption, we know that $\Theta$ is compact and convex. It thus suffices to show that the map $\theta \mapsto \text{Proj}_{\Theta}(\theta + \eta \DPR{\theta}{\theta} )$ is continuous. 
 
First, we show that aggregate risk smoothness implies that $\theta \mapsto \text{Proj}_{\Theta}(\theta + \eta \DPR{\theta}{\theta} )$ is a continuous map. By aggregate risk smoothness, we know that $\DPR{\theta}{ \theta}$ exists for all $\theta \in \Theta$. Moreover, for any $\theta \in \Theta$, aggregate risk smoothness tells us that:
\[\lim_{\theta' \rightarrow \theta} \norm{\DPR{\theta}{\theta} - \DPR{\theta'}{\theta'}} \le \lim_{\theta' \rightarrow \theta} \norm{\DPR{\theta}{\theta} - \DPR{\theta'}{\theta}} + \lim_{\theta' \rightarrow \theta} \norm{\DPR{\theta'}{\theta} - \DPR{\theta'}{\theta'}} .\] Thus, $\DPR{\theta}{ \theta}$ is continuous in $\theta$. Moreover, since the sum of continuous functions is continuous, this means that $\theta \mapsto \theta + \eta \DPR{\theta}{\theta}$ is continuous. Now, since projection onto a convex set is a contraction map, we can conclude that $\theta \mapsto \text{Proj}_{\Theta}(\theta + \eta \DPR{\theta}{\theta} )$ is continuous as desired. 

\end{proof}

\subsection{Proof of Proposition \ref{cor:mixtures}} 

We now prove Proposition \ref{cor:mixtures}. 
\begin{proof}[Proof of Proposition \ref{cor:mixtures}]
Since derivatives are linear, we can break $\DPR{\theta}{\theta}$ into a term for non-strategic agents and a term for strategic agents. Since the sum of two continuous function is continuous, it suffices to show that $\DPR{\theta}{\theta}$ exists and is continuous for non-strategic agents and for strategic agents. For strategic agents, this follows from aggregate risk smoothness. For non-strategic agents, since the (non-performative) risk $R(\theta) := \mathbb{E}_{(x,y) \in \DBase} \Indicator\{f_{\theta}(x) = y\}$ is differentiable in $\theta$ and $\DPR{\theta}{\theta} = \nabla_{\theta} R(\theta)$ is continuous in $\theta$ as desired. 
\end{proof}

\subsection{Proof of Proposition \ref{prop:samplecomplexity}}
We consider the 1-dimensional setting where $X \subseteq \mathbb{R}$ and $\Theta$ is the class of threshold functions. First, we show a bound on the performative risk in terms of the Kolmogrov-Sminorff (KS) distance between the true distribution map and estimated distribution map. To state this bound, we introduce the following notation. We use a subscript notation $\DMap_{S(\Theta_0,c)}(\theta; M)$ to denote the aggregate response distribution $\DMap(\theta; M)$ restricted to agents with true features $x \in S(\Theta_0,c)$, where $S(\Theta_0, c)$ is defined as in \eqref{eq:S}.  Let $\DMap^{0}_{S(\Theta_0, c)}(\theta; M)$ be the marginal distribution over $x$ of the conditional distribution of $(x, y) \sim \DMap_{S(\Theta_0, c)}(\theta; M)$ conditional on $y = 0$. We define $\DMap^{1}_{S(\Theta_0, c)}(\theta; M)$, $\DMap^{0}_{S(\Theta_0, c)}(\theta; \tilde{M})$, and $\DMap^{1}_{S(\Theta_0, c)}(\theta; \tilde{M})$ analogously.
\begin{lemma}
\label{lem:estdistmapKS}
Let $\Theta$ be a function class of threshold functions, and $c$ be an outcome-valid cost function. Suppose that $M, \tilde{M}$ restricted to the domain $(X \setminus S(\Theta_0, c)) \times Y$ are expenditure-constrained. Then, for any $\Theta_0\subseteq \Theta:\thetaPO(M)\in\Theta_0$, the predicted performative optima $\thetaPO(\tilde{M})$ satisfies\footnote{We note that the decision-maker must only search over $\Theta_0$ in  \eqref{eq:PO_M} when computing $\thetaPO(\tilde{M})$. In particular, they must compute $\argmin_{\theta\in\Theta_0} \mathbb E_{(x,y)\sim\cD(\theta;\tilde M)} \left[\mathds{1}\{y\neq f_\theta(x)\}\right]$ rather than $\argmin_{\theta\in\Theta} \mathbb E_{(x,y)\sim\cD(\theta;\tilde M)} \left[\mathds{1}\{y\neq f_\theta(x)\}\right]$.}: 
 \[\PR_M(\thetaPO(\tilde{M}) \le \PR_M(\thetaPO(M))) + 2  \xi\]
 where $\xi$ is defined to be 
 \[\sup_{\theta} (A_{\theta} + B_{\theta}) \]
 where 
 \begin{align*}
     A(\theta) := \mathbb{P}[x \in S(\Theta_0, c) \And y = 0 ] \KS\left(\DMap^{0}_{S(\Theta_0, c)}(\theta; M), \DMap^{0}_{S(\Theta_0, c)}(\theta; \tilde{M})\right) \\
     B(\theta) := \mathbb{P}[x \in S(\Theta_0, c) \And y = 1 ] \KS\left(\DMap^{1}_{S(\Theta_0, c)}(\theta; M), \DMap^{1}_{S(\Theta_0, c)}(\theta; \tilde{M})\right).
 \end{align*}
\end{lemma}
\begin{proof}
Let $\PR(\theta; M)$ denote the performative risk at $\theta$ on $\DMap(\theta; M)$ and let $\PR(\theta; \tilde{M})$ denote the performative risk at $\theta$ on $\DMap(\theta; \tilde{M})$. 

First, we claim that it suffices to show that $|\PR(\theta; M) - \PR(\theta; \tilde{M})| \le \xi$ for all $\theta \in \Theta_0$. This is because the decision-maker only searches over $\Theta_0$ when computing $\thetaPO(\tilde{M})$ and the true optimal point $\thetaPO(M)$ lies in $\Theta_0$, and so if $|\PR(\theta; M) - \PR(\theta; \tilde{M})| \le \xi$ for all $\theta \in \Theta_0$, then 
\[\PR(\thetaPO(\tilde{M}); M) \le \PR(\thetaPO(\tilde{M}); \tilde{M}) + \xi \le \PR(\thetaPO(M); \tilde{M})  + \xi \le \PR(\thetaPO(M); M) + 2\xi,\]
as desired.

The remainder of the proof boils down to showing that  $|\PR(\theta; M) - \PR(\theta; \tilde{M})| \le \xi$ for all $\theta \in \Theta_0$. 
Notice that: 
\[|\PR(\theta; M) - \PR(\theta; \tilde{M})| = \left|\mathbb{E}_{(x,y) \sim \DMap(\theta; M)}[\;\Indicator\{y \neq f_{\theta}(x)\}] - \mathbb{E}_{(x, y) \sim \DMap(\theta; \tilde{M})}[\;\Indicator\{y \neq f_{\theta}(x)\}]\right|.\]

Let's let $\DTrue$ be the distribution of $(t, x, y)$ where $(x,y) \sim \DBase$ and $t \sim M(x,y)$. Similarly, let $\DApprox$ be the distribution of $(t, x, y)$ where $(x,y) \sim \DBase$ and $t \sim \tilde{M}(x,y)$. We can thus obtain that:
\[|\PR(\theta; M) - \PR(\theta; \tilde{M})| =  \left|\mathbb{E}_{(t, x, y) \sim \DTrue}[\;\Indicator\{y \neq f_{\theta}(\Response_t(x, \theta))\}] - \mathbb{E}_{(t, x, y) \sim \DApprox}[\;\Indicator\{y \neq f_{\theta}(\Response_t(x, \theta))\}]\right|.\]
\noindent Now, we claim that for any agent $(t,x)$ where $x\not\in S(\Theta_0, c)$ and for $t \in \text{supp}(\DTrue) \cup \text{supp}(\DApprox)$, it holds that $f_{\theta}(\Response_t(x, \theta)) = f_{\theta}(x)$ for every $\theta \in \Theta_0$. Note that since $M$ satisfies the expenditure constraint with respect to $c$, then we know that if $x\not\in S(\Theta_0, c)$, it holds that $f_{\theta}(\Response_t(x, \theta)) = f_{\theta}(x)$. This yields the desired statement. Thus we have that:
\small{
\begin{align*}
   &\left|\mathbb{E}_{(t, x, y) \sim \DTrue}[\;\Indicator\{y \neq f_{\theta}(\Response_t(x, \theta))\}] - \mathbb{E}_{(t, x, y) \sim \DApprox}[\;\Indicator\{y \neq f_{\theta}(\Response_t(x, \theta))\}]\right|  \\
   &\le  \left|\mathbb{E}_{(t, x, y) \sim \DTrue}[\;\Indicator\{y \neq f_{\theta}(\Response_t(x, \theta))\} \;\Indicator\{x \not\in S(\Theta_0, c) \}] - \mathbb{E}_{(t, x, y) \sim \DApprox}[\;\Indicator\{y \neq f_{\theta}(\Response_t(x, \theta))\}\;\Indicator\{x \not\in S(\Theta_0, c) \}]\right| \\
   &+ \left|\mathbb{E}_{(t, x, y) \sim \DTrue}[\;\Indicator\{y \neq f_{\theta}(\Response_t(x, \theta))\} \;\Indicator\{x \in S(\Theta_0, c) \}] - \mathbb{E}_{(t, x, y) \sim \DApprox}[\;\Indicator\{y \neq f_{\theta}(\Response_t(x, \theta))\}\;\Indicator\{x \in S(\Theta_0, c) \}]\right| \\
     &=  \left|\mathbb{E}_{(t, x, y) \sim \DTrue}[\;\Indicator\{y \neq f_{\theta}(\Response_t(x, \theta))\} \;\Indicator\{x \in S(\Theta_0, c) \}] - \mathbb{E}_{(t, x, y) \sim \DApprox}[\;\Indicator\{y \neq f_{\theta}(\Response_t(x, \theta))\}\;\Indicator\{x \in S(\Theta_0, c) \}]\right| \\
    &= \left|\mathbb{E}_{(x,y) \sim \DMap(\theta)}[\;\Indicator\{y \neq f_{\theta}(x)\} \;\Indicator\{x \in S(\Theta_0, c) \}] - \mathbb{E}_{(x,y) \sim \tilde{\DMap}(\theta)}[\;\Indicator\{y \neq f_{\theta}(x)\} \;\Indicator\{x \in S(\Theta_0, c) \}]\right|.
\end{align*}
}
\normalsize{}
We can break this into terms where $y = 0$ and terms where $y = 1$. Thus, it suffices to bound: 
\small{
\[E_1 := \left|\mathbb{E}_{(x,y) \sim \DMap(\theta; M)}[\;\Indicator\{y \neq f_{\theta}(x)\} \;\Indicator\{x \in S(\Theta_0, c) \} \;\Indicator\{y = 0\}] - \mathbb{E}_{(x,y) \sim \DMap(\theta; \tilde M)}[\;\Indicator\{y \neq f_{\theta}(x)\} \;\Indicator\{x \in S(\Theta_0, c) \} \;\Indicator\{y = 0\}] \right| \]}
\normalsize{}
and 
\normalsize{}
\small{
\[E_2 := \left|\mathbb{E}_{(x,y) \sim \DMap(\theta; M)}[\;\Indicator\{y \neq f_{\theta}(x)\} \;\Indicator\{x \in S(\Theta_0, c) \} \;\Indicator\{y = 1\}]- \mathbb{E}_{(x,y) \sim \DMap(\theta; \tilde M)}[\;\Indicator\{y \neq f_{\theta}(x)\} \;\Indicator\{x \in S(\Theta_0, c) \} \;\Indicator\{y = 1\}]\right|.\]}
\normalsize{}
It suffices to show that the first term is upper bounded by $A(\theta)$  and the second term is upper bounded by $B(\theta)$. Since these two bounds follow from analogous  arguments, we only present the proof of the first bound. 
\begin{align*}
    E_1 &= \mathbb{P}[x \in S(\Theta_0, c) \And y = 0 ]  \left|\mathbb{E}_{(x,y) \sim \DMap^{0}_{S(\Theta_0, c)}(\theta; M)}[\;\Indicator\{x \ge \theta\}] - \mathbb{E}_{(x,y) \sim \DMap^{0}_{S(\Theta_0, c)}(\theta; \tilde M)}[\;\Indicator\{x \ge \theta\}] \right|\\
    &= \mathbb{P}[x \in S(\Theta_0, c) \And y = 0 ]  \left|\mathbb{E}_{l \sim \DMap^{0}_{S(\Theta_0, c)}(\theta; M)}[\;\Indicator\{l \ge \theta\}] - \mathbb{E}_{l \sim \DMap^{0}_{S(\Theta_0, c)}(\theta; \tilde M)}[\;\Indicator\{l \ge \theta\}] \right| \\
    &\le \mathbb{P}[x \in S(\Theta_0, c) \And y = 0 ]  \KS\left(\DMap^{0}_{S(\Theta_0, c)}(\theta; M), \DMap^{0}_{S(\Theta_0, c)}(\theta; \tilde M)\right). 
\end{align*}
\end{proof}

\noindent
Now, we are ready to prove Proposition \ref{prop:samplecomplexity}. 

\begin{proof}[Proof of Proposition \ref{prop:samplecomplexity}]
Let $\Theta_{\text{net}}$ be an $\epsilon$ net of $\Theta_{0}$. The decision-maker uses the agent response oracle as follows. For each $\theta \in \Theta_{\text{net}}$, they can generate $n_0$ samples as follows: draw a sample $(x,y) \sim \DBase$ conditioned on $y = 0$. If $x \in S(\Theta_0, c)$, then query the agent response oracle on $x$ at $\theta$. It is easy to see that these samples are distributed as $n_0$ independent samples from $\DMap^{0}_{S(\Theta_0, c)}(\theta)$. Similarly, the decision-maker uses the agent response oracle to draw $n_1$ samples that are distributed as $n_1$ independent samples from $\DMap^{1}_{S(\Theta_0, c)}(\theta)$. (We will specify the values of $n_0$ and $n_1$ later.) 

First, we define a distribution map $\tilde{\DMap}$ using these samples and the base distribution. Let's define $D_0(\theta)$ to be the empirical distribution of the $n_0$ samples, and let $D_1(\theta)$ be the empirical distribution of the $n_1$ samples. Let $D'(\theta)$ be the distribution given by a mixture of $(x,0)$ where $x \sim D_0(\theta)$ with probability $\mathbb{P}_{\DBase}[y =0 \mid x \in S(\Theta_0, c)]$ and $x \sim D_1(\theta)$ with probability $\mathbb{P}_{\DBase}[y =1 \mid x \in S(\Theta_0, c)]$. Let $D''(\theta)$ be the distribution given by $(x,y)$ drawn from the conditional distribution of $\DBase$ given $x \not\in S(\Theta_0,c)$. We let $\tilde{\DMap}$ be the distribution given by a mixture of $D'(\theta)$ with probability $\mathbb{P}_{\DBase}[x \in S(\Theta_0, c)]$ and $D''(\theta)$ with probability $1 -\mathbb{P}_{\DBase}[x \in S(\Theta_0, c)]$.

We can microfound $\tilde{\DMap}$ with a map $\tilde{M}$ as follows. Let $\tilde{M}(x,y) = x$ when $x \not\in S(\Theta_0,c)$. Let $\tilde{M}$ on $S(\Theta_0,c) \times Y$ be defined in such any way it microfounds $D'(\theta)$ (this is possible because of Proposition \ref{prop:microfound}). It is easy to see that $\tilde{M}$ microfounds $\tilde{\DMap}$ and that $\tilde{M}$ restricted to the domain $(X \setminus S(\Theta_0, c)) \times Y$ is expenditure-constrained. This means that we can apply Lemma \ref{lem:estdistmapKS}. 

Now, we bound the performative risk $\PR(\thetaPO(\tilde{M}))$, where:
\[\thetaPO(\tilde{M}) =\argmin_{\theta\in\Theta_0} \mathbb E_{(x,y)\sim\DMap(\theta; \tilde{M})} \left[\Indicator\{y\neq f_\theta(x)\}\right] = \argmin_{\theta\in\Theta_{\text{net}}} \mathbb E_{(x,y)\sim\DMap(\theta; \tilde{M})} \left[\Indicator\{y\neq f_\theta(x)\}\right].\] 
In order to apply Lemma \ref{lem:estdistmapKS}, we need to bound:
\begin{equation}
\label{eq:mainterm}
    \sup_{\theta \in \Theta_{\text{net}}} \left\{A(\theta) + B(\theta) \right\}
\end{equation}
where:
\begin{align*}
     A(\theta) := \mathbb{P}[x \in S(\Theta_0, c) \And y = 0 ] \cdot\KS\left(\DMap^{p,0}_{S(\Theta_0, c)}(\theta; M), \DMap^{p,0}_{S(\Theta_0, c)}(\theta; \tilde{M})\right) \\
     B(\theta) := \mathbb{P}[x \in S(\Theta_0, c) \And y = 1 ] \cdot\KS\left(\DMap^{p,1}_{S(\Theta_0, c)}(\theta; M), \DMap^{p,1}_{S(\Theta_0, c)}(\theta; \tilde{M})\right).
\end{align*}
To bound \eqref{eq:mainterm}, we union bound over $\Theta_{\text{net}}$. This set has cardinality $O(1/\epsilon)$. Notice that with probability $\ge 1- \alpha$, we know that:
\[\KS\left(\DMap^{p,0}_{S(\Theta_0, c)}(\theta; M), \DMap^{p,0}_{S(\Theta_0, c)}(\theta; \tilde{M})\right) \le \sqrt{\frac{\ln(2/\alpha)}{2n_0}}.\]
\[\KS\left(\DMap^{p,1}_{S(\Theta_0, c)}(\theta; M), \DMap^{p,1}_{S(\Theta_0, c)}(\theta; \tilde{M})\right) \le \sqrt{\frac{\ln(2/\alpha)}{2n_0}}.\]
We can now set $\alpha = \Theta(\epsilon / 100)$ in the previous result to obtain that with probability $\ge 99/100$, the expression in \eqref{eq:mainterm} is bounded by: 
\[E := O\left(\mathbb{P}_{\DBase}[x \in S(\Theta_0, c) \wedge y = 0] \sqrt{\frac{\ln(2/\epsilon)}{2n_0}}\right) + O\left(\mathbb{P}_{\DBase}[x \in S(\Theta_0, c) \wedge y = 1] \sqrt{\frac{\ln(2/\epsilon)}{2n_1}}\right).\]
We can now apply Lemma \ref{lem:estdistmapKS} to $\Theta_{\text{net}}$ to see that:
\[\PR(\thetaPO(\tilde{M})) \le E + \min_{\theta\in \Theta_{\text{net}}} \mathbb E_{(x,y)\sim\cD(\theta; M)} \left[\Indicator\{y\neq f_\theta(x)\}\right], \]

Now, let's use the Lipschitz requirement on the distribution map to move to the set $\Theta_{0}$. Let's consider a distribution map $\DMap'$ that is defined as follows: for $\theta \in \Theta_{\text{net}}$, we take $\DMap'(\theta) := \DMap(\theta)$ , and for $\theta \not\in \Theta_{\text{net}}$, we take $\DMap'(\theta) := \DMap(\theta')$ where $\theta'$ is the closest element in $\Theta_{\text{net}}$ to $\theta$. Now, let's apply Lemma \ref{lemma:TVbound} to $\DMap$ and $\DMap'$ on $\Theta_0$ to obtain that: 
\begin{align*}
\min_{\theta\in\Theta_{\text{net}}} \mathbb E_{(x,y)\sim\cD(\theta; M)} \left[\Indicator\{y\neq f_\theta(x)\}\right]) &\le \epsilon + \min_{\theta\in \Theta_{0}} \mathbb E_{(x,y)\sim\cD(\theta; M)} \left[\Indicator\{y\neq f_\theta(x)\}\right] \\
&= \epsilon + \min_{\theta\in \Theta} \mathbb E_{(x,y)\sim\cD(\theta; M)} \left[\Indicator\{y\neq f_\theta(x)\}\right].
\end{align*}
This means that 
\[\PR(\thetaPO(\tilde{M}))  \le E + \epsilon + \PR(\thetaPO(M)).  \]

Thus, it suffices to bound $E$ and set $n_0$ and $n_1$ appropriately. Suppose that
\[n_0 = \Theta\left(\mathbb{P}_{\DBase}[x \in S(\Theta_0, c) \wedge y = 0]^2 \frac{\ln(1/\epsilon)}{2\epsilon^2}  \right)\] and 
\[n_1 = \Theta\left(\mathbb{P}_{\DBase}[x \in S(\Theta_0, c) \wedge y = 1]^2 \frac{\ln(1/\epsilon)}{2\epsilon^2} \right).\] Plugging in these expressions into the expression for $E$, we obtain the desired bounds. Moreover, notice that the total number of queries to the oracle is $\Theta(1/\epsilon) \cdot (n_0 + n_1) \le \Theta\left(\mathbb{P}_{\DBase}[x \in S(\Theta_0, c)]^2 \frac{\ln(1/\epsilon)}{2\epsilon^3} \right)$.

\end{proof}

\section{Proofs for Section \ref{sec:framework}}

\subsection{Proof of Proposition \ref{thm:continuous}}\label{appendix:continuous}
\begin{proof}[Proof of Proposition \ref{thm:continuous}]
We use the following notation for this proof. Let's extend the cost function to be defined and valid on all of $\mathbb{R}$ rather than just $X$. For $x \in X$, let's use the notation $l_x \in \mathbb{R}$ to denote the unique value such that $l_x < x$ and $c(l_x, x) = 1$. Similarly, let $u_x \in \mathbb{R}$ denote the unique value such that $u_x > x$ and $c(x, u_x) = 1$. These values are unique by the definition of a valid cost function.

Fix $\sigma \in (0,\infty)$, and $x' \in X$. Let's characterize the agents who will change their features to $x'$ when the threshold is $\theta$. Either the agents' true features are equal to $x'$ and their perception function $P(\theta) \not\in (x', u_{x'}]$, or the agents' perception function $P(\theta) = x'$ and their true features $x$ are in $[l_{x'}, x']$. Since the base distribution and the noise distribution are continuous, this means that there are no point masses in the distribution. To see that a probability density function exists everywhere and is continuous, let's compute the density. Let $p_{\text{base}}$ denote the pdf of the base distribution (which is assumed to exist and be continuous since $\DBase$ is a continuous distribution), and let $p_{\text{noise}}$ denote the pdf of $D$ (which is continuous since it is the pdf of a gaussian). Notice that the probability density of $\DMap(\theta)$ at $(x', y')$ is 
\[p_{\text{base}}((x',y')) \cdot \mathbb{P}_{D}[\eta \not\in (x' - \theta, u_{x'} - \theta)] + p_{\text{noise}}(x' - \theta) \cdot \mathbb{P}_{\DBase}[x \in [l_{x'}, x'], y = y'].\]

This is continuous in $x'$ because $u_{x'}$ and $l_{x'}$ are continuous in $x'$. Moreover, this is nonzero on all $x'$ because for all $x' \in X$, we see that $p_{\text{base}}((x', y')) > 0$ and $\mathbb{P}_{D}[\eta \not\in (x' - \theta, u_{x'} - \theta)] > 0$ as well.

Now, we show aggregate smoothness. We see that the probability density $p_{\theta}((x', y'))$ at $(x',y')$ is continuous in $x'$ because each term is continuous in $x'$. Similarly, we see that this is continuous in $\theta$ because each term is continuous in $\theta$. By Proposition \ref{prop:1dsmooth}, this implies aggregate smoothness. 
\end{proof}

\subsection{Proof of Lemma \ref{lemma:tvlipschitz}}
\begin{proof}[Proof of Lemma \ref{lemma:tvlipschitz}]
For a given classifier $f_{\theta}$, consider the product distribution $\mathcal{D}_{\text{prod}}(\sigma, \theta)$ of the base distribution $\DBase$ and the multivariate gaussian distribution $\mathcal{N}(\bf \theta, \sigma \cdot I)$. This is a distribution over $X \times \mathbb{R}^d$ that will describe the distribution over noise vectors and features vectors. That is, an agent $A$ with features $x_A$ and noise $\eta_A$ corresponds to $(x_A, \theta + \eta_A)$. If we apply the function in \eqref{eq:RFP} to  $\mathcal{D}_{\text{prod}}$ so that $(x_A, \theta'_A) \mapsto \argmax_{x'\in \mathbb{R}^d} \; \left[\gamma \cdot f_{\theta'_A}(x') - \costfn(x_A, x')\right]$, then it is easy to see that we obtain the distribution $\DMap(\theta)$. 

Since the total variation distance can only decrease when we apply a function to the distributions, we know that $\TV(\DMap(\theta_1), \DMap(\theta_2)) \le \TV(\mathcal{D}_{\text{prod}}(\sigma, \theta_1), \mathcal{D}_{\text{prod}}(\sigma, \theta_2))$. Thus, it suffices to bound $\TV(\mathcal{D}_{\text{prod}}(\sigma, \theta_1), \mathcal{D}_{\text{prod}}(\sigma, \theta_2))$. Using the properties of product distributions, we see that:
\begin{align*}
    \TV(\mathcal{D}_{\text{prod}}(\sigma, \theta_1), \mathcal{D}_{\text{prod}}(\sigma, \theta_2)) &\le \TV(\mathcal{N}(\bf \theta_1, \sigma \cdot I), \mathcal{N}(\bf \theta_2, \sigma \cdot I)) \\
    &\le \frac{1}{2 \sigma} \norm{\theta_1 - \theta_2}_2.
\end{align*}
Now, the result follows from a bound on the total variation distance between two multivariate gaussians (e.g. see Proposition 2.1 in \cite{TVbound})
\end{proof}

\subsection{Social burden of noisy response in general}
We show that for any valid cost function, noisy response results in an optimal point with no higher social burden than the optimal point deduced from standard microfoundations. 
\begin{proposition}
\label{prop:expenditurerationality}
Let $\sigma \in (0, \infty)$, and let $c$ be a valid cost function. Consider a 1-dimensional setting where $X \subseteq \mathbb{R}$ and $\Theta$ is a function class of threshold functions. Then, the following holds: 
\begin{align*}
    &\quad\quad\quad\thetaPO(M_{\mathrm{SM}}) \geq \thetaPO(\Fuzzy{\sigma})\\
    &\SocialBurden{\thetaPO(M_{\mathrm{SM}})} \geq \SocialBurden{\thetaPO(\Fuzzy{\sigma})},
\end{align*}
where $M_{\mathrm{SM}}$ is the mapping induced by standard microfoundations.
\end{proposition}
\begin{proof}
By Proposition \ref{prop:socialburdenformal}, it suffices to show that $\Fuzzy{\sigma}$ satisfies expenditure monotonicity and Assumption \ref{assumption:gamingbehavior}. The fact that $\Fuzzy{\sigma}$ satisfies Assumption \ref{assumption:gamingbehavior} follows from its definition. For expenditure monotonicity, note that the first condition follows from the fact that the optimization problem in \eqref{eq:RFP} tells us that fuzzy perception agents never exceed their utility of a positive outcome from manipulation expenditure. We now show that the second condition is satisfied. Note that each agents' perception function takes the form $P(\theta) = \theta + \eta$ for some fixed $\eta$. Thus, any given agent either consistently overshoots or consistently undershoots the threshold. If $\eta < 0$, then the agent will only be positively classified if and only if $\theta \le x$ where $x$ are the agent's true features. If $\eta > 0$, then the agent will be positively classified if and only if $c(x, \theta + \eta) \le 1$ or $\theta \le x$. This proves the desired statement. 
\end{proof}

\subsection{Proof of Proposition \ref{cor:socialburden}}

\begin{proof}[Proof of Proposition \ref{cor:socialburden}]
By Proposition \ref{prop:expenditurerationality}, we see that $\thetaPO(\DMap_{\text{SM}}) \ge \thetaPO(\DMap)$. It thus suffices to show that $\thetaPO(\DMap_{\text{SM}}) > \thetaPO(\DMap)$. To show this, it suffices to show that the derivative of the performative risk exists and is nonzero at $\thetaPO(\DMap_{\text{SM}})$. 

Like in the proof of Theorem \ref{thm:continuous}, we use the notation $l_x$, $u_x$, $p_{\text{base}}$, and $\mathbb{P}_{D}$. By the expenditure constraint, we know that agents with true features $x \le l_{\theta}$ will all be classified as $0$, and so their net contribution to the performative risk is $\int_{-\infty}^{l_{\theta}}
 p_{\text{base}}((x, 1))dx$. By the properties of noisy response, we know that agents with true features $x \ge \theta$ will be classified as $1$, so  their net contribution to the performative risk is $\int_{-\infty}^{l_{\theta}}
 p_{\text{base}}((x, 1))dx$. For agents with true features $x \in (\theta - 1, \theta)$, agents will be classified as $1$ if and only if they manipulate features to $x' \ge \theta$. This will happen if and only if they perceive the threshold to be in the range $[\theta, u_x]$. This will happen if and only if their noise vector $\eta$ is in $[0, u_{x} - \theta]$. Putting this all together, we see that the performative risk is thus equal to:
\small{
\begin{align*}
  \PR(\theta) &= \int_{\theta}^{\infty}
 p_{\text{base}}((x, 0))dx + \int_{-\infty}^{\theta - 1}
 p_{\text{base}}((x, 1))dx + \int_{\theta-1}^{\theta}  p_{\text{base}}((x', 0)) \mathbb{P}_{D}[\eta \in [0, u_{x'} - \theta]] dx' \\
 &+ \int_{\theta-1}^{\theta}  p_{\text{base}}((x', 1)) \mathbb{P}_{D}[\eta \not\in [0, u_{x'} - \theta]] dx' \\
 &= \int_{\theta}^{\infty}
 p_{\text{base}}((x, 0))dx + \int_{-\infty}^{\theta - 1}
 p_{\text{base}}((x, 1))dx + \int_{\theta-1}^{\theta}  p_{\text{base}}((x', 0)) \mathbb{P}_{D}[\eta \in [0, x' + 1 - \theta]] dx' \\
 &+ \int_{\theta-1}^{\theta}  p_{\text{base}}((x', 1)) \mathbb{P}_{D}[\eta \not\in [0, x' + 1 - \theta]] dx' \\
  &= \int_{\theta}^{\infty}
 p_{\text{base}}((x, 0))dx + \int_{-\infty}^{\theta - 1}
 p_{\text{base}}((x, 1))dx +  \int_{0}^{1}  p_{\text{base}}((\theta -1 + x, 0)) \mathbb{P}_{D}[\eta \in [0, x]] dx \\
 &+ \int_{0}^{1}  p_{\text{base}}((\theta - 1 + x, 1)) \mathbb{P}_{D}[\eta \not\in [0, x]] dx \\
   &= \int_{\theta}^{\infty}
 p_{\text{base}}((x, 0))dx + \int_{-\infty}^{\theta - 1}
 p_{\text{base}}((x, 1))dx + \mathbb{P}_{\DBase}[x \in (\theta -1, \theta), y = 0] \\
 &+ \int_{0}^{1}  (p_{\text{base}}((\theta - 1 + x, 1)) - p_{\text{base}}((\theta - 1 + x, 0)))  \mathbb{P}_{D}[\eta \not\in [0, x]] dx. 
\end{align*}}

\normalsize{}

Let's write $\int_{0}^{1}  (p_{\text{base}}((\theta - 1 + x, 1)) - p_{\text{base}}((\theta - 1 + x, 0)))  \mathbb{P}_{D}[\eta \not\in [0, x]] dx$ in a slightly different form. 
\begin{align*}
    &\int_{0}^{1}  (p_{\text{base}}((\theta - 1 + x, 1)) - p_{\text{base}}((\theta - 1 + x, 0)))  \mathbb{P}_{D}[\eta \not\in [0, x]] dx \\
    &=  (\mathbb{P}_{D}[\eta \in [-\infty,0]] + \mathbb{P}_{D}[\eta \in [1, \infty]]) \int_{0}^{1}  (p_{\text{base}}((\theta - 1 + x, 1)) - p_{\text{base}}((\theta - 1 + x, 0)))  dx \\
    &+ \int_{0}^{1}  (p_{\text{base}}((\theta - 1 + x, 1)) - p_{\text{base}}((\theta - 1 + x, 0))) \mathbb{P}_{D}[\eta \in [x, 1]]  dx\\
     &=  (\mathbb{P}_{D}[\eta \in [-\infty,0]] + \mathbb{P}_{D}[\eta \in [1, \infty]]) (\mathbb{P}_{\DBase}[x \in (\theta - 1, \theta), y =1] - \mathbb{P}_{\DBase}[x \in (\theta - 1, \theta), y =0]) \\
    &+ \int_{0}^{1}  (p_{\text{base}}((\theta - 1 + x, 1)) - p_{\text{base}}((\theta - 1 + x, 0))) \mathbb{P}_{D}[\eta \in [x, 1]]  dx. 
\end{align*}

We can rewrite: 
\begin{align*}
    &\int_{0}^{1}  (p_{\text{base}}((\theta - 1 + x, 1)) - p_{\text{base}}((\theta - 1 + x, 0))) \mathbb{P}_{D}[\eta \in [x, 1]]  dx \\
    &= \int_{0}^{1} \int_{x}^1 (p_{\text{base}}((\theta - 1 + x, 1)) - p_{\text{base}}((\theta - 1 + x, 0))) p_{\text{noise}}(z)  dz dx \\
    &=  \int_{0}^1 p_{\text{noise}}(z)   \int_{0}^{z}   (p_{\text{base}}((\theta - 1 + x, 1)) - p_{\text{base}}((\theta - 1 + x, 0))) dx dz \\
    &= \int_{0}^1 p_{\text{noise}}(z)   (\mathbb{P}_{\DBase}[x \in ((\theta - 1,\theta - 1 + z)), y = 1] - \mathbb{P}_{\DBase}[x \in ((\theta - 1,\theta - 1 + z))), y = 0]) dz \\
\end{align*}

When we take a derivative with respect to $\theta$, we obtain:
\small{
\begin{align*}
 \pd{\PR(\theta)}{\theta} &= -p_{\text{base}}((\theta, 0)) +  p_{\text{base}}((\theta-1, 1)) -p_{\text{base}}((\theta-1, 0)) + p_{\text{base}}((\theta, 0)) \\
   &+ (\mathbb{P}_{D}[\eta \in [-\infty,0]] + \mathbb{P}_{D}[\eta \in [1, \infty]]) (p_{\text{base}}((\theta, 1)) - p_{\text{base}}((\theta - 1, 1)) - p_{\text{base}}((\theta, 0)) + p_{\text{base}}((\theta - 1, 0))) \\
   &+ \int_{0}^1 p_{\text{noise}}(z)   (p_{\text{base}}((\theta-1+z, 1)) - p_{\text{base}}((\theta-1, 1)) - p_{\text{base}}((\theta-1+z, 0)) +p_{\text{base}}((\theta-1, 0))) dz.
\end{align*}}
\normalsize{}

Let's analyze this expression at $\theta = \thetaPO(\DMap_{\text{SM}})$. By the assumptions on the cost function, and using that $\theta_{\text{SL}} + 1 \in \Theta \cap X$, we see that $\thetaPO(\DMap_{\text{SM}}) = \theta_{\text{SL}} + 1$, so $\theta - 1 = \theta_{\text{SL}}$. This means that $p_{\text{base}}((\theta-1, 1))  -p_{\text{base}}((\theta-1, 0)) = p_{\text{base}}((\theta_{\text{SL}}, 1))  -p_{\text{base}}((\theta_{\text{SL}}, 0)) = 0$.
Thus, the expression simplifies to:
\begin{align*}
   \pd{\PR(\theta)}{\theta} &=  (\mathbb{P}_{D}[\eta \in [-\infty,0]] + \mathbb{P}_{D}[\eta \in [1, \infty]]) (p_{\text{base}}((\theta, 1)) - p_{\text{base}}((\theta, 0))) \\
   &+ \int_{0}^1 p_{\text{noise}}(z)   (p_{\text{base}}((\theta-1+z, 1))- p_{\text{base}}((\theta-1+z, 0))) dz.
\end{align*}
We see that $p_{\text{base}}((\theta', 1)) > p_{\text{base}}((\theta', 0))$ for all $\theta' \ge  \theta_{\text{SL}}$ by the assumption on $\mu$ in Setup \ref{ex:stablepointsnotexist}. This implies that the first term is positive and the second term is nonnegative, so $\pd{\PR(\theta)}{\theta}$ is positive as desired.
\end{proof}

\subsection{Proof of Lemma \ref{lemma:tvboundfuzzy}}
\begin{proof}[Proof of Lemma \ref{lemma:tvboundfuzzy}]
Consider the product distribution $\mathcal{D}_{\text{prod}}(\sigma)$ of the base distribution $\DBase$ and the multivariate Gaussian distribution $\mathcal{N}(\bf 0, \sigma \cdot I)$. This is a distribution over $X \times \mathbb{R}^d$ that will describe the distribution over noise vectors and features vectors. That is, an agent $a$ with features $x_a$ and noise $\eta_A$ corresponds to $(x_A, \eta_A)$. If we apply the function in \eqref{eq:RFP} to  $\mathcal{D}_{\text{prod}}$ so that $(x_A, \eta_A) \mapsto \argmax_{x'\in \mathbb{R}^d} \; \left[\gamma \cdot f_{\theta + \eta_A}(x') - \costfn(x_A, x')\right]$, then it is easy to see that we obtain the distribution $\DMap(\theta)$.

Since the total variation distance can only decrease when we apply a function to the distributions, we know that $\TV(\DMap(\theta), \DMap(\tilde{\theta}) \le \TV(\mathcal{D}_{\text{prod}}(\sigma), \mathcal{D}_{\text{prod}}(\tilde{\sigma}))$. Thus, it suffices to bound $\TV(\mathcal{D}_{\text{prod}}(\sigma), \mathcal{D}_{\text{prod}}(\tilde{\sigma}))$. Using the properties of product distributions, we see that
\[\TV(\mathcal{D}_{\text{prod}}(\sigma), \mathcal{D}_{\text{prod}}(\tilde{\sigma})) \le \TV(\mathcal{N}(\bf 0, \sigma \cdot I), \mathcal{N}(\bf 0, \tilde{\sigma} \cdot I)).\] Now, the result follows from the standard bound on the total variation distance between two multivariate gaussians (e.g. see Proposition 2.1 in \cite{TVbound}). 

\end{proof}

\subsection{Proof of Corollary \ref{cor:fuzzyTV}}\label{appendix:naive}
In order to prove Corollary \ref{cor:fuzzyTV}, we show that if the estimated distribution map is sufficiently close to the true distribution map, then the optimal point computed using this model will achieve near-optimal performative risk. 
\begin{restatable}{lemma}{tvbound}
\label{lemma:TVbound}
Let $\tilde{M}$ be an estimate of the true distribution map $M$. Then the suboptimality of the performative risk of $\thetaPO(M)$ as per \eqref{eq:PO_M} is bounded by: 
\[\PR(\thetaPO(\tilde M))- \PR(\thetaPO(M))\leq 2 \sup_{\theta} \left\{\TV\big(\DMap(\theta; M), \DMap(\theta; \tilde{M})\big)\right\},\] where $\PR(\theta) := \mathbb E_{(x,y)\sim \DMap(\theta)} \left[\mathds{1}\{y\neq f_\theta(x)\}\right]$ denotes the performative risk with respect to $M$. 
\end{restatable}
\begin{proof}[Proof of Lemma \ref{lemma:TVbound}]
Let $\xi = \left\{\TV\big(\DMap(\theta; M), \DMap(\theta; \tilde{M})\big)\right\}$. Let $\PR(\theta; M)$ denote the performative risk at $\theta$ on $\DMap(\theta; M)$ and let $\PR(\theta; \tilde{M})$ denote the performative risk at $\theta$ on $\DMap(\theta; \tilde{M})$. It suffices to show that $|\PR(\theta; M) - \PR(\theta; \tilde{M})| \le \xi$ (since this would mean that $\PR(\thetaPO(\tilde{M}); M) \le \PR(\thetaPO(\tilde{M}); \tilde{M}) + \xi \le \PR(\thetaPO(M); \tilde{M})  + \xi \le \PR(\thetaPO(M); M) + 2\xi$, as desired). 
Notice that: 
\[|\PR(\theta; M) - \PR(\theta; \tilde{M})| = \left|\mathbb{E}_{(x,y) \sim \DMap(\theta; M)}[\;\Indicator\{y \neq f_{\theta}(x)\}] - \mathbb{E}_{(t, x, y) \sim \DMap(\theta; \tilde{M})}[\;\Indicator\{y \neq f_{\theta}(x)\}]\right|.\]
Since the indicator variables are always constrained between $0$ and $1$, we can immediately obtain an upper bound of $TV(\DMap(\theta; M),\DMap(\theta; \tilde{M}))$. 
\end{proof}

We can easily deduce Corollary \ref{cor:fuzzyTV} from these above facts. 
\begin{proof}[Proof of Corollary \ref{cor:fuzzyTV}]
We use Lemma \ref{lemma:tvboundfuzzy} to see that 
\[\TV(\DMap(\theta; \Fuzzy{\sigma}), \DMap(\theta, \Fuzzy{\tilde{\sigma}})) \le \frac{1}{2} \sqrt{\frac{|\sigma^2 - \tilde{\sigma}^2| m}{\min(\sigma^2, \tilde{\sigma}^2)}}.\]
Then we use Lemma \ref{lemma:TVbound} to see that:
\[\PR(\thetaPO(\Fuzzy{\tilde{\sigma}}))- \PR(\thetaPO(\Fuzzy{\sigma}))\leq 2 \sup_{\theta} \left\{\TV\big(\DMap(\theta; \Fuzzy{\sigma}), \DMap(\theta; \Fuzzy{\tilde{\sigma}})\big)\right\} \le  \sqrt{\frac{|\sigma^2 - \tilde{\sigma}^2| m}{\min(\sigma^2, \tilde{\sigma}^2)}}. \]

\end{proof}

\section{Lipschitzness in Wasserstein distance}\label{appendix:lipschitz}

In existing performative prediction approaches \citep{PZMH20, MPZH20, BHK20}, it is assumed that the distribution map is Lipschitz with respect to changes in $\theta$. In particular,  
\[\mathcal{W}(\DMap(\theta), \DMap(\theta')) \le C \norm{\theta - \theta'}_2\] 
for some constant $C > 0$, where $\mathcal{W}$ is the Wasserstein-1 distance. 

\subsection{Lipschitzness is not sufficient for stability}
\label{app:lipschitznot sufficient}

We show that for binary classification, Lipschitzness in Wasserstein distance is not sufficient to guarantee the existence of stable points. We construct a simple example of a Lipschitz distribution map for which stable points do not exist. Consider our counterexample in Setup~\ref{ex:stablepointsnotexist} with $p\in(0,1)$: let the cost function of the best responding agents be linear, and consider a uniform base distribution; this results in a Lipschitz distribution map. By Proposition~\ref{propo:stablepointsnotexist}, stable points do not exist.

\subsection{Lipschitzness can be restrictive}
In the context of binary classification, we show that this requirement is quite restrictive on the cost function and the base distribution, even in the context of a 1-d setting and standard microfoundations. Example~\ref{example:violateLipschitz} provides a simple 1-d setting where the aggregate response distribution $\cD(\theta)$ induced by best-responding agents does not satisfy Lipschitzness:

\begin{example}
\label{example:violateLipschitz}
Suppose that $\Theta = [0,1]$ and $X = [-10, 10]$. Let the marginal distribution of the features for $\DBase$ be uniform on $[0, 1]$ and consider the cost function $c(x, y) = |x^2 - y^2|$. The distribution map $\DMap_{\text{NS}}(\theta)$ given by agents who follow standard microfoundations is not Lipschitz in Wasserstein distance.
\end{example}

To prove Example \ref{example:violateLipschitz}, we show the following lemma. 
\begin{lemma}[Violation of Lipschitzness] 
\label{lemma:violation}
Assume that $\Theta$ is $1$-dimensional. Suppose that $c(x,y) \le k |x-y|$ for some constant $k > 0$. For each $\theta \in \Theta$, let $S_{\theta} = \{x\in X: \exists x' \in X : f_{\theta}(x') \neq f_{\theta}(x) \wedge c(x, x') \le 1\}$ be the set of agents who are sufficiently close to the decision boundary that they are able to cross it without expending more than 1 unit of cost. If the following condition holds:
\[\sup_{\theta' > \theta > 0} \frac{\mathbb{P}_{\DBase} [S_{\theta} \setminus S_{\theta'}]}{|\theta' - \theta|} = \infty ,\]
then $\DMap(\theta)$ is not Lipschitz in Wasserstein distance. 
\end{lemma}
\begin{proof}[Proof of Lemma \ref{lemma:violation}]
First, we show that when $\theta' \le \theta - 1/k$ and $\theta' > \theta > 0$, it holds that \[\mathcal{W}(\DMap(\theta), \DMap(\theta')) \ge \frac{1}{2k} \mathbb{P}_{\DBase} [S_{\theta'} \setminus S_{\theta}].\] The distribution map $\DMap(\theta)$ corresponds to $\DBase$ with all of the agents with type $x \in S_{\theta}$ moving up to $\theta$. This means that $W(\DMap(\theta'), \DMap(\theta))$ must move all of the agents with features $x \in S_{\theta'} \setminus S_{\theta}$ up to at least $\theta$. Notice that $\min_{x \in S_{\theta'} \setminus S_{\theta}} |x - \theta| = |\sup(S_{\theta'}) - \theta| \ge |\sup(S_{\theta'}) - \theta'| - |\theta' - \theta| \ge \frac{1}{k} c(\sup(S_{\theta'}), \theta') - |\theta' - \theta| \ge \frac{1}{k} - |\theta' - \theta| \ge \frac{1}{2k}$. This means that $\mathcal{W}(\DMap(\theta), \DMap(\theta')) \ge \frac{1}{2k} \mathbb{P}_{\DBase} [S_{\theta'} \setminus S_{\theta}]$, as desired. 

Now, suppose that the condition $\sup_{\theta' > \theta > 0} \frac{\mathbb{P}_{\DBase} [S_{\theta'} \setminus S_{\theta}]}{|\theta' - \theta|} = \infty$ holds. Notice that this implies that 
\[\sup_{\theta' > \theta > 0, \theta' \le \theta - 1/k} \frac{\mathbb{P}_{\DBase} [S_{\theta'} \setminus S_{\theta}]}{|\theta' - \theta|} = \infty.\] This implies that $\sup_{\theta' > \theta > 0, \theta' \le \theta - 1/k} \frac{\mathcal{W}(\DMap(\theta), \DMap(\theta'))}{|\theta' - \theta|} = \infty$, and so the Lipschitzness constraint is violated.
\end{proof}

Now, we prove Example \ref{example:violateLipschitz} from Lemma \ref{lemma:violation}.
\begin{proof}[Proof of Example \ref{example:violateLipschitz}]
We apply Lemma \ref{lemma:violation}. First, observe that $c(x,y) = |x^2 - y^2| = |x - y| |x + y| \le 20 |x - y|$ as desired. Thus, it suffices to show that 
\[\sup_{\theta' > \theta > 0} \frac{\mathbb{P}_{\DBase} [S_{\theta} \setminus S_{\theta'}]}{|\theta' - \theta|} = \infty.\] Let's take $\theta = 1$ and $\theta' = 1 + \epsilon$. Because of the uniform density, it suffices to show that 
\[\lim_{\epsilon \rightarrow 0} \frac{|S_{1} \setminus S_{1 + \epsilon}|}{\epsilon} = \infty.\] Notice that for sufficiently small $\epsilon$, we see that $S_{\theta} \setminus S_{\theta'} = [0, \sqrt{\epsilon^2 + 2 \epsilon}]$, and so $|S_{\theta} \setminus S_{\theta'}| =  \sqrt{\epsilon^2 + 2 \epsilon}$. We see that $\frac{|S_{\theta} \setminus S_{\theta'}|}{\theta' - \theta} = \frac{\sqrt{\epsilon^2 + 2 \epsilon}}{\epsilon} = 1 + \frac{2}{\sqrt{\epsilon}}$. This approaches $\infty$ as $\epsilon \rightarrow 0$, as desired. 
\end{proof}

\section{General implications of the expenditure constraint}\label{app:microgeneral}
In Appendix \ref{appendix:construction}, we explicitly construct $\Theta_0$ and $S(\Theta_0, c)$ in an example setting. In Appendix \ref{appendix:generalmodel}, we demonstrate that our insights about how the expenditure constraint results can reduce empirical burden extend to more general function classes. 

\subsection{Example construction of $\Theta_0$ and $S(\Theta_0, c)$}\label{appendix:construction}
In this section, we formalize Example \ref{ex:salient}. Let's formally construct $\Theta_0 := [l, u]$. We first define $l$: let $s' < \theta_{\text{SL}}$ be the value such that $c(\theta_{\text{SL}}, s') = 1$, let $s'' < s'$ be the value such that $c(s'', s') = 1$, and let $l < s$ be the value such that $c(s'', l) = 1$. We define $u$ similarly: let $t' > \theta_{\text{SL}}$ be the value such that $c(\theta_{\text{SL}}, t') = 1$, let $t'' > t$ be the value such that $c(t'', t') = 1$, and let $u > t''$ be the value such that $c(t'', u) = 1$. (To be precise, if we ever reach a stage when defining $l$ where no such point exists, then we take $l = 0$; similarly, if we ever reach a stage when defining $u$ where no such point exists, then we take $u = 1$.)  
\begin{proposition}
\label{prop:thetapruned}
Suppose that the agent response types $T \subseteq \mathcal{T}$ are expenditure-constrained with respect to cost function $c$ that is an outcome-valid cost function. The set $\Theta_{0}$ (defined above) contains a performatively optimal point $\thetaPO$ of $\DTrue$. 
\end{proposition}
\begin{proof}
Since $s', t' \in \Theta_0$, it suffices to show that $\PR(t') \le \PR(u)$ and $\PR(s') \le \PR(l)$. 

First, we show that $\PR(t') \le \PR(u)$. For both classifiers, by the expenditure constraint, all agents with true features $x$ such that $x < \theta_{\text{SL}}$ will necessarily be classified as $0$ by both $f_{t'}$ and $f_{u}$. Thus, we only need to consider $x$ such that $x \ge \theta_{\text{SL}}$. 
By the expenditure constraint, all agents with features $x$ such that $x > t''$ will necessarily be classified as $1$ by $f_{t'}$. By Assumption \ref{assumption:gamingbehavior}, coupled with the fact that $x \ge \theta_{\text{SL}}$ for these agents, the classifier $f_{u}$ cannot achieve a better loss for these agents. We can ignore agents with $x = t''$ since these agents form a measure $0$ set and $\DBase$ is continuous. For agents with true features $x$ such that $\theta_{\text{SL}} \le x < t''$, notice that these agents will necessarily be classified as $0$ by $f_{u}$ due to the expenditure constraint. Thus, by Assumption \ref{assumption:gamingbehavior}, coupled with the fact that $x \ge \theta_{\text{SL}}$ for these agents, the classifier $f_{t'}$ will not achieve a worse loss for these agents than $f_u$. 

We use a similar argument to show that $\PR(s') \le \PR(l)$. For both classifiers, by the expenditure constraint, all agents with true features $x$ such that $x \ge \theta_{\text{SL}}$ will necessarily be classified as $0$ by both $f_{t'}$ and $f_{u}$. Thus, we only need to consider $x$ such that $x < \theta_{\text{SL}}$.  By the expenditure constraint, all agents with features $x$ such that $x < s''$ will necessarily be classified as $0$ by $f_{s'}$. By Assumption \ref{assumption:gamingbehavior}, coupled with the fact that $x \le \theta_{\text{SL}}$ for these agents, the classifier $f_{l}$ cannot achieve a better loss for these agents. We can ignore agents with $x = s''$ since these agents form a measure $0$ set and $\DBase$ is continuous. For agents with true features $x$ such that $s'' < x \le \theta_{\text{SL}}$, notice that these agents will necessarily classified as $1$ by $f_{l}$ due to the expenditure constraint, and so by Assumption \ref{assumption:gamingbehavior}, coupled with the fact that $x \le \theta_{\text{SL}}$ for these agents, the classifier $f_{s'}$ will not achieve a worse loss for these agents. 

\end{proof}

We now show how to formally construct $S(\Theta_0, c)$. Let's construct a set $S'(\Theta_0, c) := [l', u']$ as follows. We can $l'$ and $u'$ in terms of $l$ and $u$ (the upper and lower endpoints of $\Theta_0$). Let $l' < l$ be the value such that $c(l', l) = 1$ and let $u' > u$ be the value such that $c(u, u') = 1$. Then $S'(\Theta_0, c) = [l', u']$. (To be precise, if no such $l'$ exists, then we take $l = 0$; similarly, if no such $u'$ exists, then we take $u' = 1$.) We now show that $S'(\Theta_0, c) = S(\Theta_0, c)$.
\begin{proposition}
\label{prop:xsalient}
Let $S'(\Theta_0, c)$ be defined as above. Then, $S(\Theta_0, c) = S'(\Theta_0, c)$.
\end{proposition}
\begin{proof}
It suffices to show that $S'(\Theta_0, c) = \cup_{\theta} S_{\theta}$, where $S_{\theta} = $. 

Suppose that $x \in S_{\theta}$. Then there exists $x'$ such that $c(x, x') \le 1$ and $f_{\theta}(x') \neq f_{\theta}(x)$ (which is equivalent to $\theta$ is between $x$ and $x'$). Using Assumption \ref{assumption:validcost}, this implies that $c(x, \theta) \le 1$. Using Assumption \ref{assumption:validcost} again, we see that since $\theta \in [l,u]$, this means that either $x \in [l,u]$, or $c(x, l) \le 1$, or $c(x, u) \le 1$ must be true. This implies that $x \in S'(\Theta_0, c)$. 

Suppose that $x \in S'(\Theta_0, c)$. If $x \in \Theta_0$, then we know that $x \in S_{x}$. If $x \not\in \Theta_0$ and $x < l$, then we know $x \in S_{l}$. If $x \not\in \Theta_0$ and $x > u$, then we know that $x \in S_u$. 
\end{proof}

\subsection{General function classes}\label{appendix:generalmodel}

While we focused on the 1-dimensional setting in Section \ref{sec:tractability}, we now demonstrate that the expenditure constraint can reduce empirical burden on the decision-maker in general settings. The following lemma formalizes the intuition that an appropriate estimate of agent's response types on $S(\Theta_0,c)$ is sufficient to achieve near-optimal performative risk. We use a subscript notation $\DMap_{S(\Theta_0,c)}(\theta)$ to denote the aggregate response distribution $\DMap(\theta)$ restricted to agents with true features $x \in S(\Theta_0,c)\subseteq X$. 
\begin{restatable}{lemma}{estdistmap}
\label{lem:estdistmap}
Let $c$ be a valid cost function, let $M, \tilde{M}$ be mappings that satisfy the expenditure constraint. Then, for any $\Theta_0\subseteq \Theta:\thetaPO(M)\in\Theta_0$, it holds that: 
 \[\PR(\thetaPO(M)) \le \PR(\thetaPO(\tilde M)) + 2  \xi\]
with
 $\xi:=\sup_{\theta} \left\{\mathbb{P}_{\DBase}\left[x \in S(\Theta_0,c)\right] \cdot \TV\big(\DMap_{S(\Theta_0,c)}(\theta; M), \DMap_{S(\Theta_0,c)}(\theta; \tilde M)\big)\right\}$.\footnote{We note that the decision-maker must only search over $\Theta_0$ in  \eqref{eq:PO_M} when computing $\thetaPO(\tilde{M})$. In particular, they must compute $\argmin_{\theta\in\Theta_0} \mathbb E_{(x,y)\sim\cD(\theta;\tilde M)} \left[\mathds{1}\{y\neq f_\theta(x)\}\right]$ rather than $\argmin_{\theta\in\Theta} \mathbb E_{(x,y)\sim\cD(\theta;\tilde M)} \left[\mathds{1}\{y\neq f_\theta(x)\}\right]$.}
\end{restatable}
In words, to achieve a small error $\xi$, we need to accurately estimate the distribution map for all $x\in S(\Theta_0,c)$, and the larger $S(\Theta_0,c)$, the more accurate the estimate must be. 

We prove Lemma \ref{lem:estdistmap}. 
\begin{proof}[Proof of Lemma \ref{lem:estdistmap}]
This proof is very similar to the proof of Lemma \ref{lem:estdistmapKS}. Like in that proof, let $\PR(\theta; M)$ denote the performative risk at $\theta$ on $\DMap(\theta; M)$ and let $\PR(\theta; \tilde{M})$ denote the performative risk at $\theta$ on $\DMap(\theta; \tilde{M})$. 

First, we claim that it suffices to show that $|\PR(\theta; M) - \PR(\theta; \tilde{M})| \le \xi$ for all $\theta \in \Theta_0$. This is because the decision-maker only searches over $\Theta_0$ when computing $\thetaPO(\tilde{M})$ and the true optimal point $\thetaPO(M)$ lies in $\Theta_0$, and so if $|\PR(\theta; M) - \PR(\theta; \tilde{M})| \le \xi$ for all $\theta \in \Theta_0$, then 
\[\PR(\thetaPO(\tilde{M}); M) \le \PR(\thetaPO(\tilde{M}); \tilde{M}) + \xi \le \PR(\thetaPO(M); \tilde{M})  + \xi \le \PR(\thetaPO(M); M) + 2\xi,\]
as desired.

The remainder of the proof boils down to showing that  $|\PR(\theta; M) - \PR(\theta; \tilde{M})| \le \xi$ for all $\theta \in \Theta_0$. Let's let $\DTrue$ be the distribution of $(t, x, y)$ where $(x,y) \sim \DBase$ and $t \sim M(x,y)$. Similarly, let $\DApprox$ be the distribution of $(t, x, y)$ where $(x,y) \sim \DBase$ and $t \sim \tilde{M}(x,y)$. We can thus obtain that:
\[|\PR(\theta; M) - \PR(\theta; \tilde{M})| =  \left|\mathbb{E}_{(t, x, y) \sim \DTrue}[\;\Indicator\{y \neq f_{\theta}(\Response_t(x, \theta))\}] - \mathbb{E}_{(t, x, y) \sim \DApprox}[\;\Indicator\{y \neq f_{\theta}(\Response_t(x, \theta))\}]\right|.\]
\noindent Now, we claim that for any agent $(t,x)$ where $x\not\in S(\Theta_0, c)$ and for $t \in \text{supp}(\DTrue) \cup \text{supp}(\DApprox)$, it holds that $f_{\theta}(\Response_t(x, \theta)) = f_{\theta}(x)$ for every $\theta \in \Theta_0$. Note that since $M$ satisfies the expenditure constraint with respect to $c$, then we know that if $x\not\in S(\Theta_0, c)$, it holds that $f_{\theta}(\Response_t(x, \theta)) = f_{\theta}(x)$. This yields the desired statement. Thus we have that:
\small{
\begin{align*}
   &\left|\mathbb{E}_{(t, x, y) \sim \DTrue}[\;\Indicator\{y \neq f_{\theta}(\Response_t(x, \theta))\}] - \mathbb{E}_{(t, x, y) \sim \DApprox}[\;\Indicator\{y \neq f_{\theta}(\Response_t(x, \theta))\}]\right|  \\
   &\le  \left|\mathbb{E}_{(t, x, y) \sim \DTrue}[\;\Indicator\{y \neq f_{\theta}(\Response_t(x, \theta))\} \;\Indicator\{x \not\in S(\Theta_0, c) \}] - \mathbb{E}_{(t, x, y) \sim \DApprox}[\;\Indicator\{y \neq f_{\theta}(\Response_t(x, \theta))\}\;\Indicator\{x \not\in S(\Theta_0, c) \}]\right| \\
   &+ \left|\mathbb{E}_{(t, x, y) \sim \DTrue}[\;\Indicator\{y \neq f_{\theta}(\Response_t(x, \theta))\} \;\Indicator\{x \in S(\Theta_0, c) \}] - \mathbb{E}_{(t, x, y) \sim \DApprox}[\;\Indicator\{y \neq f_{\theta}(\Response_t(x, \theta))\}\;\Indicator\{x \in S(\Theta_0, c) \}]\right| \\
     &=  \left|\mathbb{E}_{(t, x, y) \sim \DTrue}[\;\Indicator\{y \neq f_{\theta}(\Response_t(x, \theta))\} \;\Indicator\{x \in S(\Theta_0, c) \}] - \mathbb{E}_{(t, x, y) \sim \DApprox}[\;\Indicator\{y \neq f_{\theta}(\Response_t(x, \theta))\}\;\Indicator\{x \in S(\Theta_0, c) \}]\right| \\
    &= \left|\mathbb{E}_{(x,y) \sim \DMap(\theta; M)}[\;\Indicator\{y \neq f_{\theta}(x)\} \;\Indicator\{x \in S(\Theta_0, c) \}] - \mathbb{E}_{(x,y) \sim \DMap(\theta; \tilde{M})}[\;\Indicator\{y \neq f_{\theta}(x)\} \;\Indicator\{x \in S(\Theta_0, c) \}]\right| \\
    &=  \mathbb{P}[x \in S(\Theta_0, c)] \left|\mathbb{E}_{(x,y) \sim \DMap(\theta; M)}[\;\Indicator\{y \neq f_{\theta}(x)\}  \mid x \in S(\Theta_0, c) ] - \mathbb{E}_{(x,y) \sim \DMap(\theta; \tilde{M})}[\;\Indicator\{y \neq f_{\theta}(x)\} \mid x \in S(\Theta_0, c)]\right| \\
     &=  \mathbb{P}[x \in S(\Theta_0, c)] \left|\mathbb{E}_{(x,y) \sim \DMap_{S(\Theta_0, c)}(\theta; M)}[\;\Indicator\{y \neq f_{\theta}(x)\}] - \mathbb{E}_{(x,y) \sim \DMap_{S(\Theta_0, c)}(\theta; \tilde{M})}[\;\Indicator\{y \neq f_{\theta}(x)\}]\right| \\
    &\le \mathbb{P}[x \in S(\Theta_0, c)] \TV\left(\DMap_{S(\Theta_0, c)}(\theta; M), \DMap_{S(\Theta_0, c)}(\theta; \tilde{M})\right). 
\end{align*}
}
\normalsize{}
\end{proof}

\end{document}